\documentclass[a4paper,english,10pt]{article}

\usepackage[main, final]{./style/neurips_2026}

\usepackage{amsmath,amsfonts,bm,amssymb,nicefrac,algorithm}
\usepackage{amsthm}
\usepackage{booktabs}          %
\usepackage{graphicx}          %
\usepackage{subcaption}        %
\usepackage{multirow}        %
\usepackage{stmaryrd}        %
\usepackage{yhmath}        %

\def\eqref#1{Eq.~\ref{#1}}   %

\newcommand{\E}{\mathbb{E}}

\newtheoremstyle{mythmstyle} 
{\topsep}    %
{\topsep}    %
{\itshape}   %
{0pt}        %
{\bfseries}  %
{}           %
{ }          %
{}           %
\theoremstyle{mythmstyle}
\newtheorem{theorem}{Theorem}

\newtheorem{lemma}[theorem]{Lemma}

\newtheorem{proposition}[theorem]{Proposition}

 \newtheorem*{proposition*}{Proposition}
 \newtheorem*{theorem*}{Theorem}
  \newtheorem*{lemma*}{Lemma}

\usepackage{physics}
\usepackage{hyperref}
\usepackage[capitalize,noabbrev]{cleveref}

\usepackage[dvipsnames]{xcolor}  %

\definecolor{finn}{HTML}{7A2E6A}  %
\definecolor{purple}{HTML}{AA7BE0}
\definecolor{rose}{HTML}{E05260}
\definecolor{orange}{HTML}{ff9f00}
\hypersetup{
   breaklinks=true,   %
   colorlinks=true,   %
   citecolor=finn,
   linkcolor=rose,
   urlcolor=finn,
}

\setcitestyle{square,numbers,sort&compress}
\usepackage[toc,page,header]{appendix}
\usepackage{mathtools} %
\usepackage{xcolor}
\usepackage{booktabs}
\usepackage{cancel}
\usepackage{graphicx}
\usepackage{wrapfig}
\usepackage{pifont}
\newcommand{\cmark}{\ding{51}}
\newcommand{\xmark}{\ding{55}}

\usepackage[percent]{overpic}

\usepackage{tabularx}
\usepackage{array}

\usepackage{algpseudocode}

\usepackage{xspace} %

\newcommand{\tnce}{tNCE\xspace}
\newcommand{\tcnce}{tCNCE\xspace}
\newcommand{\stnce}{stNCE\xspace} %
\newcommand{\tstnceMixture}{stNCE-m\xspace}
\newcommand{\tstnceWhite}{stNCE-w\xspace}
\newcommand{\tstnceSelf}{stNCE-s\xspace}
\newcommand{\tstnceOracle}{stNCE-o\xspace}
\newcommand{\ctsm}{CTSM\xspace}
\newcommand{\dsm}{DSM\xspace}
\newcommand{\dualSm}{Dual SM\xspace}
\newcommand{\rneSelf}{RNE-s\xspace}
\newcommand{\rneOracle}{RNE-o\xspace}
\newcommand{\tstnceSelfDsm}{stNCE-s$+$DSM\xspace}
\newcommand{\tnceDsm}{tNCE$+$DSM\xspace}

\begin{document}

\title{Learning Energy-Based Models from Stochastic Interpolants using Spatiotemporal Differences}

\author{
  Hanlin Yu \\
  University of Helsinki \\
  \And
  RuiKang OuYang \\
  University of Cambridge \\
  \And
  Partha Kaushik \\
  Carnegie Mellon University \\
  \And
  Arto Klami \\
  University of Helsinki \\
  \And
  Michael U. Gutmann \\
  University of Edinburgh \\
  \And
  Omar Chehab \\
  Carnegie Mellon University
}

\addtocontents{toc}{\protect\setcounter{tocdepth}{0}}

\maketitle

\begin{abstract}
    
Learning an energy-based model from data samples is a central problem in machine learning. Many recent and popular methods, such as denoising score matching for training energy-based diffusion models, use stochastic interpolants to corrupt data samples at different noise levels indexed by a time variable. This defines a joint density over both the data space and time, and most methods learn its energy through either spatial or temporal differences. We identify distinct failure modes for both of these approaches. To solve them, we propose Spatiotemporal Noise-Contrastive Estimation (\stnce), a framework for learning the energy through joint spatiotemporal differences. \stnce unifies many existing methods and leads to new training objectives. Experiments on images and molecules demonstrate performance competitive with state-of-the-art density estimation methods. Code is available at \url{https://github.com/ksnxr/stnce}.

\end{abstract}

\section{Introduction}
\label{sec:introduction}

A central problem in machine learning is to learn the underlying probability density of a data distribution from finite samples. Accurate density models are useful in a wide range of applications. 
In computational chemistry, for example, they can be used to correct samples from a proposal distribution to the equilibrium distribution \citep{noe2019bg}. They also enable sampling from compositions of densities (e.g., products or tempered variants) which is used to combine generative models in a modular way~\citep{du2023rrr,skreta2025feynmankac}.
To support these applications efficiently, recent work has focused on learning likelihoods that can be evaluated in \textit{one shot}, i.e., with a single neural network evaluation~\citep{aggarwal2025boltznce,guth2025dual,ai2026f2d2,ouyang2026scallop}. A natural approach is to parameterize the likelihood as an Energy-based Model (EBM), where a neural network defines an unnormalized density. While EBMs also support sampling, one can alternatively use a separate generative model to draw samples~\citep{aggarwal2025boltznce}, and we focus on density estimation.

Over the past two decades, a wide range of methods has been proposed for training energy-based models~\citep{song2021ebms,gutmann2012nce_jmlr,ceylan2018cnce,hinton2002cd,hyvarinen2005sm,rhodes2020telescoping,choi2022densityratio,gao2020flowcontrastive}. In this work we focus on contrastive estimation approaches, showing that many established methods like Noise-Contrastive Estimation, score matching and contrastive divergence can be unified under the same framework. Recent works also looked into training flows in conjunction with EBMs~\citep{gao2020flowcontrastive,grenioux2023balanced} and there is an established literature on training EBMs based on likelihoods~\citep{hyvarinen2007connections}. 
Most approaches proceed in two steps. First, they define corrupted versions of the data by interpolating between data samples and samples from a reference distribution (e.g., a Gaussian)~\citep{rhodes2020telescoping, albergo2025stochasticinterpolant}. This introduces a \textit{time} variable $t$ indexing corruption levels and defines a joint density $p(x,t)$ over the data space and time. Second, the model is trained by estimating how this density varies, either across time or across space. This viewpoint yields our taxonomy. \textit{Temporal methods} estimate variations across corruption levels. For instance, Noise-Contrastive Estimation (NCE) learns large temporal differences between data and reference distributions~\citep{gutmann2012nce_jmlr}, while subsequent work improves performance by focusing on smaller temporal differences, as in telescoping density ratio estimation and time score matching~\citep{rhodes2020telescoping,choi2022densityratio,yu2025dre,chen2025diffusiondre}. \textit{Spatial methods}, in turn, estimate variations across data points. For instance, Conditional NCE~\citep{ceylan2018cnce} learns spatial difference between nearby data points. When these points are infinitesimally close, this becomes score matching~\citep{hyvarinen2005sm} that is widely used to train diffusion-based EBMs~\citep{du2023rrr,thornton2025ebmdiffusion}.

While both approaches for learning energy-based models have been widely applied~\citep{hyvarinen2005sm,song2020ssm,vincent2011dsm,rhodes2020telescoping,choi2022densityratio,yu2025dre,chen2025diffusiondre}, they are known to suffer from catastrophic failure modes that remain poorly understood. 
For example, spatial methods such as score matching~\citep{hyvarinen2005sm} can incur large estimation errors in multi-modal settings~\citep{wenliang2021blindness,koehler2023scorematching}. Conversely, temporal methods based on large differences, such as Noise-Contrastive Estimation, can fail when the high-density regions of the data distribution (its ``support'') do not overlap with those of the reference distribution.
However, these failure modes are rarely stated explicitly or analyzed in a unified manner in the literature. Often the analysis is limited to behavior of a specific method, instead of the broader approach.

A widely held view is that these challenges can be alleviated by learning across multiple corruption levels (or times). The intuition is that some levels are inherently easier than others, and that \emph{amortization}, through parameter sharing across levels, allows harder tasks to benefit from the easier ones. Empirically, this strategy has led to improvements in both temporal and spatial methods.
On the temporal side, multi-level NCE~\citep{rhodes2020telescoping} has been shown to improve sample efficiency, with some theoretical support~\citep{rhodes2020telescoping,chehab2023provable}. On the spatial side, multi-level denoising score matching for energy-based diffusion models~\citep{thornton2025ebmdiffusion} enjoys emerging theoretical justification~\citep{qin2024ebmdiffusion,chewi2025ebmdiffusion}.
Yet, this perspective is not fully satisfactory. Recent work indicates that such approaches do not fundamentally resolve the underlying issues: \citet{srivastava2023multi} show that multi-level NCE does not eliminate support mismatch, while \citet{yu2025cnce} provides empirical evidence that multi-level score matching continues to struggle in very simple multimodal settings.

Overall, our understanding of the failure modes of existing methods remains limited. Progress in the field has largely emphasized the development of new techniques, rather than the analysis of the limitations they aim to address. Although recent approaches report strong empirical gains~\citep{he2026rne,guth2025dual,plainer2025consistent,aggarwal2025boltznce,ouyang2026diffusiveclassificationlosslearning}, we still lack a clear understanding of when and why they outperform prior methods.

\paragraph{Contributions}
We revisit the problem of learning EBMs through the lens of spatial and temporal variations. We first introduce a unifying taxonomy that organizes existing methods along these two axes. We then explicitly identify their limitations when used in isolation: temporal methods struggle under support mismatch, while spatial methods fail in multi-modal settings, \textit{even when learning across different times} (corruption levels), contrary to popular belief.

Motivated by these observations, we propose to jointly learn spatial and temporal variations. We formalize this as \textit{Spatiotemporal Noise-Contrastive Estimation} (\stnce), a general framework that unifies existing approaches and addresses the previous limitations. While some very recent works empirically combine spatial and temporal variations, they lack a clear conceptual justification and a unified framework~\citep{he2026rne,guth2025dual,plainer2025consistent,aggarwal2025boltznce,ouyang2026diffusiveclassificationlosslearning}. Our framework unifies known methods and shows, perhaps remarkably, that they are \textit{all} instances of a \textit{single} binary classification task. Our framework leads also to new training objectives that are competitive with the state-of-the-art across synthetic, image, and molecular benchmarks, while providing a clearer understanding of when and why they succeed.

\begin{figure}[t]
    \centering
    \begin{overpic}[width=\textwidth,percent]{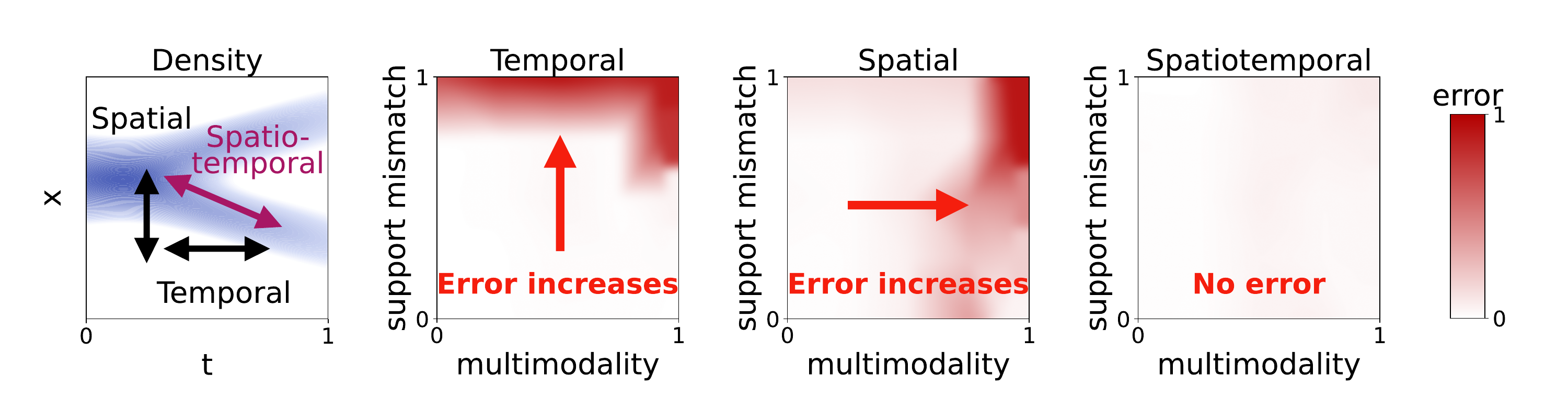}
    \put(11.5,24){\textbf{(a)}}
    \put(33.8,24){\textbf{(b)}}
    \put(56.2,24){\textbf{(c)}}
    \put(78.5,24){\textbf{(d)}}
    \end{overpic}
    \caption{\textit{(a): interpolant density $p_{d}(x, t)$}. We learn it via spatial differences which vary $x$ but fix $t$, temporal differences which vary $t$ but fix $x$, or spatiotemporal differences that vary both $x$ and $t$.
\textit{(b), (c), (d): empirical error of different learning methods} when trained to model the stochastic interpolant between standard normal noise and data given by a one-dimensional Gaussian mixture with two components, where the means vary and the standard deviation is fixed. Temporal methods incur high estimation error when the data distribution has high density in low density region of the noise distribution (the means of data components are far from high density region of standard noral noise, abusively, ``support mismatch"). Spatial methods incur high estimation error when the data distribution is multi-modal (the two components are far apart). Our spatiotemporal method achieves here virtually zero error in both setups. Further details are provided in Section~\ref{sec:failure-modes-exp} and Appendix~\ref{app:sec:failure-modes-details}.}
\label{fig:teaser_figure}
\end{figure}

\section{Background}
\label{sec:background}

\paragraph{Stochastic Interpolants}
Stochastic interpolants (SIs)~\citep{rhodes2020telescoping,albergo2025stochasticinterpolant} can be used to sample from a \textit{joint} density $p_d(x, t)$
over data points $x$ and their corruption levels $t$ ($0$ is fully corrupted, $1$ is clean), where we use $p_d(x|t)$ to denote the time-varying density induced by the SIs given the data distribution. SIs interpolate between clean data samples $x_1 \sim p_1$ and samples $x_0 \sim p_0$ from a reference distribution $p_0$, which is often a standard Gaussian. A simple interpolation is \citep{lipman2023fm}
\begin{align}
    \label{eq:stochastic_interpolant}
    x_t = (1-t) x_0 + t x_1,
\end{align}
but other choices are possible too~\citep{rhodes2020telescoping}. These interpolation schemes can be used to learn score fields, velocity fields and densities. In this work, we focus on learning the densities as time-varying energy-based models.
The joint density can be modelled as $p_{\theta}(x, t) = p(t) p_{\theta}(x | t)$, where $p(t)$ is a distribution over corruption levels chosen by the user, for example uniform between $0$ (fully corrupt data) and $1$ (clean data). The time-dependent density $p_{\theta}(x | t)$ is an energy-based
model
\begin{align}
p_{\theta}(x|t) = \exp(-E_{\theta}(x, t) - \log Z_t),
\end{align}
where $E_{\theta}(x,t)$ is typically a neural network and $Z_t = \int \exp(-E_{\theta}(x,t))\,dx$ is the normalizing constant, which is generally intractable. The goal is to learn a parameter $\theta$ such that the model of the clean data density $p_{\theta}(x\mid 1)$ is accurate. A popular way to estimate $\theta$ without having to compute the intractable $Z_t$ is to model variations of the joint density.

\paragraph{Binary classification for learning variations}
A seminal approach for learning the difference between two log densities $\log p_A(x)$ and $\log p_B(x)$ is to solve a binary classification task. This is core to many  algorithms~\citep{gutmann2012nce_jmlr,ceylan2018cnce,goodfellow2020gan}, including the ones in this paper. Consider data from two classes $p_A(x)$ and $p_B(x)$ and a binary classification task minimizing the logistic loss
\begin{align}
\label{eq:binary_classification}
\mathcal{L}_{\mathrm{logistic}}(F)
= 
- \mathbb{E}_{x \sim p_{A}} \left[ 
\log \mathrm{sigmoid}(F(x)) 
\right]
- 
\mathbb{E}_{x \sim p_{B}} \left[ 
\log (1 - \mathrm{sigmoid}(F(x))) 
\right].
\end{align}
The optimal classifier is the \textit{difference} between log densities $F(x) = \log p_A(x) - \log p_B(x)$~\citep{sugiyama2012density}. 

\paragraph{Learning the energy by temporal differences.}
\label{sec:energy-temporal}
One class of methods learns the energy by estimating \textit{temporal} differences $E(x, t') - E(x, t) + \log Z(t') - \log Z(t)$.
A canonical approach is Noise-Contrastive Estimation (NCE)~\citep{gutmann2012nce_jmlr}, using binary classification as above. First, generate data $x$ from two classes $A$ and $B$ with equal probability, where the classes correspond to the densities of $x$ under different times of the stochastic interpolant induced by the data distribution, $t$ and $t'$
\begin{align}
    p_{A}(x)
    = 
    p_d(x | t)
    , \quad
    p_{B}(x)
    = 
    p_d(x | t').
\end{align}
For example, generating data from class $A$ is done by corrupting a clean data sample $x_1$ with level $t$ following~\eqref{eq:stochastic_interpolant}. Then, minimizing the logistic loss yields the optimal classifier
\begin{align}
F(x, t, t')
=
- E(x,t) + E(x,t') - \log Z(t) + \log Z(t'),
\end{align}
that equals the temporal energy difference. In practice, we jointly parameterize $E_{\theta}(\cdot, \cdot)$ and $\log Z(\cdot)$ using a single neural network and learn them.
\citet{rhodes2020telescoping} observed that the performances become stronger when $t$ and $t'$ are sampled between $0$ and $1$ and close to each other; this was later generalized to continuous time \citep{aggarwal2025boltznce,ouyang2026diffusiveclassificationlosslearning}, where one define a specific $p(t)$ and sample $t' \sim p(t'|t)$, where $p(t'|t)$ can be close to $\mathcal{N}(t';t,\sigma^{2})$ for some small $\sigma$. In the limit where $t$ and $t'$ are infinitesimally close, that energy difference approaches a quantity determined by the \emph{time score} $\partial_t \log p_t(x)$~\citep{choi2022densityratio,yu2025dre}. We provide further details on $p(t)$ and $p(t'|t)$ in Appendix~\ref{app:sec:sampling}. We formalize this algorithm family as temporal NCE (\tnce) and present the resulting loss as follows
\begin{align}
&\mathcal{L}_{\mathrm{tNCE}}(\theta)
=
-
\mathbb{E}_{p_d(x,t)p_n(t'|t)}
\left[
\log
\mathrm{sigmoid}
\!\left(
\log p_\theta(x|t)-\log p_\theta(x|t')
+c(t,t')
\right)
\right]
\\
&\quad
-
\mathbb{E}_{p_d(x,t')p_n(t|t')}
\left[
\log
\left(
1-
\mathrm{sigmoid}
\!\left(
\log p_\theta(x|t)-\log p_\theta(x|t')
+c(t,t')
\right)
\right)
\right],
\end{align}
where $c(t, t') = \log p(t) + \log p_n(t'\mid t)- \log p(t') - \log p_n(t\mid t')$.

\paragraph{Learning the energy by spatial differences.}
Another class of methods learns the energy by estimating \textit{spatial} differences $E(x',t) - E(x,t) + \log Z(t) - \log Z(t)$. 
A fundamental example is Conditional Noise-Contrastive Estimation (CNCE)~\citep{ceylan2018cnce}, which also solves a binary classification task. First, generate data $(x, x')$ from two classes A and B
\begin{align}
    p_A(x, x') = p_d(x | t) p_n(x' \mid x)
    ,\quad
    p_B(x, x') = p_d(x' | t) p_n(x \mid x')
\end{align}
To generate data from class $A$, 
first generate a data point $x$ with corruption level $t$ following~\eqref{eq:stochastic_interpolant}, and then perturb it into $x'$, using a perturbation kernel $p_n$ chosen by the user. Switching the positions of $x$ and $x'$ yields data from class $B$. Then, minimizing the logistic loss yields the optimal classifier
\begin{align}
F(x, x', t)
=
- E(x, t) + E(x', t)
+ \log p_n(x' \mid x) - \log p_n(x \mid x').
\end{align}
It equals the spatial energy difference up to known terms.
We parameterize $E(\cdot, \cdot)$ using a neural network. As $x$ and $x'$ become infinitesimally close, the energy difference is determined by the \emph{space score} $\nabla_x \log p_t(x)$~\citep{hyvarinen2005sm,ceylan2018cnce}. Based on extensions of score matching to continuous time~\citep{song2019score,song2021sde}, a natural continuous time extension of CNCE is to sample $t$ between $0$ and $1$ and take expectation over $p(t)$. We formalize this algorithm as temporal CNCE (\tcnce), where the loss is
\begin{align}
&\mathcal{L}_{\mathrm{tCNCE}}(\theta)
=
-
\mathbb{E}_{p_d(x,t)p_n(x'|x)}
\left[
\log
\mathrm{sigmoid}
\!\left(
\log p_\theta(x|t)
-
\log p_\theta(x'|t)
+
c(x,x')
\right)
\right]
\\
&\quad
-
\mathbb{E}_{p_d(x',t)p_n(x|x')}
\left[
\log
\left(
1-
\mathrm{sigmoid}
\!\left(
\log p_\theta(x|t)
-
\log p_\theta(x'|t)
+
c(x,x')
\right)
\right)
\right].
\end{align}
We again pack the terms independent of $\theta$ into the notation $c(x,x') = \log p_n(x'|x) - \log p_n(x|x')$.

\section{Limitations of current methods}
\label{sec:limitations}

\begin{figure}[t]
\centering
\includegraphics[width=\textwidth]{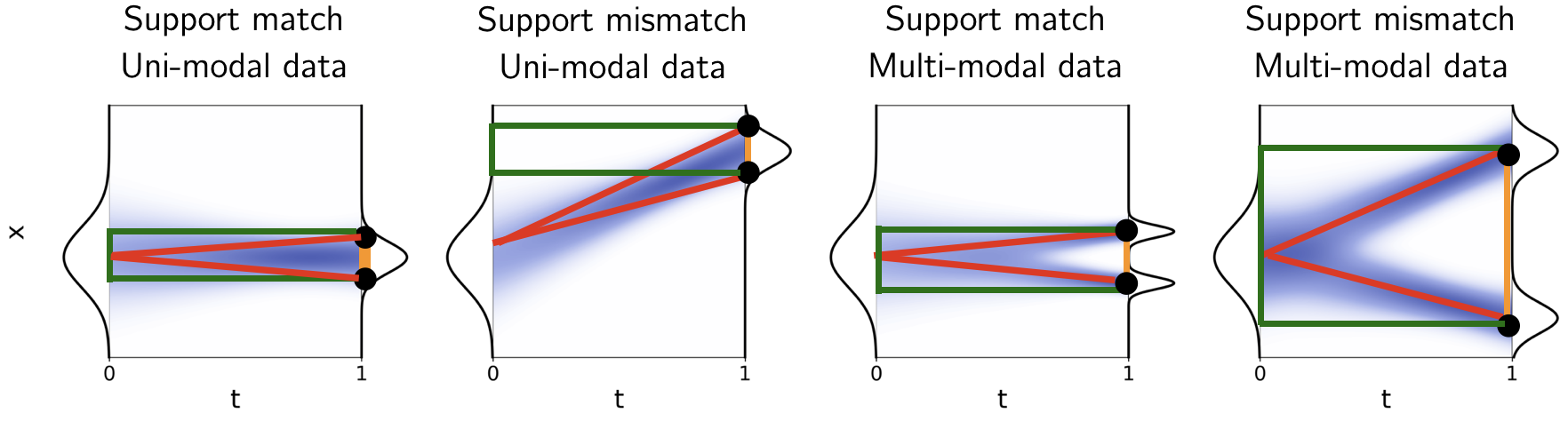}
\caption{
Ability to compute log-likelihood differences between two data points using learned spatial (orange), temporal (green), or spatiotemporal (red) differences. Methods based purely on spatial or temporal differences \textit{fail when traversing low-density regions} (white), where these differences are poorly estimated. In contrast, spatiotemporal methods remain within high-density regions (blue), enabling accurate estimation.
}
\label{fig:limitations}
\end{figure}

Despite widespread success, current methods suffer from failure modes. It is well known that methods based on spatial differences~\citep{ceylan2017cncethesis,vincent2011dsm} and temporal differences~\citep{gutmann2012nce_jmlr} can fail even in simple settings. A common belief is that these issues can be mitigated by \textit{learning across multiple corruption levels}, through temporal schemes~\citep{rhodes2020telescoping,choi2022densityratio,yu2025dre,chen2025diffusiondre} or spatial ones~\citep{du2023rrr,thornton2025ebmdiffusion,song2021sde}. Here we argue—and empirically demonstrate—that this is \textit{not} the case: they still fail in the same simple regimes.

\paragraph{A unifying heuristic.}
To analyze these failures, we study the ability of a model to relate the log-density of two points. For any two points, $(x,t)$ and $(x',t')$, we decompose the difference into spatial and/or temporal differences that need to be learnt by the model. This viewpoint captures both local properties (e.g., mode geometry for nearby points) and global properties (e.g., relative mode weights for distant points). Crucially, we will show that purely spatial or temporal decompositions inevitably require evaluating $p(x,t)$ in low-density regions, where samples are scarce and estimation is therefore worse. Figure~\ref{fig:teaser_figure} empirically validates these claims to hold even for a bivariate density.

\paragraph{Spatial differences struggle with multimodality.}
Spatial methods like \tcnce learn $p_\theta(x,t)$ via differences in $x$, where one interpolates between $x$ and $y$ and discretizes it as $x_k$ such that $x_0=y$ and $x_K=x$. Using our heuristic, we decompose the log-density difference between $x$ and $y$ as
\begin{align}
\log p(x,1)-\log p(y,1)
=\sum_{i=0}^{K-1}\left(\log p(x_{i+1},1)-\log p(x_i,1)\right).
\end{align}
When $p_1$ is multimodal (Figure~\ref{fig:limitations}, two right panels), these spatial differences in orange necessarily traverse low-density regions (white). 
As a result, spatial methods—including energy-based diffusion models—struggle to capture global structure such as mode weights~\citep{yu2025cnce}.

\paragraph{Temporal differences suffer from support mismatch.}
Temporal methods learn $p_\theta(x,t)$ via differences in $t$. 
For \tnce, using our heuristic, we discretize $t$ into a grid in the form of $t_k$ where $t_0=0$ and $t_K=1$, and decompose the log-density difference between $x$ and $y$ as
\begin{align}
\log p(x,1)-\log p(y,1)
=\log p(x,0)-\log p(y,0)+\Delta,
\end{align}
where $p(\cdot,0)$ is a known reference distribution, and
\begin{align}
\Delta
=\sum_{i=0}^{K-1}\!\bigl(\log p(x,t_{i+1})-\log p(x,t_i)\bigr)
-\sum_{i=0}^{K-1}\!\bigl(\log p(y,t_{i+1})-\log p(y,t_i)\bigr).
\end{align}
When the supports of $p_0$ and $p_1$ differ (Figure~\ref{fig:limitations}, second and fourth panels), these temporal differences in green traverse low-density regions in white. As before, this degrades estimation. 

\paragraph{Spatiotemporal differences avoid both issues.}
We now consider a joint discretization in space and time, combining the two discretizations above. Let $\{(x_i,y_i,t_i)\}_{i=0}^K$ with $x_i,y_i \sim p_{t_i}$. Then
\begin{equation}
\log p(x,1)-\log p(y,1)
=\log p(x_0,0)-\log p(y_0,0)+\Delta,
\end{equation}
where
$
\Delta
=\sum_{i=0}^{K-1}\!\bigl(\log p(x_{i+1},t_{i+1})-\log p(x_i,t_i)\bigr)
-\sum_{i=0}^{K-1}\!\bigl(\log p(y_{i+1},t_{i+1})-\log p(y_i,t_i)\bigr).
$
As illustrated in Figure~\ref{fig:limitations}, these spatiotemporal differences in red can be chosen to remain within high-density regions in blue. Since such regions are well sampled, the corresponding spatiotemporal differences are accurately estimated.
This observation may explain the success of recent spatiotemporal methods~\citep{guth2025dual,he2026rne,aggarwal2025boltznce,ouyang2026diffusiveclassificationlosslearning,plainer2025consistent}, as they simultaneously address multimodality and support mismatch. We empirically validate this conjecture in Figure~\ref{fig:teaser_figure}. This motivates the following unifying framework for spatiotemporal methods. We discuss the relationship between our \stnce framework and related spatiotemporal methods in Appendix~\ref{app:sec:related-spatiotemporal}.

\section{Learning the energy by spatiotemporal differences}

In this section, we introduce a framework for estimating the parameters of an EBM by learning its variation jointly in data space and time. Our approach reduces to a \textit{single} binary classification problem that unifies many popular methods~\citep{hyvarinen2005sm,vincent2011dsm,gutmann2012nce_jmlr,ceylan2018cnce,rhodes2020telescoping,choi2022densityratio,yu2025dre,chen2025diffusiondre,guth2025dual,aggarwal2025boltznce,ouyang2026diffusiveclassificationlosslearning}; see Table~\ref{table:stnce_special_cases}.

\paragraph{Spatiotemporal Noise-Contrastive Estimation (\stnce)}
We propose to learn the energy by estimating its \textit{spatiotemporal} differences, $E(x', t') - E(x, t)$ through binary classification. We term this Spatiotemporal Noise-Contrastive Estimation (\stnce). We start from two classes $A$ and $B$
\begin{align}
    p_A(x, t, x', t') 
    = p_d(x, t) p_n(x', t' | x, t), \quad
    p_B(x, t, x', t') 
    = p_d(x', t') p_n(x, t | x', t').
\end{align}
To generate data from class $A$, first  sample a data point and its corruption level $(x, t)$ following~\eqref{eq:stochastic_interpolant}, and then perturb them into $(x', t')$ using a perturbation kernel $p_n$ that is chosen by the user. Switching the positions of $(x, t)$ and $(x', t')$ in the tuple yields data from class $B$.
Strategies for sampling $(x,t,x',t')$ are discussed in Appendix~\ref{app:sec:sampling}.
Minimizing the logistic loss yields the optimal classifier
\begin{align}
\label{eq:stnce_classifier}
F(x, t, x', t')
&=
-E(x,t)+E(x',t')-\log Z(t)+\log Z(t') + \log p(t) - \log p(t') \\
&\quad +\log p_n(x',t'\mid x,t)-\log p_n(x,t\mid x',t'),
\end{align}
which equals the spatiotemporal energy difference up to known terms. In practice, we parameterize $E(x,t)$ and $\log Z(t)$ using neural networks and learn them jointly. Note that this can be viewed as CNCE from Section~\ref{sec:background} on the augmented space $(x, t)$
. This leads to the following loss
\begin{align}
&\mathcal{L}_{\mathrm{stNCE}}(\theta)
=
-
\mathbb{E}_{p_d(x,t)p_n(x', t'|x, t)}
\left[
\log
\mathrm{sigmoid}
\!\left(
\log p_\theta(x|t)
-
\log p_\theta(x'|t')
+
c(x,t,x',t')
\right)
\right]
\\
&\quad
-
\mathbb{E}_{p_d(x',t')p_n(x, t|x', t')}
\left[
\log
\left(
1-
\mathrm{sigmoid}
\!\left(
\log p_\theta(x|t)
-
\log p_\theta(x'|t')
+
c(x,t,x',t')
\right)
\right)
\right].
\end{align}
For a fixed perturbation kernel, we again pack the terms independent of the optimized parameters into $c(x,t,x',t') = \log p_n(x', t'|x, t) - \log p_n(x, t|x', t') + \log p(t) - \log p(t')$.

\paragraph{Statistical guarantees.}
We next present statistical guarantees on the \stnce estimator under a fixed perturbation kernel: it is consistent and we can quantify its sample-efficiency. We remark that it applies to the forward-reverse kernel where the score estimates are potentially arbitrary but fixed, though the practical \tstnceSelf variant uses score estimates that are updated during optimization and hence is beyond the scope of our existing theory.

\begin{theorem}[Non-parametric consistency]
\label{theorem:fisher_consistency}
Let $p_\theta(x|t) = \exp(-E_\theta(x,t))/Z_\theta(t)$ be a parametric model of the joint density where $p_\theta(x\mid 0)$ equals the known Gaussian noise, and suppose the data is generated with parameter $\theta^*$. Assume that the fixed perturbation kernel $p_n$ has full support for every conditioning point, i.e. $p_n(x',t'|x,t) > 0$ for all $(x,t),(x',t') \in \mathbb{R}^D \times [0,1]$, where $D$ is the dimensionality of $x$. Under standard regularity conditions ensuring convergence of empirical risk minimizers to population minimizers, the \stnce estimator $\hat\theta$, obtained by minimizing the empirical logistic loss with $N$ samples, leads to a pointwise consistent estimator:
\begin{align}
    p_{\hat\theta} \to p_{\theta^*} \quad \text{as } N \to \infty.
\end{align}
\end{theorem}
We prove this in Appendix~\ref{app:ssec:fisher_consistency}, and next characterize the estimator's asymptotic sample efficiency.

\begin{theorem}[Sample efficiency]
\label{theorem:sample_efficiency}
(Based on \citet{yu2025cnce})
Let
\begin{align}
F_\theta(x,t,x',t')
= \log p_\theta(x,t) - \log p_\theta(x',t')
+ \log p_n(x',t'\mid x,t)
- \log p_n(x,t\mid x',t').
\end{align}
Then the estimator satisfies
\begin{equation}
\E\|\hat\theta - \theta^*\|^2
= \frac{1}{N_d}\,\mathrm{Tr}\!\bigl(C_1^{-1} C_2 C_1^{-1}\bigr)
+ o\!\left(\frac{1}{N_d}\right),
\end{equation}
where $C_1$ and $C_2$ are given as follows, with expectations taken under $p_d(x,t)p_n(x',t'\mid x,t)$:
\begin{align}
C_1
= \E\!\left[
\nabla_\theta F_\theta \nabla_\theta F_\theta^\top
\,\mathrm{sigmoid}(-F_{\theta^*})\,\mathrm{sigmoid}(F_{\theta^*})
\right]
,\quad
C_2
= \E\!\left[
\nabla_\theta F_\theta \nabla_\theta F_\theta^\top
\,\mathrm{sigmoid}(-F_{\theta^*})^2
\right].
\end{align}
\end{theorem}

\paragraph{Design choices.}
The main design choice is the \textit{perturbation kernel} $p_n$. While our \stnce estimator is consistent for any choice of $p_n$ satisfying the assumption as in Theorem~\ref{theorem:fisher_consistency}, its sample efficiency depends critically on this choice. As shown in Table~\ref{table:stnce_special_cases}, specific choices recover many existing methods based on temporal differences~\citep{gutmann2012nce_jmlr,rhodes2020telescoping,choi2022densityratio,yu2025dre}, spatial differences~\citep{ceylan2018cnce,hyvarinen2005sm,song2020ssm}, and spatiotemporal differences~\citep{guth2025dual,aggarwal2025boltznce,he2026rne,ouyang2026diffusiveclassificationlosslearning}.
In principle, one could optimize the asymptotic variance expression with respect to $p_n$. As is often the case, this is analytically difficult. Instead, we adopt an empirical approach and propose choices of $p_n$ that are heuristically well suited to the problem.

\section{Choices of perturbation kernels}
\label{sec:perturbation_kernels}

\begin{wrapfigure}{r}{0.45\textwidth}
    \centering
    \includegraphics[width=0.8\linewidth]{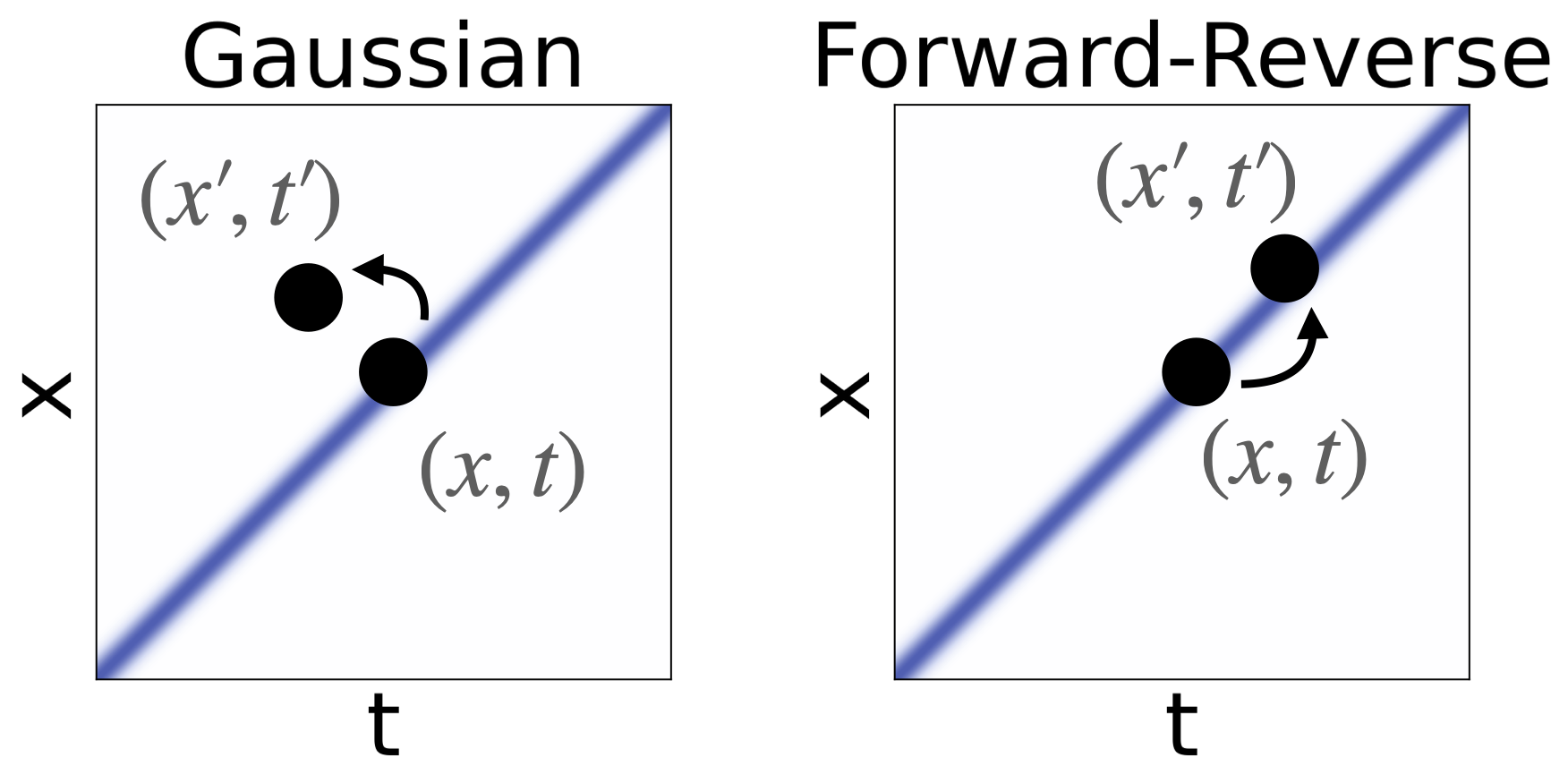}
    \caption{Choice of kernels. Left: departs from the data manifold (blue). Right: remains on the data manifold (blue).}
    \label{fig:perturbation_kernels_illustration}
\end{wrapfigure}

In this section, we propose three choices for the perturbation kernels, called \textit{mixture}, \textit{white}, and \textit{forward-reverse}. We argue that the first two are sample-efficient when the perturbation is ``small enough", while the last is sample-efficient when a certain quantity called the ``space score" is well-approximated.
In both cases, the guiding principle is that the pairs $(x,t)$ and $(x', t')$ used to estimate energy difference should remain within high density region of $p_d(x,t)$; because one of them is sampled from $p_d(x, t)$ and is therefore likely to be on the data manifold, the perturbation kernel $p_n$ used to obtain the second should avoid mapping samples \textit{outside} of the data manifold, see Figure \ref{fig:perturbation_kernels_illustration}.
We explain that certain recent methods in the literature~\citep{guth2025dual,he2026rne} can be viewed as special cases or variants on these perturbation kernels.

\paragraph{Mixture kernel.}
This kernel, leading to \tstnceMixture variant, perturbs \textit{either} $x$ or $t$:
\begin{equation}
p_n(x',t'\mid x,t)
= \tfrac{1}{2}\, p_n(t'\mid t)\,\delta_{x}(x')
+ \tfrac{1}{2}\, p_n(x'\mid x)\,\delta_{t}(t').
\end{equation}
It recovers a mixture of \tnce and \tcnce objectives, as we show in Appendix~\ref{app:sec:connection_to_other_methods}. 

\paragraph{White noise kernel.}
This kernel perturbs $x$ using Gaussian noise and varies $t$
\begin{equation}
p_n(x',t'\mid x,t)
= \mathcal{N}(x'\mid x,\sigma^2 I)\,p(t'|t).
\end{equation}
We refer to the resulting algorithm as \tstnceWhite. Sufficiently small perturbations keep $(x',t')$ in the data manifold, and as $\sigma \to 0$, this recovers Dual Score Matching~\citep{guth2025dual}, as we prove in Appendix~\ref{app:sec:dual_score_matching}.  

These two kernels suffer from the same pathology: for non-vanishing perturbations, $(x',t')$ could naturally be generated off the high density region of $p_{d}(x,t)$. The manifold hypothesis~\citep{bengio2013repre} states that real world data lies on a low dimensional manifold, such that adding isotropic Gaussian noise often results in off-manifold samples. Furthermore, samples from Gaussian distribution concentrate near a sphere in high dimensions~\citep{vershynin2018high}, resulting in pathologies even for noise-perturbed data.
In contrast to the kernels above, the following kernel aims to stay within high density region.

\paragraph{Forward-reverse kernel.}
We propose perturbing $(x,t)$ along trajectories consistent with $p_d(x,t)$:
\begin{equation}
p_n(x',t'\mid x,t)
= p_n(x'\mid x,t,t')\,p_{n}(t'|t).
\end{equation}
We first perturb $t$ to $t'$, then apply a kernel on $x$: a \textit{noising} step if $t'<t$, and a \textit{denoising} step if $t'>t$. The idealized versions of these kernels are Bayes inverses, as in diffusion-based methods~\citep{gao2021learning,doucet2022diffusionais,he2026rne}. An explicit construction based on the stochastic interpolant as in~\eqref{eq:stochastic_interpolant} is:
\begin{equation}
\label{eq:forward-reverse-kernel}
p_{n}(x'\mid x,t,t') =
\begin{cases}
\mathcal{N}\!\left(x'\middle| (t'/t)x,\; \left((1-t')^2 - t'^2(1-t)^2/t^2\right) I \right),
& t'<t,\\[0.3em]
\mathcal{N}\!\left(
x'\middle|
(t'/t)x + \tfrac{2(1-t)(t'-t)}{t}\,\text{sg}\left(\nabla_x \log p_t(x)\right),\;
\tfrac{(t'-t)(t+t'-2tt')}{t^2} I
\right),
& t'>t,
\end{cases}
\end{equation}
where $\text{sg}$ denotes stop gradient, preventing gradient propagating through the score in the kernel, such that the kernel is indeed fixed during optimization, while refreshed after each step.
The noising kernel (top) is tractable, while the denoising kernel, obtained using the exponential integrator formula~\citep{gonzalez2023seeds},  depends on the space score $\nabla_x \log p_t(x)$, which is typically unknown and is the key quantity that diffusion models approximate~\citep{song2021sde}. We use either:
(i) \texttt{oracle} scores (when available), denoted as \tstnceOracle, or
(ii) \texttt{self} scores, obtained from the current model, denoted as \tstnceSelf.
The latter is reminiscent of consistency models and flow maps~\citep{song2023consistency,boffi2025how}. 
The recent RNE~\citep{he2026rne} proposed to regularize energy-based diffusion models is related too. 
However, while \stnce is consistent by itself already for suitable kernels as in Theorem~\ref{theorem:fisher_consistency}, given imperfect scores RNE is no longer consistent and cannot be used independently. See Appendix~\ref{app:sec:discussion-rne} for further discussion.

\paragraph{Infinitesimal limit.}
We now characterize the behavior of \stnce for certain \textit{infinitesimal} perturbation kernels, specifically perturbation kernels that are transition probabilities along a trajectory of a Stochastic Differential Equation (SDE) as in \eqref{eq:forward-reverse-kernel} or of an Ordinary Differential Equation (ODE), assuming oracle access to the ground truth scores.
We prove the following results in Appendix~\ref{app:sec:proofs}.

\begin{proposition}[Transition probability of an SDE]
\label{th:sm_mle}
Consider an SI whose marginal laws $p_t$ are simulated by the time-reversal of an SDE, $\dd X_t = b(X_t, t) \dd t + g(t) \dd W_t$, where $b$ and $g$ are SDE parameters. Define the perturbation kernel to be the transition probability of that SDE. In the limit $\Delta t \to 0$, \stnce recovers the space score matching loss with maximum-likelihood weighting~\citep{song2021mle}:
\begin{align}
\mathcal{L}_{\mathrm{\stnce}}(\theta)
= 2\log 2
+ \frac{\Delta t}{4}\,
\E_{p_t}\!\left[
g(t)^2 \|\nabla_x \log p_t(x) - \nabla_x \log p_t^{\theta}(x)\|^2
\right]
+ o(\Delta t).
\end{align}
\end{proposition}

\begin{proposition}[Transition probability of an ODE]
\label{th:total_derivative_matching}
Consider an SI whose marginal laws $p_t$ are simulated by an ODE, $\dd X_t = v_t(X_t) \dd t$, where $v_t$ is the velocity field. 
Define the perturbation kernel to be the transition probability of that ODE. Under this scenario, the transition kernels are Dirac deltas subject to Jacobian scalings. 
In the limit $\Delta t \to 0$, \stnce recovers a loss for matching the trajectory-wise rate of change of the true and model log densities:
\begin{align}
\mathcal{L}_{\mathrm{\stnce}}(\theta)
= 2\log 2
+ \frac{\Delta t^2}{4}\,
\E_{p_t}\!\left[
\| \dd_t \log p_t(x) - \dd_t \log p_t^\theta(x) \|^2
\right]
+ o(\Delta t^2).
\end{align}
where $\dd_t \log p_t(x) := \partial_t \log p_t(x) + \nabla_{x} \log p_t(x) \cdot v_t(x)$.
\end{proposition}

\section{Experiments}
\label{sec:experiments}

We empirically evaluate our framework for learning densities
of synthetic data, images and molecules. In particular, we report results of the different variants of our Spatiotemporal Noise-Contrastive Estimation (\tstnceMixture, \tstnceWhite, \tstnceSelf, \tstnceOracle) that are detailed in Section~\ref{sec:perturbation_kernels}. 
We provide a compact glossary of the different algorithms in Table~\ref{tbl:glossary}. We may also combine \tnce or \tstnceSelf with a loss function on the space score estimates induced by the energy network, in which case we explicitly indicate it by writing $+$DSM; see Appendix~\ref{app:sec:adding-dsm}. These space score estimates can be reused for \tstnceSelf loss. After the models are trained, we perform an approximate normalization step following~\citet{guth2025dual}; see Appendix~\ref{app:sec:normalization}. Our implementation of \tstnceMixture naively adds up \tnce loss and \tcnce loss, leading to slower running time, as indicated using $^*$. Across datasets, our \tstnceSelf and variants are most promising: they are competitive with the state of the art.

\subsection{Failure modes}
\label{sec:failure-modes-exp}

Many methods for learning energy-based models, including denoising score matching used in diffusion models~\citep{thornton2025ebmdiffusion,du2019implicit}, are typically evaluated on complex datasets. Yet, Figure~\ref{fig:teaser_figure} shows that they can already fail in a simple one-dimensional setting. This experiment highlights two key properties of the data distribution that drive these failures. First, multimodality increases the error of methods based on learning spatial differences. Second, a mismatch between the high-density regions (``supports'') of the data and reference distributions increases the error of methods based on learning temporal differences. Methods based on learning spatiotemporal difference are robust to both sources of error. Details on the simulations are provided in Appendix~\ref{app:sec:failure-modes-details}.

\subsection{Gaussian mixtures}
\label{sec:mnist-mixture-exp}

Image data are high-dimensional and sparse. To emulate these properties, we build a Gaussian mixture from the MNIST dataset \citep{lecun2010mnist} by sampling $100$ images ($10$ per digit class) and using them as centroids with diagonal covariance $0.01 I$.
This synthetic yet realistic distribution exhibits the two pathologies identified in Section~\ref{sec:limitations}: strong multimodality and minimal overlap with a standard Gaussian reference distribution due to its thin support. We therefore expect purely spatial and purely temporal methods to struggle, while spatiotemporal methods with density-aware kernels should perform well.
We use LogMSE, LogRatio, LogNormMSE and NormNLL as evaluation metrics and benchmark different EBM training algorithms, measuring the quality of the learned energies, either with or without using a separately estimated normalizing constant using importance sampling, where always lower is better.
Further details on the evaluation metrics are provided in Appendix~\ref{app:sec:mnist-mixture-evaluations}.

The results are shown in Table~\ref{tbl:mnist-mixture-results}. 
\tstnceWhite results match the intuition of Figure~\ref{fig:perturbation_kernels_illustration}: finite Gaussian perturbations perform poorly because they push samples off the data manifold, whereas infinitesimal perturbations (\dualSm~\citep{guth2025dual}) preserve proximity to the manifold and yield accurate estimates, albeit at higher computational cost due to double differentiation through the network.
Results for the forward–reverse kernel further support this interpretation. Even with finite corruption, \tstnceSelf remains competitive, possibly because it stays close to the data manifold. By contrast, as discussed in Section~\ref{sec:limitations}, RNE~\citep{he2026rne} is consistent only when given access to the true scores $\nabla_{x_t} \log p_t(x_t)$ (\rneOracle), which are unavailable in practice; using estimated scores (\rneSelf) can lead to substantial degradation. Overall, \tstnceSelf provides a favorable trade-off between accuracy and speed, and is the only NCE-based method that successfully learns this challenging Gaussian mixture.

We additionally perform a fine-grained analysis to validate that the error analysis of different methods in Figure~\ref{fig:teaser_figure} and Figure~\ref{fig:limitations} holds for the trained models. As presented in Appendix~\ref{app:sec:analysis-mnist-mixture}, our unifying heuristics as discussed in Section~\ref{sec:limitations} provides a reasonable prediction for the performance of the different methods, where the estimated density ratios have higher errors in low density region.

\begin{table*}[!htbp]
\centering
\setlength{\tabcolsep}{3pt} %
\caption{Results on MNIST mixture (784 dim). Time means average time per step in seconds. For methods not in Oracle category best results are in \textbf{bold}, second ones are \underline{underlined}. Under oracle category best results are in \textbf{bold}.}
\label{tbl:mnist-mixture-results}
\begin{tabular}{@{}llcccccc@{}}
\toprule
			\textbf{Category} & \textbf{Method} & \textbf{LogMSE} $\downarrow$ & \textbf{LogRatio} $\downarrow$ & \textbf{LogNormMSE} $\downarrow$ & \textbf{NormNLL} $\downarrow$ & \textbf{Time} $\downarrow$ \\
			\midrule
			\multirow{2}{*}{\begin{tabular}[c]{@{}c@{}}Temporal\end{tabular}} & \tnce & 12.37 $\pm$ 0.02 & 6.85 $\pm$ 0.14 & 8.86 $\pm$ 0.19 & -606.76 $\pm$ 7.83 & \underline{0.39} \\
			 & \ctsm & 12.44 $\pm$ 0.45 & 7.14 $\pm$ 0.41 & 8.96 $\pm$ 0.24 & -602.92 $\pm$ 9.59 & 0.68 \\
			\midrule
			\multirow{2}{*}{\begin{tabular}[c]{@{}c@{}}Spatial\end{tabular}} & \tcnce & 12.28 $\pm$ 0.26 & 7.9 $\pm$ 0.22 & 9.13 $\pm$ 0.18 & -598.73 $\pm$ 7.82 & 0.4 \\
			 & \dsm & 16.39 $\pm$ 0.07 & 9.31 $\pm$ 0.43 & 10.24 $\pm$ 0.39 & -535.54 $\pm$ 28.81 & 0.7 \\
			\midrule
			\multirow{5}{*}{\begin{tabular}[c]{@{}c@{}}Spatial-\\temporal\end{tabular}} & \tstnceMixture & 10.31 $\pm$ 0.12 & 6.29 $\pm$ 0.45 & 7.64 $\pm$ 0.42 & -644.8 $\pm$ 8.51 & 0.45$^{*}$ \\
			 & \tstnceWhite & 9.96 $\pm$ 0.05 & 5.96 $\pm$ 0.33 & 7.27 $\pm$ 0.29 & -652.43 $\pm$ 5.22 & 0.4 \\
			 & \tstnceSelf & \underline{4.8 $\pm$ 0.1} & \underline{3.34 $\pm$ 0.07} & \underline{4.17 $\pm$ 0.12} & \underline{-680.99 $\pm$ 0.53} & \textbf{0.36} \\
			 & \rneSelf & 14.86 $\pm$ 0.31 & 7.75 $\pm$ 0.39 & 9.62 $\pm$ 0.05 & -570.51 $\pm$ 1.09 & 0.44 \\
			 & \dualSm & \textbf{3.73 $\pm$ 0.18} & \textbf{2.88 $\pm$ 0.06} & \textbf{3.25 $\pm$ 0.11} & \textbf{-683.96 $\pm$ 0.33} & 0.7 \\
			\midrule
			\multirow{2}{*}{\begin{tabular}[c]{@{}c@{}}Oracle\end{tabular}} & \tstnceOracle & 4.99 $\pm$ 0.11 & 3.69 $\pm$ 0.06 & 4.95 $\pm$ 0.03 & -677.07 $\pm$ 0.12 & \textbf{0.32} \\
			 & \rneOracle & \textbf{2.92 $\pm$ 0.08} & \textbf{2.06 $\pm$ 0.02} & \textbf{2.24 $\pm$ 0.07} & \textbf{-685.75 $\pm$ 0.09} & 0.4 \\
			\bottomrule
\end{tabular}
\end{table*}

\subsection{Images}

In this section, we follow previous related works~\citep{guth2025dual,yu2025dre} and report the Negative Log-Likelihood (NLL) in bits per dimension, defined as
$\mathrm{NLL}_{\mathrm{bpd}}
:=
\frac{
-\mathbb{E}_{x \sim p_d(\cdot \mid 1)}
\!\left[
\log p_\theta(x \mid 1)
\right]
}{
D \log 2
}$,
where $D$ is the data dimension (the number of pixels times channels). Here the DSM loss is a conditional flow matching loss on the velocity estimates recovered from energy network; see Appendix~\ref{app:sec:adding-dsm} for details.

\paragraph{MNIST}
\label{sec:mnist-exp}
We next train the different algorithms to learn the data density of images from the MNIST dataset of handwritten digits~\citep{lecun2010mnist}. 
Results are shown in Table~\ref{table:mnist}: our \tstnceSelf is the best among NCE algorithms while being fast to train. In this table the algorithms were trained for $100,000$ steps; we verify in Table~\ref{tbl:mnist-tstnce-long}
that \tstnceSelf continues to improve when optimized longer. Our \tstnceSelfDsm reaches results competitive with strong baselines, by jointly training the spatial scores of the density. We additionally verify the sampling capabilities of the trained model; see Appendix~\ref{app:sec:mnist-additional-results}.

\paragraph{ImageNet}
\label{sec:image-imagenet-exp}

\begin{table}[htbp]
\centering
\caption{Negative log likelihood (in bits/dimension) on ImageNet64 test set as in~\citet{guth2025dual}.}
\label{table:imagenet64}
\begin{tabular*}{\linewidth}{@{}l@{\extracolsep{\fill}}ccccccr@{}}
\toprule
\textbf{Method} & \shortstack{\textbf{Anti-}\\\textbf{aliasing}} & \shortstack{\textbf{Horizontal}\\\textbf{flips}} & \textbf{Discreteness} & \textbf{Type} & \shortstack{\textbf{Single}\\\textbf{NFE}} & \textbf{NLL} \\
\midrule
Glow~\citep{kinga2018glow} & \xmark & \xmark & Continuous & Normalized & \cmark & 3.81 \\
PixelCNN~\citep{vandenoord2016pixelcnn} & \xmark & \xmark & Discrete & Normalized & \xmark & 3.57 \\
I-DDPM~\citep{nichol2021denoising} & \xmark & \xmark & Continuous & Upper bound & \xmark & 3.54 \\
VDM~\citep{kingma2021variationaldiffusion} & \xmark & \xmark & Discrete & Upper bound & \xmark & 3.40 \\
FM~\citep{lipman2023fm} & \cmark & \xmark & Uniform* & Normalized & \xmark & 3.31 \\
NFDM~\citep{bartosh2024neural} & \xmark & \cmark & Uniform* & Normalized & \xmark & 3.20 \\
TarFlow~\citep{zhai2025tarflow} & \xmark & \cmark & Uniform & Normalized & \cmark & 2.99 \\
Dual SM ~\citep{guth2025dual} & \cmark & \cmark & Continuous & Estimate & \cmark & 3.36 \\
\midrule
Ours (\tstnceSelfDsm) & \cmark & \cmark & Continuous & Estimate & \cmark & 2.94 \\
\bottomrule
\end{tabular*}
\end{table}

\begin{figure}[h]
    \centering
    \includegraphics[width=0.9\linewidth]{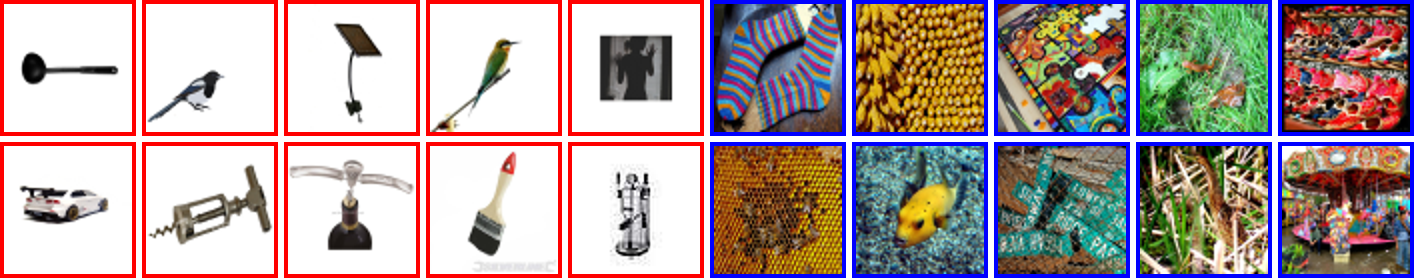}
    \caption{ImageNet test set images with the highest densities (left) and the lowest densities (right) as identified using our model. Higher likelihood images appear to have higher frequencies manifested by sharper contours and shapes, while lower likelihood images appear to have lower frequencies in the form of regular geometric patterns. This is consistent with the literature~\citep{guth2025dual}. 
    }
    \label{fig:imagenet-high-low}
\end{figure}

As a comparison against the state-of-the-art in larger scale, we repeat the experiment on ImageNet64 \citep{chrabaszcz2017imagenet64} using our strongest method as identified based on the previous results, \tstnceSelfDsm. In Table~\ref{table:imagenet64}, we can see that our \tstnceSelfDsm is competitive with Flow Matching~\citep{lipman2023fm} and outperforms approaches such Dual Score-Matching~\citep{guth2025dual}, which similarly models the data density as an energy-based model. We verify in Figure~\ref{fig:imagenet-high-low} that our learnt model assigns high probability to sharper images and lower probability to blurrier images.

\subsection{Molecules}
\label{sec:molecules-exp}

We wish to model a target distribution (\textit{a.k.a.} equilibrium distribution) over configurations of a molecule using an EBM parameterized with a neural network. Each data point in that space is the collection of three-dimensional positions of the atoms, or rather groups of atoms called chains, that make up the molecule. We consider two ``coarse-grained" molecular systems: Alanine dipeptide (ALDP) and Chignolin. ALDP is represented by $5$ chains and Chignolin by $10$ chains; because each chain has three-dimensional spatial coordinates, this leads to $15$-dimensional and $30$-dimensional data, respectively. 
We use the datasets provided by \citet{plainer2025consistent} and \citet{doi:10.1126/science.1208351}, largely adapting the same neural network and experimental settings as \citet{plainer2025consistent}.

We compare \tnce and \tstnceSelf, possibly adding DSM loss, against another spatiotemporal method from~\citet{plainer2025consistent} based on the Fokker-Planck Equation (FPE).
FPE consists in learning \textit{spatial} differences of the energy using DSM, and \textit{temporal} differences of the energy by minimizing a PINN loss that effectively forces the energy to follow a certain Partial Differential Equation.
Although technically sound, FPE is computationally expensive due to multiple additional network evaluations. 
For evaluation, we compare ground-truth samples $x \sim p_d(x | 1)$ from the the dataset and samples from our model $x \sim p_{\theta}(x | 1)$ obtained through two score-based samplers, the Langevin sampler which uses the final score $\nabla \log p_{\theta}(x | 1)$~\citep{eastman2024openmm} and the diffusion sampler which uses the intermediate scores $\nabla \log p_{\theta}(x | t)$~\citep{song2021sde}. 
We then project the ground-truth and model samples onto the $2$-dimensional subspace, and estimate the densities of the empirical samples to compute the Potential of Mean Force (PMF) errors. We additionally estimate the running times of the different methods.

The results are reported in Figure~\ref{fig:molecule-results}. Jensen-Shannon divergence, another metric employed by \citet{plainer2025consistent}, largely follows the same trend as PMF; we report detailed results in Appendix~\ref{app:sec:molecule-additional-results}.
In general, training jointly with DSM consistently improves performance while introducing limited computational overhead possibly due to the smaller scale of these problems.  On ALDP, both \tnceDsm and \tstnceSelfDsm achieve better results than FPE. On Chignolin, they achieve comparable accuracies when compared against FPE at roughly half the training cost. 

\begin{figure}
    \centering
    \includegraphics[width=\linewidth]{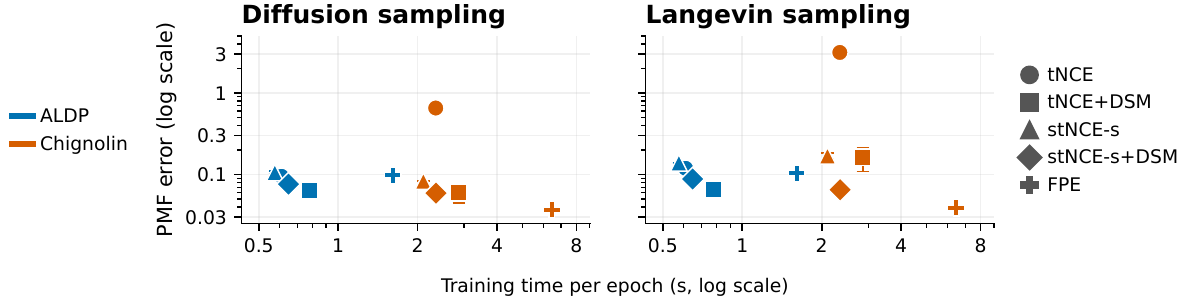}
    \caption{Diffusion and Langevin PMF versus training time. On ALDP dataset, both \tnce and \tstnceSelf yields strong performance compared to FPE. On the higher dimensional Chignolin dataset, \tstnceSelf yields better performance than \tnce, either with or without DSM. Notably, \tstnceSelfDsm provides comparable accuracy at roughly half the training cost.}
    \label{fig:molecule-results}
\end{figure}

\section{Discussion}

In conclusion, we show that despite their diversity, most energy-based model estimation methods fall into three categories: spatial, temporal, and spatiotemporal. We identify fundamental failure modes of purely spatial and temporal approaches, and show how spatiotemporal methods overcome them. Our unified binary classification framework recovers many existing methods and yields new objectives with near state-of-the-art performance across synthetic, image, and molecular data.

\paragraph{Limitations and future directions}
While models trained using NCE loss functions can accurately recover the energies, the scores may not be learned well compared to the score matching variants. Future work could explore better samplers for models trained with NCE losses that either do not rely on gradient information or could better handle noisy gradients. The perturbation kernel is an important design choice but is currently left as heuristics. The strongest results typically rely on a separate DSM loss, and the best stand-alone variant does not reach those numbers on its own.

\begin{ack}
The authors wish to acknowledge Florentin Guth for interesting discussions on Dual Score Matching~\citep{guth2025dual}.
Hanlin Yu and Arto Klami acknowledge the research environment provided by ELLIS Institute Finland, and were supported by the Research Council of Finland Flagship programme: Finnish Center for Artificial Intelligence FCAI.
RuiKang OuYang acknowledges the UK Engineering and Physical Sciences Research Council (EPSRC) grant EP/L016516/1 for the University of Cambridge Centre for Doctoral Training, the Cambridge Centre for Analysis.
Michael U. Gutmann was supported in part by Gen AI - the AI Hub for Generative Models, funded by EPSRC (EP/Y028805/1).
The authors wish to acknowledge CSC - IT Center for Science, Finland, for computational resources, and acknowledge CSC - IT Center for Science, Finland for awarding this project access to the LUMI supercomputer, owned by the EuroHPC Joint Undertaking, hosted by CSC(Finland) and the LUMI consortium through CSC - IT Center for Science, Finland.
The authors declare no competing interests.
\end{ack}

\bibliography{references}

\begin{thebibliography}{86}
\providecommand{\natexlab}[1]{#1}
\providecommand{\url}[1]{\texttt{#1}}
\expandafter\ifx\csname urlstyle\endcsname\relax
  \providecommand{\doi}[1]{doi: #1}\else
  \providecommand{\doi}{doi: \begingroup \urlstyle{rm}\Url}\fi

\bibitem[Aggarwal et~al.(2025)Aggarwal, Chen, Boffi, and Koes]{aggarwal2025boltznce}
R.~Aggarwal, J.~Chen, N.~M. Boffi, and D.~Koes.
\newblock {BoltzNCE}: {Learning} likelihoods for boltzmann generation with stochastic interpolants and noise contrastive estimation.
\newblock In \emph{Conference on neural information processing systems}, 2025.

\bibitem[Ai et~al.(2026)Ai, He, Gu, Salakhutdinov, Kolter, Boffi, and Simchowitz]{ai2026f2d2}
X.~Ai, Y.~He, A.~Gu, R.~Salakhutdinov, J.~Kolter, N.~Boffi, and M.~Simchowitz.
\newblock Joint distillation for fast likelihood evaluation and sampling in flow-based models.
\newblock In \emph{International Conference on Learning Representations}, 2026.

\bibitem[Albergo et~al.(2025)Albergo, Boffi, and Vanden-Eijnden]{albergo2025stochasticinterpolant}
M.~Albergo, N.~M. Boffi, and E.~Vanden-Eijnden.
\newblock Stochastic interpolants: A unifying framework for flows and diffusions.
\newblock \emph{Journal of Machine Learning Research}, 26\penalty0 (209):\penalty0 1--80, 2025.

\bibitem[Ansel et~al.(2024)Ansel, Yang, He, Gimelshein, Jain, Voznesensky, Bao, Bell, Berard, Burovski, Chauhan, Chourdia, Constable, Desmaison, DeVito, Ellison, Feng, Gong, Gschwind, Hirsh, Huang, Kalambarkar, Kirsch, Lazos, Lezcano, Liang, Liang, Lu, Luk, Maher, Pan, Puhrsch, Reso, Saroufim, Siraichi, Suk, Suo, Tillet, Wang, Wang, Wen, Zhang, Zhao, Zhou, Zou, Mathews, Chanan, Wu, and Chintala]{ansel2024pytorch}
J.~Ansel, E.~Yang, H.~He, N.~Gimelshein, A.~Jain, M.~Voznesensky, B.~Bao, P.~Bell, D.~Berard, E.~Burovski, G.~Chauhan, A.~Chourdia, W.~Constable, A.~Desmaison, Z.~DeVito, E.~Ellison, W.~Feng, J.~Gong, M.~Gschwind, B.~Hirsh, S.~Huang, K.~Kalambarkar, L.~Kirsch, M.~Lazos, M.~Lezcano, Y.~Liang, J.~Liang, Y.~Lu, C.~Luk, B.~Maher, Y.~Pan, C.~Puhrsch, M.~Reso, M.~Saroufim, M.~Y. Siraichi, H.~Suk, M.~Suo, P.~Tillet, E.~Wang, X.~Wang, W.~Wen, S.~Zhang, X.~Zhao, K.~Zhou, R.~Zou, A.~Mathews, G.~Chanan, P.~Wu, and S.~Chintala.
\newblock {PyTorch} 2: {Faster} machine learning through dynamic python bytecode transformation and graph compilation.
\newblock In \emph{{ACM} international conference on architectural support for programming languages and operating systems, volume 2 ({ASPLOS} '24)}. ACM, Apr. 2024.

\bibitem[Bartosh et~al.(2024)Bartosh, Vetrov, and Naesseth]{bartosh2024neural}
G.~Bartosh, D.~Vetrov, and C.~A. Naesseth.
\newblock Neural flow diffusion models: Learnable forward process for improved diffusion modelling.
\newblock In \emph{Annual Conference on Neural Information Processing Systems}, 2024.

\bibitem[Batatia et~al.(2022)Batatia, Kovacs, Simm, Ortner, and Csanyi]{batatia2022macehigherorderequivariant}
I.~Batatia, D.~P. Kovacs, G.~Simm, C.~Ortner, and G.~Csanyi.
\newblock {MACE}: {Higher} {Order} {Equivariant} {Message} {Passing} {Neural} {Networks} for {Fast} and {Accurate} {Force} {Fields}.
\newblock In \emph{Advances in {Neural} {Information} {Processing} {Systems}}, volume~35, pages 11423--11436. Curran Associates, Inc., 2022.

\bibitem[Behrmann et~al.(2019)Behrmann, Grathwohl, Chen, Duvenaud, and Jacobsen]{behrmann2019iresnet}
J.~Behrmann, W.~Grathwohl, R.~T.~Q. Chen, D.~Duvenaud, and J.-H. Jacobsen.
\newblock Invertible residual networks.
\newblock In \emph{International Conference on Machine Learning}, volume~97, pages 573--582. PMLR, 09--15 Jun 2019.

\bibitem[Bengio et~al.(2013)Bengio, Courville, and Vincent]{bengio2013repre}
Y.~Bengio, A.~Courville, and P.~Vincent.
\newblock Representation {Learning}: {A} {Review} and {New} {Perspectives}.
\newblock \emph{IEEE Transactions on Pattern Analysis and Machine Intelligence}, 35\penalty0 (8):\penalty0 1798--1828, 2013.

\bibitem[Biewald(2020)]{wandb2020}
L.~Biewald.
\newblock Experiment tracking with weights and biases, 2020.
\newblock URL \url{https://www.wandb.com/}.

\bibitem[Boffi et~al.(2025)Boffi, Albergo, and Vanden-Eijnden]{boffi2025how}
N.~Boffi, M.~Albergo, and E.~Vanden-Eijnden.
\newblock How to build a consistency model: {Learning} flow maps via self-distillation.
\newblock In \emph{Advances in {Neural} {Information} {Processing} {Systems}}, volume 38, Main Conference, pages 33346--33382. Curran Associates, Inc., 2025.

\bibitem[Bradbury et~al.(2018)Bradbury, Frostig, Hawkins, Johnson, Katariya, Leary, Maclaurin, Necula, Paszke, VanderPlas, Wanderman-Milne, and Zhang]{jax2018github}
J.~Bradbury, R.~Frostig, P.~Hawkins, M.~J. Johnson, Y.~Katariya, C.~Leary, D.~Maclaurin, G.~Necula, A.~Paszke, J.~VanderPlas, S.~Wanderman-Milne, and Q.~Zhang.
\newblock {JAX}: composable transformations of {Python}+{NumPy} programs, 2018.
\newblock URL \url{http://github.com/jax-ml/jax}.

\bibitem[Ceylan(2017)]{ceylan2017cncethesis}
C.~Ceylan.
\newblock Conditional noise-contrastive estimation: With application to natural image statistics.
\newblock Master’s thesis, KTH, School of Computer Science and Communication (CSC), Stockholm, Sweden, 2017.
\newblock Independent thesis, advanced level (20 credits / 30 HE credits).

\bibitem[Ceylan and Gutmann(2018)]{ceylan2018cnce}
C.~Ceylan and M.~U. Gutmann.
\newblock Conditional {Noise}-{Contrastive} {Estimation} of {Unnormalised} {Models}.
\newblock In \emph{{International} {Conference} on {Machine} {Learning}}, volume~80, pages 726--734. PMLR, July 2018.

\bibitem[Chehab et~al.(2023)Chehab, Hyvarinen, and Risteski]{chehab2023provable}
O.~Chehab, A.~Hyvarinen, and A.~Risteski.
\newblock Provable benefits of annealing for estimating normalizing constants: Importance sampling, noise-contrastive estimation, and beyond.
\newblock In \emph{Conference on Neural Information Processing Systems}, 2023.

\bibitem[Chen et~al.(2018)Chen, Rubanova, Bettencourt, and Duvenaud]{chen2018neuralordinarydifferentialequations}
R.~T.~Q. Chen, Y.~Rubanova, J.~Bettencourt, and D.~K. Duvenaud.
\newblock Neural {Ordinary} {Differential} {Equations}.
\newblock In \emph{Advances in {Neural} {Information} {Processing} {Systems}}, volume~31. Curran Associates, Inc., 2018.

\bibitem[Chen et~al.(2025{\natexlab{a}})Chen, Li, Li, Xu, Lin, Yang, Zeng, Paisley, and Zhao]{chen2025diffusionsecantalignmentscorebased}
W.~Chen, S.~Li, J.~Li, J.~Xu, Z.~Lin, J.~Yang, D.~Zeng, J.~Paisley, and Q.~Zhao.
\newblock Diffusion {Secant} {Alignment} for {Score}-{Based} {Density} {Ratio} {Estimation}, 2025{\natexlab{a}}.
\newblock arXiv: 2509.04852 [stat.ML].

\bibitem[Chen et~al.(2025{\natexlab{b}})Chen, Li, Li, Yang, Paisley, and Zeng]{chen2025diffusiondre}
W.~Chen, S.~Li, J.~Li, J.~Yang, J.~Paisley, and D.~Zeng.
\newblock Dequantified diffusion-schrödinger bridge for density ratio estimation.
\newblock In \emph{International Conference on Machine Learning}, volume 267, pages 8427--8452. PMLR, 13--19 Jul 2025{\natexlab{b}}.

\bibitem[Chen et~al.(2026)Chen, Li, Li, Lin, Yang, Paisley, and Zeng]{chen2026timescore}
W.~Chen, J.~Li, S.~Li, Z.~Lin, J.~Yang, J.~Paisley, and D.~Zeng.
\newblock A minimum variance path principle for accurate and stable score-based density ratio estimation.
\newblock In \emph{International Conference on Learning Representations}, 2026.

\bibitem[Chewi et~al.(2025)Chewi, Kalavasis, Mehrotra, and Montasser]{chewi2025ebmdiffusion}
S.~Chewi, A.~Kalavasis, A.~Mehrotra, and O.~Montasser.
\newblock {DDPM} score matching is asymptotically efficient.
\newblock In \emph{ICLR 2025 Workshop on Deep Generative Model in Machine Learning: Theory, Principle and Efficacy}, 2025.

\bibitem[Choi et~al.(2022)Choi, Meng, Song, and Ermon]{choi2022densityratio}
K.~Choi, C.~Meng, Y.~Song, and S.~Ermon.
\newblock Density ratio estimation via infinitesimal classification.
\newblock In \emph{International Conference on Artificial Intelligence and Statistics}, volume 151, pages 2552--2573. PMLR, 28--30 Mar 2022.

\bibitem[Chrabaszcz et~al.(2017)Chrabaszcz, Loshchilov, and Hutter]{chrabaszcz2017imagenet64}
P.~Chrabaszcz, I.~Loshchilov, and F.~Hutter.
\newblock A downsampled variant of imagenet as an alternative to the cifar datasets, 2017.

\bibitem[Doucet et~al.(2022)Doucet, Grathwohl, Matthews, and Strathmann]{doucet2022diffusionais}
A.~Doucet, W.~Grathwohl, A.~G. Matthews, and H.~Strathmann.
\newblock Score-based diffusion meets annealed importance sampling.
\newblock In \emph{Advances in Neural Information Processing Systems}, volume~35, pages 21482--21494. Curran Associates, Inc., 2022.

\bibitem[Du and Mordatch(2019)]{du2019implicit}
Y.~Du and I.~Mordatch.
\newblock Implicit generation and modeling with energy-based models.
\newblock In \emph{international conference on neural information processing systems}. Curran Associates Inc., Red Hook, NY, USA, 2019.

\bibitem[Du et~al.(2023)Du, Durkan, Strudel, Tenenbaum, Dieleman, Fergus, Sohl-Dickstein, Doucet, and Grathwohl]{du2023rrr}
Y.~Du, C.~Durkan, R.~Strudel, J.~B. Tenenbaum, S.~Dieleman, R.~Fergus, J.~Sohl-Dickstein, A.~Doucet, and W.~S. Grathwohl.
\newblock Reduce, {Reuse}, {Recycle}: {Compositional} {Generation} with {Energy}-{Based} {Diffusion} {Models} and {MCMC}.
\newblock In \emph{{International} {Conference} on {Machine} {Learning}}, volume 202, pages 8489--8510. PMLR, July 2023.

\bibitem[Durumeric et~al.(2024)Durumeric, Chen, Noé, and Clementi]{durumeric2024learningdataefficientcoarsegrained}
A.~E.~P. Durumeric, Y.~Chen, F.~Noé, and C.~Clementi.
\newblock Learning data efficient coarse-grained molecular dynamics from forces and noise, 2024.
\newblock arXiv: 2407.01286 [physics.bio-ph].

\bibitem[Eastman et~al.(2024)Eastman, Galvelis, Peláez, Abreu, Farr, Gallicchio, Gorenko, Henry, Hu, Huang, Krämer, Michel, Mitchell, Pande, Rodrigues, Rodriguez-Guerra, Simmonett, Singh, Swails, Turner, Wang, Zhang, Chodera, De~Fabritiis, and Markland]{eastman2024openmm}
P.~Eastman, R.~Galvelis, R.~P. Peláez, C.~R.~A. Abreu, S.~E. Farr, E.~Gallicchio, A.~Gorenko, M.~M. Henry, F.~Hu, J.~Huang, A.~Krämer, J.~Michel, J.~A. Mitchell, V.~S. Pande, J.~P. Rodrigues, J.~Rodriguez-Guerra, A.~C. Simmonett, S.~Singh, J.~Swails, P.~Turner, Y.~Wang, I.~Zhang, J.~D. Chodera, G.~De~Fabritiis, and T.~E. Markland.
\newblock {OpenMM} 8: {Molecular} {Dynamics} {Simulation} with {Machine} {Learning} {Potentials}.
\newblock \emph{The Journal of Physical Chemistry B}, 128\penalty0 (1):\penalty0 109--116, Jan. 2024.
\newblock ISSN 1520-6106.

\bibitem[Gao et~al.(2020)Gao, Nijkamp, Kingma, Xu, Dai, and Wu]{gao2020flowcontrastive}
R.~Gao, E.~Nijkamp, D.~P. Kingma, Z.~Xu, A.~M. Dai, and Y.~N. Wu.
\newblock Flow contrastive estimation of energy-based models.
\newblock In \emph{2020 {IEEE}/{CVF} conference on computer vision and pattern recognition ({CVPR})}, pages 7515--7525, 2020.

\bibitem[Gao et~al.(2021)Gao, Song, Poole, Wu, and Kingma]{gao2021learning}
R.~Gao, Y.~Song, B.~Poole, Y.~N. Wu, and D.~P. Kingma.
\newblock Learning energy-based models by diffusion recovery likelihood.
\newblock In \emph{International conference on learning representations}, 2021.

\bibitem[Gonzalez et~al.(2023)Gonzalez, Fernandez~Pinto, Tran, Gherbi, Hajri, and Masmoudi]{gonzalez2023seeds}
M.~Gonzalez, N.~Fernandez~Pinto, T.~Tran, E.~Gherbi, H.~Hajri, and N.~Masmoudi.
\newblock {SEEDS}: {Exponential} {SDE} solvers for fast high-quality sampling from diffusion models.
\newblock In \emph{Advances in neural information processing systems}, volume~36, pages 68061--68120. Curran Associates, Inc., 2023.

\bibitem[Goodfellow et~al.(2020)Goodfellow, Pouget-Abadie, Mirza, Xu, Warde-Farley, Ozair, Courville, and Bengio]{goodfellow2020gan}
I.~Goodfellow, J.~Pouget-Abadie, M.~Mirza, B.~Xu, D.~Warde-Farley, S.~Ozair, A.~Courville, and Y.~Bengio.
\newblock Generative {Adversarial} {Networks}.
\newblock \emph{Commun. ACM}, 63\penalty0 (11):\penalty0 139--144, Oct. 2020.
\newblock ISSN 0001-0782.
\newblock Place: New York, NY, USA.

\bibitem[Grathwohl et~al.(2019)Grathwohl, Chen, Bettencourt, Sutskever, and Duvenaud]{grathwohl2018scalable}
W.~Grathwohl, R.~T.~Q. Chen, J.~Bettencourt, I.~Sutskever, and D.~Duvenaud.
\newblock {FFJORD}: {Free}-{Form} {Continuous} {Dynamics} for {Scalable} {Reversible} {Generative} {Models}.
\newblock In \emph{International {Conference} on {Learning} {Representations}}, 2019.

\bibitem[Grenioux et~al.(2023)Grenioux, Moulines, and Gabrié]{grenioux2023balanced}
L.~Grenioux, E.~Moulines, and M.~Gabrié.
\newblock Balanced training of energy-based models with adaptive flow sampling.
\newblock In \emph{{ICML} 2023 workshop on structured probabilistic inference \& generative modeling}, 2023.

\bibitem[Gugger et~al.(2022)Gugger, Debut, Wolf, Schmid, Mueller, Mangrulkar, Sun, and Bossan]{accelerate2022}
S.~Gugger, L.~Debut, T.~Wolf, P.~Schmid, Z.~Mueller, S.~Mangrulkar, M.~Sun, and B.~Bossan.
\newblock Accelerate: {Training} and inference at scale made simple, efficient and adaptable., 2022.
\newblock URL \url{https://github.com/huggingface/accelerate}.

\bibitem[Guth et~al.(2025)Guth, Kadkhodaie, and Simoncelli]{guth2025dual}
F.~Guth, Z.~Kadkhodaie, and E.~Simoncelli.
\newblock Learning normalized image densities via dual score matching.
\newblock In \emph{Advances in {Neural} {Information} {Processing} {Systems}}, volume 38, Main Conference, pages 89796--89826. Curran Associates, Inc., 2025.

\bibitem[Gutmann and Hirayama(2011)]{gutmann2011bregman}
M.~U. Gutmann and J.-i. Hirayama.
\newblock Bregman divergence as general framework to estimate unnormalized statistical models.
\newblock In \emph{conference on uncertainty in artificial intelligence}, {UAI}'11, pages 283--290, Barcelona, Spain, 2011. AUAI Press.
\newblock ISBN 978-0-9749039-7-2.

\bibitem[Gutmann and Hyvärinen(2012)]{gutmann2012nce_jmlr}
M.~U. Gutmann and A.~Hyvärinen.
\newblock Noise-{Contrastive} {Estimation} of {Unnormalized} {Statistical} {Models}, with {Applications} to {Natural} {Image} {Statistics}.
\newblock \emph{Journal of Machine Learning Research}, 13\penalty0 (11):\penalty0 307--361, 2012.

\bibitem[He et~al.(2026)He, Hern{\'a}ndez-Lobato, Du, and Vargas]{he2026rne}
J.~He, J.~M. Hern{\'a}ndez-Lobato, Y.~Du, and F.~Vargas.
\newblock {RNE}: plug-and-play diffusion inference-time control and energy-based training.
\newblock In \emph{International Conference on Learning Representations}, 2026.

\bibitem[Hinton(2002)]{hinton2002cd}
G.~E. Hinton.
\newblock Training products of experts by minimizing contrastive divergence.
\newblock \emph{Neural Computation}, 14\penalty0 (8):\penalty0 1771--1800, 2002.

\bibitem[Ho et~al.(2020)Ho, Jain, and Abbeel]{ho2020ddpm}
J.~Ho, A.~Jain, and P.~Abbeel.
\newblock Denoising {Diffusion} {Probabilistic} {Models}.
\newblock In \emph{Advances in {Neural} {Information} {Processing} {Systems}}, volume~33, pages 6840--6851. Curran Associates, Inc., 2020.

\bibitem[Holderrieth et~al.(2026)Holderrieth, Singer, Jaakkola, Chen, Lipman, and Karrer]{holderrieth2026glass}
P.~Holderrieth, U.~Singer, T.~Jaakkola, R.~T.~Q. Chen, Y.~Lipman, and B.~Karrer.
\newblock {GLASS} flows: {Efficient} inference for reward alignment of flow and diffusion models.
\newblock In \emph{International conference on learning representations}, 2026.

\bibitem[Husic et~al.(2020)Husic, Charron, Lemm, Wang, Pérez, Majewski, Krämer, Chen, Olsson, de~Fabritiis, Noé, and Clementi]{Husic_2020}
B.~E. Husic, N.~E. Charron, D.~Lemm, J.~Wang, A.~Pérez, M.~Majewski, A.~Krämer, Y.~Chen, S.~Olsson, G.~de~Fabritiis, F.~Noé, and C.~Clementi.
\newblock Coarse graining molecular dynamics with graph neural networks.
\newblock \emph{The Journal of Chemical Physics}, 153\penalty0 (19), Nov. 2020.
\newblock ISSN 1089-7690.

\bibitem[Hyvarinen(2007)]{hyvarinen2007connections}
A.~Hyvarinen.
\newblock Connections between score matching, contrastive divergence, and pseudolikelihood for continuous-valued variables.
\newblock \emph{IEEE Transactions on Neural Networks}, 18\penalty0 (5):\penalty0 1529--1531, 2007.

\bibitem[Hyvärinen(2005)]{hyvarinen2005sm}
A.~Hyvärinen.
\newblock Estimation of {Non}-{Normalized} {Statistical} {Models} by {Score} {Matching}.
\newblock \emph{Journal of Machine Learning Research}, 6\penalty0 (24):\penalty0 695--709, 2005.

\bibitem[Karras et~al.(2022)Karras, Aittala, Aila, and Laine]{karras2022edm}
T.~Karras, M.~Aittala, T.~Aila, and S.~Laine.
\newblock Elucidating the {Design} {Space} of {Diffusion}-{Based} {Generative} {Models}.
\newblock In \emph{Advances in {Neural} {Information} {Processing} {Systems}}, volume~35, pages 26565--26577. Curran Associates, Inc., 2022.

\bibitem[Kingma et~al.(2021)Kingma, Salimans, Poole, and Ho]{kingma2021variationaldiffusion}
D.~Kingma, T.~Salimans, B.~Poole, and J.~Ho.
\newblock Variational diffusion models.
\newblock In \emph{Advances in Neural Information Processing Systems}, volume~34, pages 21696--21707. Curran Associates, Inc., 2021.

\bibitem[Kingma and Ba(2015)]{kingma2015adam}
D.~P. Kingma and J.~Ba.
\newblock Adam: {A} {Method} for {Stochastic} {Optimization}.
\newblock In \emph{{International} {Conference} on {Learning} {Representations}}, 2015.

\bibitem[Kingma and Dhariwal(2018)]{kinga2018glow}
D.~P. Kingma and P.~Dhariwal.
\newblock Glow: Generative flow with invertible 1x1 convolutions.
\newblock In \emph{Advances in Neural Information Processing Systems}, volume~31. Curran Associates, Inc., 2018.

\bibitem[Koehler et~al.(2023)Koehler, Heckett, and Risteski]{koehler2023scorematching}
F.~Koehler, A.~Heckett, and A.~Risteski.
\newblock Statistical efficiency of score matching: The view from isoperimetry.
\newblock In \emph{The Eleventh International Conference on Learning Representations}, 2023.

\bibitem[Lai et~al.(2023)Lai, Takida, Murata, Uesaka, Mitsufuji, and Ermon]{lai2023fpdiffusion}
C.-H. Lai, Y.~Takida, N.~Murata, T.~Uesaka, Y.~Mitsufuji, and S.~Ermon.
\newblock {FP}-diffusion: Improving score-based diffusion models by enforcing the underlying score fokker-planck equation.
\newblock In \emph{International Conference on Machine Learning}, volume 202, pages 18365--18398. PMLR, 23--29 Jul 2023.

\bibitem[LeCun et~al.(2010)LeCun, Cortes, and Burges]{lecun2010mnist}
Y.~LeCun, C.~Cortes, and C.~Burges.
\newblock {MNIST} handwritten digit database.
\newblock \emph{ATT Labs [Online]. Available: http://yann.lecun.com/exdb/mnist}, 2, 2010.

\bibitem[Li and He(2026)]{li2026basicsletdenoisinggenerative}
T.~Li and K.~He.
\newblock Back to {Basics}: {Let} {Denoising} {Generative} {Models} {Denoise}.
\newblock In \emph{Proceedings of the {IEEE}/{CVF} {Conference} on {Computer} {Vision} and {Pattern} {Recognition} ({CVPR})}, pages 36115--36125, June 2026.

\bibitem[Lindorff-Larsen et~al.(2011)Lindorff-Larsen, Piana, Dror, and Shaw]{doi:10.1126/science.1208351}
K.~Lindorff-Larsen, S.~Piana, R.~O. Dror, and D.~E. Shaw.
\newblock How fast-folding proteins fold.
\newblock \emph{Science}, 334\penalty0 (6055):\penalty0 517--520, 2011.

\bibitem[Lipman et~al.(2023)Lipman, Chen, Ben-Hamu, Nickel, and Le]{lipman2023fm}
Y.~Lipman, R.~T.~Q. Chen, H.~Ben-Hamu, M.~Nickel, and M.~Le.
\newblock Flow {Matching} for {Generative} {Modeling}.
\newblock In \emph{The {Eleventh} {International} {Conference} on {Learning} {Representations}}, 2023.

\bibitem[Lipman et~al.(2024)Lipman, Havasi, Holderrieth, Shaul, Le, Karrer, Chen, Lopez-Paz, Ben-Hamu, and Gat]{lipman2024flowmatchingguidecode}
Y.~Lipman, M.~Havasi, P.~Holderrieth, N.~Shaul, M.~Le, B.~Karrer, R.~T.~Q. Chen, D.~Lopez-Paz, H.~Ben-Hamu, and I.~Gat.
\newblock Flow matching guide and code, 2024.
\newblock arXiv: 2412.06264 [cs.LG].

\bibitem[Loshchilov and Hutter(2019)]{loshchilov2018adamw}
I.~Loshchilov and F.~Hutter.
\newblock Decoupled weight decay regularization.
\newblock In \emph{International conference on learning representations}, 2019.

\bibitem[Lu et~al.(2022)Lu, Zhou, Bao, Chen, LI, and Zhu]{lu2022dpm-solver}
C.~Lu, Y.~Zhou, F.~Bao, J.~Chen, C.~LI, and J.~Zhu.
\newblock {DPM}-solver: a fast {ODE} solver for diffusion probabilistic model sampling in around 10 steps.
\newblock In \emph{Advances in neural information processing systems}, volume~35, pages 5775--5787. Curran Associates, Inc., 2022.

\bibitem[M{\'a}t{\'e} et~al.(2022)M{\'a}t{\'e}, Klein, Golling, and Fleuret]{mate2022flowification}
B.~M{\'a}t{\'e}, S.~Klein, T.~Golling, and F.~Fleuret.
\newblock Flowification: Everything is a normalizing flow.
\newblock In \emph{Advances in Neural Information Processing Systems}, 2022.

\bibitem[Nichol and Dhariwal(2021)]{nichol2021denoising}
A.~Q. Nichol and P.~Dhariwal.
\newblock Improved denoising diffusion probabilistic models.
\newblock In \emph{International Conference on Machine Learning}, volume 139, pages 8162--8171. PMLR, 18--24 Jul 2021.

\bibitem[Noé et~al.(2019)Noé, Olsson, Köhler, and Wu]{noe2019bg}
F.~Noé, S.~Olsson, J.~Köhler, and H.~Wu.
\newblock Boltzmann generators: {Sampling} equilibrium states of many-body systems with deep learning.
\newblock \emph{Science}, 365\penalty0 (6457):\penalty0 eaaw1147, 2019.

\bibitem[Ouyang et~al.(2026)Ouyang, Grenioux, and Hernández-Lobato]{ouyang2026diffusiveclassificationlosslearning}
R.~Ouyang, L.~Grenioux, and J.~M. Hernández-Lobato.
\newblock A {Diffusive} {Classification} {Loss} for {Learning} {Energy}-based {Generative} {Models}.
\newblock In \emph{Proceedings of the 43rd {International} {Conference} on {Machine} {Learning}}, volume 306 of \emph{Proceedings of {Machine} {Learning} {Research}}, pages 95197--95252. PMLR, July 2026.

\bibitem[OuYang et~al.(2026)OuYang, Yu, Ai, He, Boffi, Ravikumar, Hernández-Lobato, Simchowitz, Miller, and Chehab]{ouyang2026scallop}
R.~OuYang, H.~Yu, X.~Ai, Y.~He, N.~M. Boffi, P.~Ravikumar, J.~M. Hernández-Lobato, M.~Simchowitz, B.~K. Miller, and O.~Chehab.
\newblock Few-{Step} {Boltzmann} {Generators} via {Scalable} {Likelihood} {Flow} {Maps}.
\newblock In \emph{{ICML} 2026 {Workshop} on {Structured} {Probabilistic} {Inference} \& {Generative} {Modeling}}, 2026.

\bibitem[Plainer et~al.(2025)Plainer, Wu, Klein, Günnemann, and Noe]{plainer2025consistent}
M.~Plainer, H.~Wu, L.~Klein, S.~Günnemann, and F.~Noe.
\newblock Consistent {Sampling} and {Simulation}: {Molecular} {Dynamics} with {Energy}-{Based} {Diffusion} {Models}.
\newblock In \emph{Advances in {Neural} {Information} {Processing} {Systems}}, volume 38, Main Conference, pages 24460--24505. Curran Associates, Inc., 2025.

\bibitem[Qin and Risteski(2024)]{qin2024ebmdiffusion}
Y.~Qin and A.~Risteski.
\newblock Fit like you sample: Sample-efficient generalized score matching from fast mixing diffusions.
\newblock In \emph{Conference on Learning Theory}, volume 247, pages 4413--4457. PMLR, 30 Jun--03 Jul 2024.

\bibitem[Rhodes et~al.(2020)Rhodes, Xu, and Gutmann]{rhodes2020telescoping}
B.~Rhodes, K.~Xu, and M.~U. Gutmann.
\newblock Telescoping density-ratio estimation.
\newblock In \emph{Advances in Neural Information Processing Systems}, volume~33, pages 4905--4916. Curran Associates, Inc., 2020.

\bibitem[Skreta et~al.(2025)Skreta, Akhound-Sadegh, Ohanesian, Bondesan, Aspuru-Guzik, Doucet, Brekelmans, Tong, and Neklyudov]{skreta2025feynmankac}
M.~Skreta, T.~Akhound-Sadegh, V.~Ohanesian, R.~Bondesan, A.~Aspuru-Guzik, A.~Doucet, R.~Brekelmans, A.~Tong, and K.~Neklyudov.
\newblock Feynman-kac correctors in diffusion: Annealing, guidance, and product of experts.
\newblock In \emph{International Conference on Machine Learning}, 2025.

\bibitem[Song and Ermon(2019)]{song2019score}
Y.~Song and S.~Ermon.
\newblock Generative {Modeling} by {Estimating} {Gradients} of the {Data} {Distribution}.
\newblock In \emph{Advances in {Neural} {Information} {Processing} {Systems}}, volume~32. Curran Associates, Inc., 2019.

\bibitem[Song and Kingma(2021)]{song2021ebms}
Y.~Song and D.~P. Kingma.
\newblock How to train your energy-based models, 2021.
\newblock arXiv: 2101.03288 [cs.LG].

\bibitem[Song et~al.(2019)Song, Meng, and Ermon]{song2019mintnet}
Y.~Song, C.~Meng, and S.~Ermon.
\newblock Mintnet: Building invertible neural networks with masked convolutions.
\newblock In \emph{Neural Information Processing Systems}, volume~32. Curran Associates, Inc., 2019.

\bibitem[Song et~al.(2020)Song, Garg, Shi, and Ermon]{song2020ssm}
Y.~Song, S.~Garg, J.~Shi, and S.~Ermon.
\newblock Sliced {Score} {Matching}: {A} {Scalable} {Approach} to {Density} and {Score} {Estimation}.
\newblock In \emph{{Uncertainty} in {Artificial} {Intelligence} {Conference}}, volume 115, pages 574--584. PMLR, July 2020.

\bibitem[Song et~al.(2021{\natexlab{a}})Song, Durkan, Murray, and Ermon]{song2021mle}
Y.~Song, C.~Durkan, I.~Murray, and S.~Ermon.
\newblock Maximum {Likelihood} {Training} of {Score}-{Based} {Diffusion} {Models}.
\newblock In \emph{Advances in {Neural} {Information} {Processing} {Systems}}, volume~34, pages 1415--1428. Curran Associates, Inc., 2021{\natexlab{a}}.

\bibitem[Song et~al.(2021{\natexlab{b}})Song, Sohl-Dickstein, Kingma, Kumar, Ermon, and Poole]{song2021sde}
Y.~Song, J.~Sohl-Dickstein, D.~P. Kingma, A.~Kumar, S.~Ermon, and B.~Poole.
\newblock Score-{Based} {Generative} {Modeling} through {Stochastic} {Differential} {Equations}.
\newblock In \emph{International {Conference} on {Learning} {Representations}}, 2021{\natexlab{b}}.

\bibitem[Song et~al.(2023)Song, Dhariwal, Chen, and Sutskever]{song2023consistency}
Y.~Song, P.~Dhariwal, M.~Chen, and I.~Sutskever.
\newblock Consistency {Models}.
\newblock In \emph{{International} {Conference} on {Machine} {Learning}}, {ICML}'23. JMLR.org, 2023.
\newblock Place: Honolulu, Hawaii, USA.

\bibitem[Srivastava et~al.(2023)Srivastava, Han, Xu, Rhodes, and Gutmann]{srivastava2023multi}
A.~Srivastava, S.~Han, K.~Xu, B.~Rhodes, and M.~U. Gutmann.
\newblock Estimating the {Density} {Ratio} between {Distributions} with {High} {Discrepancy} using {Multinomial} {Logistic} {Regression}.
\newblock \emph{Transactions on Machine Learning Research}, 2023.
\newblock ISSN 2835-8856.

\bibitem[Sugiyama et~al.(2012)Sugiyama, Suzuki, and Kanamori]{sugiyama2012density}
M.~Sugiyama, T.~Suzuki, and T.~Kanamori.
\newblock \emph{Density ratio estimation in machine learning}.
\newblock Cambridge University Press, 2012.

\bibitem[Thornton et~al.(2025)Thornton, Bethune, Zhang, Bradley, Nakkiran, and Zhai]{thornton2025ebmdiffusion}
J.~Thornton, L.~Bethune, R.~Zhang, A.~Bradley, P.~Nakkiran, and S.~Zhai.
\newblock Composition and control with distilled energy diffusion models and sequential monte carlo.
\newblock In \emph{{AISTATS} 2025}, 2025.

\bibitem[Tong et~al.(2024)Tong, Malkin, Fatras, Atanackovic, Zhang, Huguet, Wolf, and Bengio]{tong2024sb}
A.~Y. Tong, N.~Malkin, K.~Fatras, L.~Atanackovic, Y.~Zhang, G.~Huguet, G.~Wolf, and Y.~Bengio.
\newblock Simulation-{Free} {Schrödinger} {Bridges} via {Score} and {Flow} {Matching}.
\newblock In \emph{Proceedings of {The} 27th {International} {Conference} on {Artificial} {Intelligence} and {Statistics}}, volume 238 of \emph{Proceedings of {Machine} {Learning} {Research}}, pages 1279--1287. PMLR, May 2024.

\bibitem[Uehara et~al.(2018)Uehara, Matsuda, and Komaki]{uehara2018analysisnoisecontrastiveestimation}
M.~Uehara, T.~Matsuda, and F.~Komaki.
\newblock Analysis of {Noise} {Contrastive} {Estimation} from the {Perspective} of {Asymptotic} {Variance}, 2018.
\newblock arXiv: 1808.07983 [stat.ML].

\bibitem[van~den Oord et~al.(2016)van~den Oord, Kalchbrenner, Vinyals, Espeholt, Graves, and Kavukcuoglu]{vandenoord2016pixelcnn}
A.~van~den Oord, N.~Kalchbrenner, O.~Vinyals, L.~Espeholt, A.~Graves, and K.~Kavukcuoglu.
\newblock Conditional {Image} {Generation} with {PixelCNN} {Decoders}.
\newblock In \emph{Advances in {Neural} {Information} {Processing} {Systems}}, volume~29. Curran Associates, Inc., 2016.

\bibitem[van~der Vaart(1998)]{vandervaart1998asymptotic}
A.~W. van~der Vaart.
\newblock \emph{Asymptotic {Statistics}}.
\newblock Cambridge {Series} in {Statistical} and {Probabilistic} {Mathematics}. Cambridge University Press, Cambridge, UK, 1998.

\bibitem[Vershynin(2018)]{vershynin2018high}
R.~Vershynin.
\newblock \emph{High-{Dimensional} {Probability}: {An} {Introduction} with {Applications} in {Data} {Science}}.
\newblock Cambridge {Series} in {Statistical} and {Probabilistic} {Mathematics}. Cambridge University Press, 2018.

\bibitem[Vincent(2011)]{vincent2011dsm}
P.~Vincent.
\newblock A {Connection} {Between} {Score} {Matching} and {Denoising} {Autoencoders}.
\newblock \emph{Neural Computation}, 23\penalty0 (7):\penalty0 1661--1674, 2011.

\bibitem[Wenliang and Kanagawa(2021)]{wenliang2021blindness}
L.~K. Wenliang and H.~Kanagawa.
\newblock Blindness of score-based methods to isolated components and mixing proportions, 2021.
\newblock \_eprint: 2008.10087.

\bibitem[Yair and Michaeli(2021)]{yair2021contrastive}
O.~Yair and T.~Michaeli.
\newblock Contrastive {Divergence} {Learning} is a {Time} {Reversal} {Adversarial} {Game}.
\newblock In \emph{International {Conference} on {Learning} {Representations}}, 2021.

\bibitem[Yu et~al.(2025{\natexlab{a}})Yu, Gutmann, Klami, and Chehab]{yu2025cnce}
H.~Yu, M.~U. Gutmann, A.~Klami, and O.~Chehab.
\newblock Conditional noise-contrastive estimation of energy-based models by jumping between modes.
\newblock In \emph{{EurIPS} 2025 workshop on principles of generative modeling ({PriGM})}, 2025{\natexlab{a}}.

\bibitem[Yu et~al.(2025{\natexlab{b}})Yu, Klami, Hyvarinen, Korba, and Chehab]{yu2025dre}
H.~Yu, A.~Klami, A.~Hyvarinen, A.~Korba, and O.~Chehab.
\newblock Density {Ratio} {Estimation} with {Conditional} {Probability} {Paths}.
\newblock In \emph{Forty-second {International} {Conference} on {Machine} {Learning}}, 2025{\natexlab{b}}.

\bibitem[Zhai et~al.(2025)Zhai, Zhang, Nakkiran, Berthelot, Gu, Zheng, Chen, Bautista, Jaitly, and Susskind]{zhai2025tarflow}
S.~Zhai, R.~Zhang, P.~Nakkiran, D.~Berthelot, J.~Gu, H.~Zheng, T.~Chen, M.~{\'A}. Bautista, N.~Jaitly, and J.~M. Susskind.
\newblock Normalizing flows are capable generative models.
\newblock In \emph{International Conference on Machine Learning}, 2025.

\end{thebibliography}

\appendix

\appendix
\clearpage

\section*{Appendix}

The appendix is organized as follows.

\renewcommand{\contentsname}{}
\vspace{-1cm}
\addtocontents{toc}{\protect\setcounter{tocdepth}{2}}
\tableofcontents

\newpage
\section{Theoretical results}
\label{app:sec:theory}

In this section, we prove the consistency of the estimator obtained by minimizing the empirical \stnce loss. 
We also derive the simplified expressions of \stnce loss under specific perturbation kernels.

\subsection{Useful lemma on the logistic loss}

The following are useful lemma on how to rewrite the logistic classification loss. 

\begin{lemma}[Logistic loss as a Bregman divergence]
\label{lemma:logistic_loss_as_bregman_divergence}
Consider the logistic loss
\begin{align}
\mathcal{L}_{\mathrm{logistic}}(F)
= 
- \mathbb{E}_{x \sim p_{A}} \left[ 
\log \mathrm{sigmoid}(F(x)) 
\right]
- 
\mathbb{E}_{x \sim p_{B}} \left[ 
\log (1 - \mathrm{sigmoid}(F(x))) 
\right],
\end{align}
which is minimized by
\begin{equation}
F^{*}(x) = \log p_A(x) - \log p_B(x).
\end{equation}
Then, up to an additive constant independent of $F$, the logistic loss can be written as a Bregman divergence between $e^{F^*}$ and $e^F$.
\end{lemma}

\begin{proof}
This is inspired by~\citet{gutmann2011bregman,uehara2018analysisnoisecontrastiveestimation}.

Let $\Phi(u) = u \log u - (1+u)\log(1+u)$. Consider the pointwise Bregman divergence
\begin{align}
D_{\Phi}\big(e^{F^*(x)}, e^{F(x)}\big)
=
\Phi\big(e^{F^*(x)}\big)
-
\Phi\big(e^{F(x)}\big)
-
\Phi'\big(e^{F(x)}\big)\big(e^{F^*(x)} - e^{F(x)}\big).
\end{align}
Integrating against $p_B(x)\,dx$ yields a divergence between functions:
\begin{align}
D_{\Phi}\big(e^{F^*}, e^{F}\big)
=
\int D_{\Phi}\big(e^{F^*(x)}, e^{F(x)}\big)\, p_B(x)\,dx.
\end{align}
Substituting $F^*(x) = \log p_A(x) - \log p_B(x)$, a straightforward calculation shows
\begin{align}
D_{\Phi}\big(e^{F^*}, e^{F}\big)
=
\mathcal{L}_{\mathrm{logistic}}(F)
-
\mathcal{L}_{\mathrm{logistic}}(F^*),
\end{align}
which proves the claim.
\end{proof}

\begin{lemma}[Logistic loss simplified when the model is near zero]
Consider the logistic loss
\begin{align}
    \mathcal{L}(F_{\theta})
    &:=
    -
    \E_{p^{+}(x)}
    \bigl[
        \log \mathrm{sigmoid}(F_{\theta}(x))
    \bigr]
    -
    \E_{p^{-}(x)}
    \bigl[
        \log \bigl(
            1 - \mathrm{sigmoid}(F_{\theta}(x))
        \bigr)
    \bigr].
\end{align}
If $F_\theta(x)$ is close to zero, then
\begin{align}
    \mathcal{L}(F_\theta)
    &=
    \mathcal{L}(0)
    -
    \frac12
    \E_{p^{+}(x)-p^{-}(x)}
    \left[
        F_\theta(x)
    \right]
    +
    \frac18
    \E_{p^{+}(x)+p^{-}(x)}
    \left[
        F_\theta(x)^2
    \right] \\
    &\quad
    +
    \E_{p^{+}(x)+p^{-}(x)}
    \left[
        o(F_\theta(x)^2)
    \right].
\end{align}
\label{lemma:logistic_as_squared_around_zero}
\end{lemma}

\begin{proof}
For fixed $x$, define the pointwise loss
\begin{align}
    \ell(F_{\theta})
    &:=
    -
    p^{+}\log \mathrm{sigmoid}(F_{\theta})
    -
    p^{-}\log \bigl(1-\mathrm{sigmoid}(F_{\theta})\bigr).
\end{align}
A second-order Taylor expansion around $F_{\theta}=0$ gives
\begin{align}
    \ell(F_\theta)
    &=
    \ell(0)
    -
    \frac12
    \bigl(p^{+}-p^{-}\bigr)
    F_\theta
    +
    \frac18
    \bigl(p^{+}+p^{-}\bigr)
    F_\theta^2
    +
    \bigl(p^{+}+p^{-}\bigr)
    o(F_\theta^2).
\end{align}
Integrating over $x$ yields the result.
\end{proof}

\subsection{Non-parametric consistency (Theorem~\ref{theorem:fisher_consistency})}
\label{app:ssec:fisher_consistency}

Recall that the \stnce population loss is the logistic classification objective
\begin{align}
\mathcal{L}_{\mathrm{\stnce}}(F)
=
- \mathbb{E}_{x \sim p_{A}} \big[ \log \mathrm{sigmoid}(F(x)) \big]
- \mathbb{E}_{x \sim p_{B}} \big[ \log (1 - \mathrm{sigmoid}(F(x))) \big].
\end{align}
By Lemma~\ref{lemma:logistic_loss_as_bregman_divergence}, this can be written (up to an additive constant) as
\begin{align}
\mathcal{L}_{\mathrm{\stnce}}(F)
=
D_{\Phi}\big(e^{F^*}, e^{F}\big) + \mathrm{const}.
\end{align}
Hence any minimizer $F^{\theta_{\min}}$ satisfies
\begin{equation}
F^{\theta_{\min}} = F^*.
\end{equation}

Using the definition of the \stnce classifier~\eqref{eq:stnce_classifier}, this yields, for all $(x,t)$ and $(x',t')$ in the support,
\begin{align}
\log p_{\theta_{\min}}(x \mid t) - \log p_{\theta_{\min}}(x' \mid t')
=
\log p_d(x \mid t) - \log p_d(x' \mid t').
\label{eq:stnce-ratio-identity}
\end{align}
The equality is well-defined whenever the classification mixture has positive density.

Since the perturbation kernel has full support, applying Equation~\ref{eq:stnce-ratio-identity} along these transitions shows that
\begin{align}
\log p_{\theta_{\min}}(x \mid t)
=
\log p_d(x \mid t) + C,
\end{align}
where $C$ does not depend on $(x,t)$. Enforcing the constraint $p_\theta(x \mid t=0) = p_0(x)$ fixes $C$, yielding
\begin{align}
p_{\theta_{\min}}(x \mid t) = p_d(x \mid t).
\end{align}

This establishes Fisher consistency of the population objective. 

For the empirical estimator, under standard regularity conditions, empirical consistency follows from standard M-estimation theory~\citep[Theorem 5.7]{vandervaart1998asymptotic}. If the model is identifiable, we further obtain parametric Fisher consistency, i.e., $\theta_{\min} = \theta^*$.

\subsection{Forward-reverse kernel: Infinitesimal Limit (Propositions~\ref{th:sm_mle} and Proposition~\ref{th:total_derivative_matching})}
\label{app:sec:proofs}

\paragraph{Simplifying the loss around zero}

Let
\begin{align}
    p_+(x,t,x',t') &:=
    p_d(x,t)p_n(x',t' \mid x,t),
    \\
    p_-(x,t,x',t') &:=
    p_d(x',t')p_n(x,t \mid x',t').
\end{align}
The logistic loss used in \stnce is
\begin{align}
    \mathcal{L}(F_{\theta})
    &=
    -\E_{p_+(x, t, x', t')}
    [\log \mathrm{sigmoid}(F_{\theta}(x, t, x', t'))] \\
    &\quad
    -\E_{p_-(x, t, x', t')}
    [\log (1 - \mathrm{sigmoid}(F_{\theta}(x, t, x', t')))] .
\end{align}
Using Lemma~\ref{lemma:logistic_as_squared_around_zero}, we have, whenever
$F_\theta$ is close to zero,
\begin{align}
    \mathcal{L}(F_\theta)
    &=
    \mathcal{L}(0)
    -
    \frac12
    \E_{p_+ - p_-}
    \left[
        F_\theta
    \right]
    +
    \frac18
    \E_{p_+ + p_-}
    \left[
        F_\theta^2
    \right]
    +
    \E_{p_+ + p_-}
    \left[
        o(F_\theta^2)
    \right].
\end{align}
In contrast with the expansion around the optimum, this expansion generally
contains a linear term. However, this term vanishes in the forward-reverse
case considered below.

\paragraph{Plugging in the forward-reverse kernel}

Recall that the forward-reverse kernel satisfies
\begin{align}
    p_d(x,t)p_n(x',t'\mid x,t)
    =
    p_d(x',t')p_n(x,t\mid x',t'),
\end{align}
due to the Bayes rule. Equivalently,
\begin{align}
    p_+(x,t,x',t')
    =
    p_-(x,t,x',t').
\end{align}
Therefore, the linear term vanishes:
\begin{align}
    \E_{p_+ - p_-}
    \left[
        F_\theta
    \right]
    =
    0.
\end{align}
Moreover, by taking logs in the forward-reverse identity,
\begin{align}
    \log p_n(x',t'\mid x,t)
    -
    \log p_n(x,t\mid x',t')
    =
    \log p_d(x',t')
    -
    \log p_d(x,t).
\end{align}
Hence
\begin{align}
    F_\theta(x,t,x',t')
    &=
    \log p_\theta(x,t)
    -
    \log p_\theta(x',t')
    \\
    &\quad+
    \log p_n(x',t'\mid x,t)
    -
    \log p_n(x,t\mid x',t')
    \\
    &=
    \log p_\theta(x,t)
    -
    \log p_\theta(x',t')
    +
    \log p_d(x',t')
    -
    \log p_d(x,t)
    \\
    &=
    \bigl(
        \log p_d(x',t')
        -
        \log p_\theta(x',t')
    \bigr)
    -
    \bigl(
        \log p_d(x,t)
        -
        \log p_\theta(x,t)
    \bigr).
\end{align}
Defining
\begin{align}
    e_\theta(x,t)
    :=
    \log p_d(x,t)-\log p_\theta(x,t),
\end{align}
we obtain
\begin{align}
    F_\theta(x,t,x',t')
    =
    e_\theta(x',t')-e_\theta(x,t).
\end{align}
Since $p_+=p_-=p_d(x,t)p_n(x',t'\mid x,t)$, the expansion becomes
\begin{align}
    \mathcal{L}(F_\theta)
    &=
    \mathcal{L}(0)
    +
    \frac14
    \E_{p_d(x,t)p_n(x',t'\mid x,t)}
    \left[
        \bigl(
            e_\theta(x',t')-e_\theta(x,t)
        \bigr)^2
    \right]
    \\
    &\quad+
    \E_{p_d(x,t)p_n(x',t'\mid x,t)}
    \left[
        o\!\left(
            \bigl(
                e_\theta(x',t')-e_\theta(x,t)
            \bigr)^2
        \right)
    \right].
\end{align}
This expansion is valid when $F_\theta$ is close to zero. In the
infinitesimal-time regime, this follows from the fact that $(x',t')$ is close
to $(x,t)$.

\paragraph{Case of the SDE}

Assume that $x_t$ follows the SDE
\begin{align}
    \dd x_t
    =
    b(x_t,t)\dd t
    +
    g(t)\dd W_t.
\end{align}
Let $t'=t+\Delta t$. Conditional on $(x_t,t)=(x,t)$, the next point satisfies
\begin{align}
    x'
    =
    x
    +
    b(x,t)\Delta t
    +
    g(t)\Delta W
    +
    R_t,
    \qquad
    \Delta W\sim \mathcal{N}(0,\Delta t\,I),
\end{align}
where
\begin{align}
    R_t
    =
    \int_t^{t+\Delta t}
    \bigl(
        b(x_s,s)-b(x,t)
    \bigr)\,ds
    +
    \int_t^{t+\Delta t}
    \bigl(
        g(s)-g(t)
    \bigr)\,dW_s .
\end{align}
Therefore, the expectation with respect to the transition kernel
$p_n(x',t'\mid x,t)$ may be written as an expectation with respect to the
Brownian increment $\Delta W$.

By It\^o's lemma,
\begin{align}
    e_\theta(x',t')-e_\theta(x,t)
    &=
    \partial_t e_\theta(x,t)\Delta t
    +
    \nabla_x e_\theta(x,t)^\top
    b(x,t)\Delta t
    \\
    &\quad+
    g(t)
    \nabla_x e_\theta(x,t)^\top
    \Delta W
    +
    \frac12
    g(t)^2
    \operatorname{Tr}
    \left(
        \nabla_x^2 e_\theta(x,t)
    \right)
    \Delta t
    +
    o(\Delta t).
\end{align}
Since $\Delta W=O(\sqrt{\Delta t})$, the leading term is
\begin{align}
    e_\theta(x',t')-e_\theta(x,t)
    =
    g(t)
    \nabla_x e_\theta(x,t)^\top
    \Delta W
    +
    O(\Delta t).
\end{align}
Thus, conditionally on $(x,t)$,
\begin{align}
    \E_{\Delta W}
    \left[
        \bigl(
            e_\theta(x',t')-e_\theta(x,t)
        \bigr)^2
        \mid x,t
    \right]
    &=
    g(t)^2
    \E_{\Delta W}
    \left[
        \bigl(
            \nabla_x e_\theta(x,t)^\top \Delta W
        \bigr)^2
    \right]
    +
    o(\Delta t)
    \\
    &=
    g(t)^2
    \left\|
        \nabla_x e_\theta(x,t)
    \right\|^2
    \Delta t
    +
    o(\Delta t).
\end{align}
Substituting into the loss expansion gives
\begin{align}
    \mathcal{L}(F_\theta)
    -
    \mathcal{L}(0)
    &=
    \frac14
    \E_{p_d(x,t)}
    \E_{\Delta W}
    \left[
        \bigl(
            e_\theta(x',t')-e_\theta(x,t)
        \bigr)^2
        \mid x,t
    \right]
    +
    o(\Delta t)
    \\
    &=
    \frac{\Delta t}{4}
    \E_{p_d(x,t)}
    \left[
        g(t)^2
        \left\|
            \nabla_x e_\theta(x,t)
        \right\|^2
    \right]
    +
    o(\Delta t).
\end{align}
Finally, since
\begin{align}
    \nabla_x e_\theta(x,t)
    =
    \nabla_x \log p_d(x,t)
    -
    \nabla_x \log p_\theta(x,t),
\end{align}
we obtain
\begin{align}
    \mathcal{L}(F_\theta)
    -
    \mathcal{L}(0)
    =
    \frac{\Delta t}{4}
    \E_{p_d(x,t)}
    \left[
        g(t)^2
        \left\|
            \nabla_x\log p_d(x,t)
            -
            \nabla_x\log p_\theta(x,t)
        \right\|^2
    \right]
    +
    o(\Delta t).
\end{align}

\paragraph{Case of the ODE}

Assume now that $x_t$ follows the ODE
\begin{align}
    \dd x_t=b(x_t,t)\dd t.
\end{align}
Let $t'=t+\Delta t$. Then, conditional on $(x_t,t)=(x,t)$,
\begin{align}
    x'
    =
    x
    +
    b(x,t)\Delta t
    +
    o(\Delta t).
\end{align}
Therefore,
\begin{align}
    e_\theta(x',t')-e_\theta(x,t)
    &=
    \left(
        \partial_t e_\theta(x,t)
        +
        b(x,t)^\top \nabla_x e_\theta(x,t)
    \right)
    \Delta t
    +
    o(\Delta t).
\end{align}
Defining the total derivative
\begin{align}
    \dd_t e_\theta(x,t)
    :=
    \partial_t e_\theta(x,t)
    +
    b(x,t)^\top \nabla_x e_\theta(x,t),
\end{align}
we get
\begin{align}
    e_\theta(x',t')-e_\theta(x,t)
    =
    \dd_t e_\theta(x,t)\Delta t
    +
    o(\Delta t).
\end{align}
Using the Taylor expansion of the logistic loss around zero, and using the
forward-reverse identity $p_+=p_-$, the linear term vanishes and
\begin{align}
    \mathcal{L}(F_\theta)
    -
    \mathcal{L}(0)
    &=
    \frac14
    \mathbb{E}_{p_d(x,t)}
    \left[
        \bigl(
            e_\theta(x',t')-e_\theta(x,t)
        \bigr)^2
    \right]
    +
    \mathbb{E}_{p_d(x,t)}
    \left[
        o\!\left(
            \bigl(
                e_\theta(x',t')-e_\theta(x,t)
            \bigr)^2
        \right)
    \right].
\end{align}
Substituting the ODE expansion gives
\begin{align}
    \mathcal{L}(F_\theta)
    -
    \mathcal{L}(0)
    &=
    \frac{\Delta t^2}{4}
    \mathbb{E}_{p_d(x,t)}
    \left[
        \bigl(
            \dd_t e_\theta(x,t)
        \bigr)^2
    \right]
    +
    o(\Delta t^2).
\end{align}
Finally, since
\begin{align}
    e_\theta(x,t)
    =
    \log p_d(x,t)-\log p_\theta(x,t),
\end{align}
we have
\begin{align}
    \dd_t e_\theta(x,t)
    =
    \dd_t \log p_d(x,t)
    -
    \dd_t \log p_\theta(x,t).
\end{align}
Therefore,
\begin{align}
    \mathcal{L}(F_\theta)
    -
    \mathcal{L}(0)
    =
    \frac{\Delta t^2}{4}
    \mathbb{E}_{p_d(x,t)}
    \left[
        \left(
            \dd_t\log p_d(x,t)
            -
            \dd_t\log p_\theta(x,t)
        \right)^2
    \right]
    +
    o(\Delta t^2).
\end{align}

We remark that the SDE interpretation is closer to the practical implementation of \stnce with forward-reverse kernel. With the ODE interpretation, one would need to know the velocity field $b(x,t)$ in advance, and need to keep track of the ground truth $\dd_{t}\log p_{t}(x)$.

\subsection{White noise kernel: infinitesimal limit}
\label{app:sec:dual_score_matching}

\paragraph{Simplifying the loss around zero}

Let
\begin{align}
    p_+(x,t,x',t')
    &:=
    p_d(x,t)p_n(x',t'\mid x,t),
    \\
    p_-(x,t,x',t')
    &:=
    p_d(x',t')p_n(x,t\mid x',t').
\end{align}
The logistic loss used in \stnce is
\begin{equation}
    \mathcal L(F_\theta)
    =
    -\E_{p_+}
    \left[
        \log \mathrm{sigmoid}(F_\theta(x,t,x',t'))
    \right]
    -\E_{p_-}
    \left[
        \log\bigl(
            1-\mathrm{sigmoid}(F_\theta(x,t,x',t'))
        \bigr)
    \right].
\end{equation}
Using Lemma~\ref{lemma:logistic_as_squared_around_zero}, whenever
$F_\theta$ is close to zero,
\begin{align}
    \mathcal L(F_\theta)
    &=
    \mathcal L(0)
    -
    \frac12
    \E_{p_+-p_-}
    \left[
        F_\theta
    \right]
    +
    \frac18
    \E_{p_++p_-}
    \left[
        F_\theta^2
    \right]
    +
    \E_{p_++p_-}
    \left[
        o(F_\theta^2)
    \right].
\end{align}

\paragraph{Plugging in the white noise kernel}

Assume that the perturbation kernel is an infinitesimal Gaussian kernel in
both space and time:
\begin{align}
    x'
    &=
    x+\sqrt{\varepsilon_x}\,\xi,
    \qquad
    \xi\sim \mathcal N(0,I),
    \\
    t'
    &=
    t+\sqrt{\varepsilon_t}\,\tau,
    \qquad
    \tau\sim \mathcal N(0,1),
\end{align}
where $\xi$ and $\tau$ are independent. Since the Gaussian kernel is
symmetric,
\begin{align}
    p_n(x',t'\mid x,t)
    =
    p_n(x,t\mid x',t').
\end{align}
Thus
\begin{align}
    F_\theta(x,t,x',t')
    &=
    \log p_\theta(x,t)
    -
    \log p_\theta(x',t').
\end{align}
A first-order Taylor expansion gives
\begin{align}
    F_\theta(x,t,x',t')
    =
    -
    \sqrt{\varepsilon_x}
    \nabla_x\log p_\theta(x,t)^\top \xi
    -
    \sqrt{\varepsilon_t}
    \partial_t\log p_\theta(x,t)\tau
    +
    o(
        \sqrt{\varepsilon_x}
        +
        \sqrt{\varepsilon_t}
    ).
\end{align}
Therefore $F_\theta$ is close to zero in the infinitesimal perturbation
limit, justifying the Taylor expansion around zero.

\paragraph{The linear term}

Using symmetry of $p_c$, we have
\begin{align}
    p_+(x,t,x',t')-p_-(x,t,x',t')
    =
    \bigl(
        p_d(x,t)-p_d(x',t')
    \bigr)
    p_n(x',t'\mid x,t).
\end{align}
Moreover,
\begin{align}
    p_d(x',t')
    =
    p_d(x,t)
    +
    \sqrt{\varepsilon_x}
    \nabla_x p_d(x,t)^\top \xi
    +
    \sqrt{\varepsilon_t}
    \partial_t p_d(x,t)\tau
    +
    o(
        \sqrt{\varepsilon_x}
        +
        \sqrt{\varepsilon_t}
    ).
\end{align}
Hence
\begin{align}
    p_d(x,t)-p_d(x',t')
    &=
    -
    p_d(x,t)
    \left[
        \sqrt{\varepsilon_x}
        \nabla_x\log p_d(x,t)^\top \xi
        +
        \sqrt{\varepsilon_t}
        \partial_t\log p_d(x,t)\tau
    \right]
    \\
    &\quad+
    o(
        \sqrt{\varepsilon_x}
        +
        \sqrt{\varepsilon_t}
    ).
\end{align}
Combining this with the expansion of $F_\theta$, and averaging over
$\xi,\tau$, gives
\begin{align}
    \E_{p_+-p_-}
    \left[
        F_\theta
    \right]
    &=
    \E_{p_d(x,t)}
    \Bigl[
        \varepsilon_x
        \nabla_x\log p_d(x,t)^\top
        \nabla_x\log p_\theta(x,t)
        +
        \varepsilon_t
        \partial_t\log p_d(x,t)
        \partial_t\log p_\theta(x,t)
    \Bigr]
    \\
    &\quad+
    o(\varepsilon_x+\varepsilon_t),
\end{align}
where the mixed term vanishes because $\xi$ and $\tau$ are independent and
centered.

\paragraph{The quadratic term}

Similarly,
\begin{align}
    F_\theta^2
    &=
    \varepsilon_x
    \left(
        \nabla_x\log p_\theta(x,t)^\top \xi
    \right)^2
    +
    \varepsilon_t
    \left(
        \partial_t\log p_\theta(x,t)\tau
    \right)^2
    \\
    &\quad+
    2\sqrt{\varepsilon_x\varepsilon_t}
    \left(
        \nabla_x\log p_\theta(x,t)^\top \xi
    \right)
    \left(
        \partial_t\log p_\theta(x,t)\tau
    \right)
    \\
    &\quad+
    o(\varepsilon_x+\varepsilon_t).
\end{align}
Averaging over $\xi,\tau$ yields
\begin{align}
    \E_{\xi,\tau}
    \left[
        F_\theta^2
    \right]
    =
    \varepsilon_x
    \left\|
        \nabla_x\log p_\theta(x,t)
    \right\|^2
    +
    \varepsilon_t
    \left(
        \partial_t\log p_\theta(x,t)
    \right)^2
    +
    o(\varepsilon_x+\varepsilon_t).
\end{align}
Since
\begin{align}
    p_+(x,t,x',t')+p_-(x,t,x',t')
    =
    2p_d(x,t)p_n(x',t'\mid x,t)
    +
    o(1),
\end{align}
we get
\begin{align}
    \frac18
    \E_{p_++p_-}
    \left[
        F_\theta^2
    \right]
    =
    \frac14
    \E_{p_d(x,t)}
    \Biggl[
        \varepsilon_x
        \left\|
            \nabla_x\log p_\theta(x,t)
        \right\|^2
        +
        \varepsilon_t
        \left(
            \partial_t\log p_\theta(x,t)
        \right)^2
    \Biggr]
    +
    o(\varepsilon_x+\varepsilon_t).
\end{align}

\paragraph{Combining the terms}

Substituting the linear and quadratic terms into the expansion gives
\begin{align}
    \mathcal L(F_\theta)
    -
    \mathcal L(0)
    &=
    \frac14
    \E_{p_d(x,t)}
    \Biggl[
        \varepsilon_x
        \left\|
            \nabla_x\log p_\theta(x,t)
        \right\|^2
        +
        \varepsilon_t
        \left(
            \partial_t\log p_\theta(x,t)
        \right)^2
    \Biggr]
    \\
    &\quad
    -
    \frac12
    \E_{p_d(x,t)}
    \Biggl[
        \varepsilon_x
        \nabla_x\log p_d(x,t)^\top
        \nabla_x\log p_\theta(x,t)
        \\
        &\qquad\qquad+
        \varepsilon_t
        \partial_t\log p_d(x,t)
        \partial_t\log p_\theta(x,t)
    \Biggr]
    \\
    &\quad+
    o(\varepsilon_x+\varepsilon_t).
\end{align}
Completing the square, we obtain
\begin{align}
    \mathcal L(F_\theta)
    -
    \mathcal L(0)
    &=
    \frac14
    \E_{p_d(x,t)}
    \Biggl[
        \varepsilon_x
        \left\|
            \nabla_x\log p_\theta(x,t)
            -
            \nabla_x\log p_d(x,t)
        \right\|^2
        \\
        &\qquad\qquad+
        \varepsilon_t
        \left(
            \partial_t\log p_\theta(x,t)
            -
            \partial_t\log p_d(x,t)
        \right)^2
    \Biggr]
    \\
    &\quad
    -
    \frac14
    \E_{p_d(x,t)}
    \Biggl[
        \varepsilon_x
        \left\|
            \nabla_x\log p_d(x,t)
        \right\|^2
        +
        \varepsilon_t
        \left(
            \partial_t\log p_d(x,t)
        \right)^2
    \Biggr]
    \\
    &\quad+
    o(\varepsilon_x+\varepsilon_t).
\end{align}
Therefore, up to additive terms independent of $\theta$,
\begin{align}
    \mathcal L(F_\theta)
    &\equiv
    \frac14
    \E_{p_d(x,t)}
    \Biggl[
        \varepsilon_x
        \left\|
            \nabla_x\log p_\theta(x,t)
            -
            \nabla_x\log p_d(x,t)
        \right\|^2
        \\
        &\qquad\qquad+
        \varepsilon_t
        \left(
            \partial_t\log p_\theta(x,t)
            -
            \partial_t\log p_d(x,t)
        \right)^2
    \Biggr]
    +
    o(\varepsilon_x+\varepsilon_t).
\end{align}
Thus, infinitesimal Gaussian perturbations recover the dual score matching
objective~\citep{guth2025dual}: one term matches the spatial score
$\nabla_x\log p_d(x,t)$, while the other matches the temporal score
$\partial_t\log p_d(x,t)$.

\newpage
\section{Special cases of our method}
\label{app:sec:connection_to_other_methods}

In this section, we detail how many popular losses for learning energy-based models are in fact special cases of our loss.

First, we express the \stnce loss function in a similar form as \citet{ceylan2018cnce}. Recall the \stnce loss function:
\begin{equation}
    \mathcal L_{\mathrm{\stnce}}(F_\theta)
    =
    -\E_{p_+}
    \left[
        \log \mathrm{sigmoid}(F_\theta(x,t,x',t'))
    \right]
    -\E_{p_-}
    \left[
        \log\bigl(
            1-\mathrm{sigmoid}(F_\theta(x,t,x',t'))
        \bigr)
    \right].
\end{equation}
Since $p_{+}(x,t,x',t') = p_{d}(x,t) p_{n}(x',t'|x,t)$ and $p_{-}(x,t,x',t') = p_{d}(x',t')p_{n}(x,t|x',t')$, due to symmetry one can switch $x, t$ and $x',t'$ without affecting the correctness. As such, we can simply reuse $x, t, x', t'$ as $x', t', x, t$ and write it using a single term, given by
\begin{align}
    \mathcal L_{\mathrm{\stnce}}(F_\theta)
    &=
    -2\E_{p_{d}(x,t) p_{n}(x',t'|x,t)}
    \left[
        \log \mathrm{sigmoid}(F_\theta(x,t,x',t'))
    \right] \\
    &= -2\E_{p_{d}(x|t)p(t) p_{n}(x',t'|x,t)}
    \left[
        \log \mathrm{sigmoid}(F_\theta(x,t,x',t'))
    \right],\\
F_\theta(x,t,x',t') &= \log p_{\theta}(x|t) + \log p(t) + \log p_{n}(x',t'|x,t)\\
&\quad - \log p_{\theta}(x'|t') - \log p(t') - \log p_{n}(x,t|x',t').
\end{align}
Observe that it concerns two choices: the time prior $p(t)$ and the perturbation kernel $p_{n}(x',t'|x,t)$. We next show how to recover exisiting methods using the above formulation.

\subsection{Summary}

We summarize in Table~\ref{table:stnce_special_cases} how many methods in the literature can be viewed as special cases under our \stnce framework. In particular, the proof for Trajectory Score Matching can be found in Section~\ref{app:sec:proofs}, and the proof for Dual Score Matching can be found in Section~\ref{app:sec:dual_score_matching}. In the following we show how the previous methods fall under our \stnce framework.

\begin{table}[!htbp]
\caption{\stnce recovers known methods, obtained by specific choices of the time prior and perturbation kernel.}
\label{table:stnce_special_cases}
\centering
\renewcommand{\arraystretch}{1.3}
\begin{tabularx}{\linewidth}{
@{}
>{\raggedright\arraybackslash}p{0.40\linewidth}
>{\raggedright\arraybackslash}p{0.16\linewidth}
>{\raggedright\arraybackslash}X
@{}
}
\toprule
\textbf{Method} 
& \textbf{Time prior} 
& \textbf{Perturbation kernel} $x', t' | x, t$ \\
\midrule
\multicolumn{3}{c}{\textit{Temporal variation}} 
\\
\midrule
Noise-Contrastive Estimation (NCE) \citep{gutmann2012nce_jmlr}
&
$\tfrac12 \delta_{0}(t) + \tfrac12 \delta_{1}(t)$
&
$\delta_{x}(x')\,\delta_{1-t}(t')$
\\
Flow-Contrastive Estimation (FCE) \citep{gao2020flowcontrastive}
&
$\tfrac12 \delta_{0}(t) + \tfrac12 \delta_{1}(t)$
&
$\delta_{x}(x')\,\delta_{1-t}(t')$
\\
Temporal NCE (\tnce)\citep{rhodes2020telescoping,ouyang2026diffusiveclassificationlosslearning,aggarwal2025boltznce}
&
$p_t(t)$
&
$\delta_{x}(x')\, p_n(t' \mid t)$
\\

Time Score Matching \citep{choi2022densityratio}
&
$p_t(t)$
&
$\delta_{x}(x')\, \delta_{t + \epsilon t}(t'), \quad \epsilon \to 0$
\\
\midrule
\multicolumn{3}{c}{\textit{Spatial variation}} 
\\
\midrule
Conditional NCE (CNCE) \citep{ceylan2018cnce}
&
$\delta_{1}(t)$
&
$p_n(x' \mid x)\,\delta_{t}(t')$
\\
Contrastive Divergence (CD) \citep{hinton2002cd}
&
$\delta_{1}(t)$
&
$p_{\mathrm{MCMC}}(x' \mid x)\,\delta_{t}(t')$
\\
Temporal CNCE (\tcnce)
&
$p_t(t)$
&
$p_n(x' \mid x)\,\delta_{t}(t')$
\\

Space Score Matching~\citep{hyvarinen2005sm,vincent2011dsm}
&
$\delta_1(t)$
&
$\mathcal{N}(x'; x, \epsilon I)\,\delta_{t}(t'), \quad \epsilon \to 0$
\\

Energy-based Diffusion \citep{song2021sde,thornton2025ebmdiffusion}
&
$p_t(t)$
&
$\mathcal{N}(x'; x, \epsilon I)\,\delta_{t}(t'), \quad \epsilon \to 0$
\\
\midrule
\multicolumn{3}{c}{\textit{Spatiotemporal variation}} 
\\
\midrule

Spatiotemporal NCE (\stnce)
&
$p_t(t)$
&
$p_n(x' \mid x, t, t')\, p_n(t' \mid t)$
\\
\tstnceMixture
&
$p_t(t)$
&
$\frac12 \delta_t(t')p_n(x'|x) + \frac12 \delta_x(x')p_n(t'|t)$
\\
\tstnceWhite & $p_t(t)$ & $\mathcal{N}\left(x' \mid x, \sigma^2 I\right)\,p_n(t'\mid t)$ \\
\tstnceSelf & $p_t(t)$ & \eqref{eq:forward-reverse-kernel} with $\nabla_{x}\log p_t(x)$ given by the model, along with $p_n(t'\mid t)$ \\
\tstnceOracle & $p_t(t)$ & \eqref{eq:forward-reverse-kernel} with oracle access to the ground truth $\nabla_{x}\log p_t(x)$, along with $p_n(t'\mid t)$ \\
\tstnceOracle (forward-reverse, infinitesimal, SDE)
&
$p_t(t)$
&
see Theorem~\ref{th:sm_mle}
\\
\tstnceOracle (forward-reverse, infinitesimal, ODE) 
&
$p_t(t)$
&
see Theorem~\ref{th:total_derivative_matching}
\\
Dual Score Matching \citep{guth2025dual}
&
$p_t(t)$
&
$\mathcal{N}([x', t']; [x, t], \epsilon I) \quad \epsilon \to 0$
\\
\bottomrule
\end{tabularx}
\end{table}

We next write some additional precisions for the above table. The FCE parameterizes the  the reference distribution as a flow-based model $p_d(x | 0) = p_{\beta}(x)$ that is updated during training. Similarly, the  CD  updates the perturbation kernel during training as $p_{\mathrm{MCMC}}(x' | x)$, that  is a MCMC kernel with stationary distribution $\mathrm{stopgrad}(p_{\theta})$.

\subsection{Noise-Contrastive Estimation (NCE)}

\paragraph{Original loss}

The standard Noise-Contrastive Estimation (NCE)
objective~\cite{gutmann2012nce_jmlr} is given by
\begin{align}
\mathcal{L}_{\mathrm{NCE}}(\theta)
&=
-
\mathbb{E}_{p_d(x | 1)}
\left[
\log
\mathrm{sigmoid}
\!\left(
\log p_\theta(x)-\log p_d(x | 0)
\right)
\right] \\
&\quad -
\mathbb{E}_{p_d(x | 0)}
\left[
\log
\left(
1-
\mathrm{sigmoid}
\!\left(
\log p_\theta(x)-\log p_d(x | 0)
\right)
\right)
\right].
\end{align}

\paragraph{Our loss}

Recall that the generic stNCE loss is
\begin{align}
\mathcal{L}_{\mathrm{stNCE}}(F_\theta)
&=
-
\mathbb{E}_{p_d(x, t) p_n(x', t' | x, t)}
\left[
\log \mathrm{sigmoid}(F_\theta(x,t,x',t'))
\right] \\
&\quad -
\mathbb{E}_{p_d(x', t') p_n(x, t | x', t')}
\left[
\log
\left(
1-\mathrm{sigmoid}(F_\theta(x,t,x',t'))
\right)
\right],
\end{align}
where the logit is
\begin{align}
F_\theta(x,t,x',t')
&=
\log p_\theta(x|t)
+
\log p(t)
+
\log p_n(x',t'|x,t) \\
&\quad -
\log p_\theta(x'|t')
-
\log p(t')
-
\log p_n(x,t|x',t').
\end{align}
and the model distribution at time zero is known, $p_{\theta}(x, 0) := p_0(x)$. Recall that $p_d(x, t) = p(t) p_d(x | t)$. Choosing the prior on time as $p(t)
=
\frac12 \delta_1(t)
+
\frac12 \delta_0(t),
$
and the perturbation kernel as 
$
p_n(x',t'|x,t)
=
\delta_x(x')
\delta_{1-t}(t')
$ simplifies the loss. In particular, the perturbation leaves the spatial variable unchanged and swaps
the temporal index. We now obtain
\begin{align}
\mathcal{L}_{\mathrm{stNCE}}(F_\theta)
&=
-
\mathbb{E}_{p_1(x)}
\left[
\log
\mathrm{sigmoid}
\!\left(
F_\theta(x,1,x,0)
\right)
\right]
-
\mathbb{E}_{p_0(x)}
\left[
\log
\left(
1-
\mathrm{sigmoid}
\!\left(
F_\theta(x,0,x,1)
\right)
\right)
\right].
\end{align}
where
\begin{align}
F_\theta(x,1,x,0)
&=
\log p_\theta(x|1)
+
\log p(1)
+
\log \delta_x(x)
+
\log \delta_0(0) \\
&\quad
-
\log p_\theta(x|0)
-
\log p(0)
-
\log \delta_x(x)
-
\log \delta_1(1),
\\
&=
\log p_\theta(x|1)
-
\log p_0(x)
\\
F_\theta(x,0,x,1)
&=
\log p_0(x)
-
\log p_\theta(x|1).
\end{align}
Substituting these expressions back into the \stnce loss yields
\begin{align}
\mathcal{L}_{\mathrm{stNCE}}
&=
-
\mathbb{E}_{p_d(x | 1)}
\left[
\log
\mathrm{sigmoid}
\!\left(
\log p_\theta(x)-\log p_0(x)
\right)
\right] \\
&\quad -
\mathbb{E}_{p_d(x | 0)}
\left[
\log
\left(
1-
\mathrm{sigmoid}
\!\left(
\log p_\theta(x)-\log p_0(x)
\right)
\right)
\right],
\end{align}
which exactly recovers the standard NCE objective.

\subsection{Flow-Contrastive Estimation (FCE)}

This is a special case of NCE where the reference distribution $p_d(x|0)$ is not
fixed beforehand but is instead updated adaptively during training. In
particular, FCE parameterizes the reference distribution as a flow-based model $p_d(x | 0) = p_{\beta}(x)$ that is updated during the optimization. 

\subsection{Temporal Noise-Contrastive Estimation (\tnce)}

\paragraph{Original loss}

The temporal Noise-Contrastive Estimation (\tnce)
objective which is heavily inspired by~\citet{rhodes2020telescoping} is given by
\begin{align}
\mathcal{L}_{\mathrm{tNCE}}(\theta)
&=
-
\mathbb{E}_{p_d(x,t)p_n(t'|t)}
\left[
\log
\mathrm{sigmoid}
\!\left(
\log p_\theta(x|t)-\log p_\theta(x|t')
+\Delta(t,t')
\right)
\right]
\\
&\quad
-
\mathbb{E}_{p_d(x,t')p_n(t|t')}
\left[
\log
\left(
1-
\mathrm{sigmoid}
\!\left(
\log p_\theta(x|t)-\log p_\theta(x|t')
+\Delta(t,t')
\right)
\right)
\right],
\end{align}
where
\begin{equation}
\Delta(t,t')
=
\log p(t)
+
\log p_n(t'|t)
-
\log p(t')
-
\log p_n(t|t')
\end{equation}
is known to the user. For example, the user can choose $p(t)$ to be uniform over the interval $[0, 1]$, and $p_n(t' | t)$ to be a Gaussian perturbation. In that case, $\Delta(t, t') = 0$. 

\paragraph{Our loss}

Recall that the generic \stnce loss is
\begin{align}
\mathcal{L}_{\mathrm{\stnce}}(F_\theta)
&=
-
\mathbb{E}_{p_d(x, t) p_n(x', t' | x, t)}
\left[
\log \mathrm{sigmoid}(F_\theta(x,t,x',t'))
\right] \\
&\quad -
\mathbb{E}_{p_d(x', t') p_n(x, t | x', t')}
\left[
\log
\left(
1-\mathrm{sigmoid}(F_\theta(x,t,x',t'))
\right)
\right],
\end{align}
where the logit is
\begin{align}
F_\theta(x,t,x',t')
&=
\log p_\theta(x|t)
+
\log p(t)
+
\log p_n(x',t'|x,t) \\
&\quad -
\log p_\theta(x'|t')
-
\log p(t')
-
\log p_n(x,t|x',t').
\end{align}
Choosing the perturbation kernel as
$
p_n(x',t'|x,t)
=
\delta_x(x')p_n(t'|t)
$
removes all spatial perturbations. In particular, the perturbation leaves
the spatial variable unchanged while perturbing only the temporal index.
We now obtain
\begin{align}
\mathcal{L}_{\mathrm{\stnce}}(F_\theta)
&=
-
\mathbb{E}_{p_d(x,t)p_n(t'|t)}
\left[
\log
\mathrm{sigmoid}
\!\left(
F_\theta(x,t,x,t')
\right)
\right] \\
&\quad -
\mathbb{E}_{p_d(x,t')p_n(t|t')}
\left[
\log
\left(
1-
\mathrm{sigmoid}
\!\left(
F_\theta(x,t,x,t')
\right)
\right)
\right].
\end{align}
Moreover,
\begin{align}
F_\theta(x,t,x,t')
&=
\log p_\theta(x|t)
+
\log p(t)
+
\log p_n(t'|t)
-
\log p_\theta(x|t')
-
\log p(t')
-
\log p_n(t|t')
\\
&=
\log p_\theta(x|t)
-
\log p_\theta(x|t')
+
\Delta(t,t').
\end{align}
Substituting these expressions back into the \stnce loss yields
\begin{align}
\mathcal{L}_{\mathrm{\stnce}}
&=
-
\mathbb{E}_{p_d(x,t)p_n(t'|t)}
\left[
\log
\mathrm{sigmoid}
\!\left(
\log p_\theta(x|t)-\log p_\theta(x|t')
+\Delta(t,t')
\right)
\right]
\\
&\quad
-
\mathbb{E}_{p_d(x,t')p_n(t|t')}
\left[
\log
\left(
1-
\mathrm{sigmoid}
\!\left(
\log p_\theta(x|t)-\log p_\theta(x|t')
+\Delta(t,t')
\right)
\right)
\right],
\end{align}
which exactly recovers the \tnce objective.

\subsection{Conditional Noise-Contrastive Estimation (CNCE)}

\paragraph{Original loss}

The Conditional Noise-Contrastive Estimation
(CNCE) objective~\cite{ceylan2018cnce} is given by
\begin{align}
&\quad\mathcal{L}_{\mathrm{CNCE}}(\theta) \\
&=
-
\mathbb{E}_{p_d(x | 1)p_n(x'|x)}
\left[
\log
\mathrm{sigmoid}
\!\left(
\log p_\theta(x)
+
\log p_n(x'|x)
-
\log p_\theta(x')
-
\log p_n(x|x')
\right)
\right]
\\
&
-
\mathbb{E}_{p_d(x' | 1)p_n(x|x')}
\left[
\log
\left(
1-
\mathrm{sigmoid}
\!\left(
\log p_\theta(x)
+
\log p_n(x'|x)
-
\log p_\theta(x')
-
\log p_n(x|x')
\right)
\right)
\right].
\end{align}

\paragraph{Our loss}

Recall that the generic stNCE loss is
\begin{align}
\mathcal{L}_{\mathrm{stNCE}}(F_\theta)
&=
-
\mathbb{E}_{p_d(x, t) p_n(x', t' | x, t)}
\left[
\log \mathrm{sigmoid}(F_\theta(x,t,x',t'))
\right] \\
&\quad -
\mathbb{E}_{p_d(x', t') p_n(x, t | x', t')}
\left[
\log
\left(
1-\mathrm{sigmoid}(F_\theta(x,t,x',t'))
\right)
\right],
\end{align}
where the logit is
\begin{align}
F_\theta(x,t,x',t')
&=
\log p_\theta(x|t)
+
\log p(t)
+
\log p_n(x',t'|x,t) \\
&\quad
-
\log p_\theta(x'|t')
-
\log p(t')
-
\log p_n(x,t|x',t').
\end{align}
Choosing the prior on time as $p(t)=\delta_1(t)$ and the perturbation
kernel as
$
p_n(x',t'|x,t)=p_n(x'|x)\delta_t(t')
$
removes all temporal perturbations. In particular, the temporal index
remains fixed while the perturbation acts only on the spatial variable.
We now obtain
\begin{align}
\mathcal{L}_{\mathrm{stNCE}}(F_\theta)
&=
-
\mathbb{E}_{p_d(x)p_n(x'|x)}
\left[
\log
\mathrm{sigmoid}
\!\left(
F_\theta(x,1,x',1)
\right)
\right]\\
&\quad -
\mathbb{E}_{p_d(x')p_n(x|x')}
\left[
\log
\left(
1-
\mathrm{sigmoid}
\!\left(
F_\theta(x,1,x',1)
\right)
\right)
\right].
\end{align}
Moreover,
\begin{align}
F_\theta(x,1,x',1)
&=
\log p_\theta(x|1)
+
\log p(1)
+
\log p_n(x'|x)
-
\log p_\theta(x'|1) \\
&\quad
-
\log p(1)
-
\log p_n(x|x')
\\
&=
\log p_\theta(x|1)
+
\log p_n(x'|x)
-
\log p_\theta(x'|1)
-
\log p_n(x|x').
\end{align}
Identifying $p_\theta(x|1)$ with $p_\theta(x)$, we obtain
\begin{align}
F_\theta(x,1,x',1)
=
\log p_\theta(x)
+
\log p_n(x'|x)
-
\log p_\theta(x')
-
\log p_n(x|x').
\end{align}
Substituting these expressions back into the \stnce loss yields
\begin{align}
&\quad \mathcal{L}_{\mathrm{stNCE}} \\
&=
-
\mathbb{E}_{p_d(x | 1)p_n(x'|x)}
\left[
\log
\mathrm{sigmoid}
\!\left(
\log p_\theta(x)
+
\log p_n(x'|x)
-
\log p_\theta(x')
-
\log p_n(x|x')
\right)
\right]
\\
&
-
\mathbb{E}_{p_d(x' | 1)p_n(x|x')}
\left[
\log
\left(
1-
\mathrm{sigmoid}
\!\left(
\log p_\theta(x)
+
\log p_n(x'|x)
-
\log p_\theta(x')
-
\log p_n(x|x')
\right)
\right)
\right],
\end{align}
which exactly recovers the CNCE objective.

\subsection{Contrastive Divergence (CD)}

\citet{yair2021contrastive} observed that this is a special case of CNCE where the perturbation kernel $p_{n}(x' | x)$ is an MCMC kernel with stationary distribution $\mathrm{stopgrad}(p_{\theta})$ constantly updated during training.

\subsection{Temporal Conditional Noise-Contrastive Estimation (\tcnce)}

\paragraph{Original loss}

The temporal Conditional Noise-Contrastive
Estimation (\tcnce) objective, which is a natural extension of the CNCE objective, is given by
\begin{align}
&\quad \mathcal{L}_{\mathrm{tCNCE}}(\theta) \\
&=
-
\mathbb{E}_{p_d(x,t)p_n(x'|x)}
\left[
\log
\mathrm{sigmoid}
\!\left(
\log p_\theta(x|t)
+
\log p_n(x'|x)
-
\log p_\theta(x'|t)
-
\log p_n(x|x')
\right)
\right]
\\
&
-
\mathbb{E}_{p_d(x',t)p_n(x|x')}
\left[
\log
\left(
1-
\mathrm{sigmoid}
\!\left(
\log p_\theta(x|t)
+
\log p_n(x'|x)
-
\log p_\theta(x'|t)
-
\log p_n(x|x')
\right)
\right)
\right].
\end{align}

\paragraph{Our loss}

Recall that the generic stNCE loss is
\begin{align}
\mathcal{L}_{\mathrm{stNCE}}(F_\theta)
&=
-
\mathbb{E}_{p_d(x, t) p_n(x', t' | x, t)}
\left[
\log \mathrm{sigmoid}(F_\theta(x,t,x',t'))
\right] \\
&\quad -
\mathbb{E}_{p_d(x', t') p_n(x, t | x', t')}
\left[
\log
\left(
1-\mathrm{sigmoid}(F_\theta(x,t,x',t'))
\right)
\right],
\end{align}
where the logit is
\begin{align}
F_\theta(x,t,x',t')
&=
\log p_\theta(x|t)
+
\log p(t)
+
\log p_n(x',t'|x,t) \\
&\quad
-
\log p_\theta(x'|t')
-
\log p(t')
-
\log p_n(x,t|x',t').
\end{align}
Choosing the perturbation kernel as
$
p_n(x',t'|x,t)
=
p_n(x'|x)\delta_t(t')
$
leaves the temporal index unchanged while perturbing only the spatial
variable. We now obtain
\begin{align}
\mathcal{L}_{\mathrm{stNCE}}(F_\theta)
&=
-
\mathbb{E}_{p_d(x,t)p_n(x'|x)}
\left[
\log
\mathrm{sigmoid}
\!\left(
F_\theta(x,t,x',t)
\right)
\right] \\
&\quad
-
\mathbb{E}_{p_d(x',t)p_n(x|x')}
\left[
\log
\left(
1-
\mathrm{sigmoid}
\!\left(
F_\theta(x,t,x',t)
\right)
\right)
\right].
\end{align}
Moreover,
\begin{align}
F_\theta(x,t,x',t)
&=
\log p_\theta(x|t)
+
\log p(t)
+
\log p_n(x'|x)
-
\log p_\theta(x'|t)
-
\log p(t)
-
\log p_n(x|x')
\\
&=
\log p_\theta(x|t)
+
\log p_n(x'|x)
-
\log p_\theta(x'|t)
-
\log p_n(x|x').
\end{align}
Substituting these expressions back into the \stnce loss yields
\begin{align}
&\quad \mathcal{L}_{\mathrm{stNCE}} \\
&=
-
\mathbb{E}_{p_d(x,t)p_n(x'|x)}
\left[
\log
\mathrm{sigmoid}
\!\left(
\log p_\theta(x|t)
+
\log p_n(x'|x)
-
\log p_\theta(x'|t)
-
\log p_n(x|x')
\right)
\right]
\\
&
-
\mathbb{E}_{p_d(x',t)p_n(x|x')}
\left[
\log
\left(
1-
\mathrm{sigmoid}
\!\left(
\log p_\theta(x|t)
+
\log p_n(x'|x)
-
\log p_\theta(x'|t)
-
\log p_n(x|x')
\right)
\right)
\right],
\end{align}
which exactly recovers the tCNCE objective.

\subsection{\tstnceMixture and sum of \tnce and \tcnce}

\paragraph{Original loss}

Consider the sum of the tNCE and tCNCE objectives:
\begin{align}
\mathcal{L}
=
\mathcal{L}_{\mathrm{tNCE}}
+
\mathcal{L}_{\mathrm{tCNCE}}.
\end{align}

\paragraph{Our loss}

Recall that the generic stNCE loss is
\begin{align}
\mathcal{L}_{\mathrm{stNCE}}(F_\theta)
&=
-
\mathbb{E}_{p_d(x, t) p_n(x', t' | x, t)}
\left[
\log \mathrm{sigmoid}(F_\theta(x,t,x',t'))
\right] \\
&\quad -
\mathbb{E}_{p_d(x', t') p_n(x, t | x', t')}
\left[
\log
\left(
1-\mathrm{sigmoid}(F_\theta(x,t,x',t'))
\right)
\right],
\end{align}
where the logit is
\begin{align}
F_\theta(x,t,x',t')
&=
\log p_\theta(x|t)
+
\log p(t)
+
\log p_n(x',t'|x,t) \\
&\quad
-
\log p_\theta(x'|t')
-
\log p(t')
-
\log p_n(x,t|x',t').
\end{align}
Choosing the perturbation kernel as
\begin{align}
p_n(x',t'|x,t)
=
\frac12
\delta_t(t')p_n(x'|x)
+
\frac12
\delta_x(x')p_n(t'|t)
\end{align}
creates two perturbation branches. The first branch leaves the temporal
index unchanged while perturbing the spatial variable, whereas the
second branch leaves the spatial variable unchanged while perturbing the
temporal index. We now split the \stnce loss according to these two branches:
\begin{align}
\mathcal{L}_{\mathrm{stNCE}}(F_\theta)
&=
-\frac12
\mathbb{E}_{p_d(x,t)p_n(x'|x)}
\left[
\log
\mathrm{sigmoid}
\!\left(
F_\theta(x,t,x',t)
\right)
\right]
\\
&\quad
-\frac12
\mathbb{E}_{p_d(x',t)p_n(x|x')}
\left[
\log
\left(
1-
\mathrm{sigmoid}
\!\left(
F_\theta(x,t,x',t)
\right)
\right)
\right]
\\
&\quad
-\frac12
\mathbb{E}_{p_d(x,t)p_n(t'|t)}
\left[
\log
\mathrm{sigmoid}
\!\left(
F_\theta(x,t,x,t')
\right)
\right]
\\
&\quad
-\frac12
\mathbb{E}_{p_d(x,t')p_n(t|t')}
\left[
\log
\left(
1-
\mathrm{sigmoid}
\!\left(
F_\theta(x,t,x,t')
\right)
\right)
\right].
\end{align}
For the spatial perturbation branch, we obtain
\begin{align}
F_\theta(x,t,x',t)
&=
\log p_\theta(x|t)
+
\log p(t)
+
\log p_n(x'|x)
\\
&\quad
-
\log p_\theta(x'|t)
-
\log p(t)
-
\log p_n(x|x')
\\
&=
\log p_\theta(x|t)
+
\log p_n(x'|x)
-
\log p_\theta(x'|t)
-
\log p_n(x|x'),
\end{align}
which is exactly the tCNCE logit. Similarly, for the temporal perturbation branch,
\begin{align}
F_\theta(x,t,x,t')
&=
\log p_\theta(x|t)
+
\log p(t)
+
\log p_n(t'|t)
\\
&\quad
-
\log p_\theta(x|t')
-
\log p(t')
-
\log p_n(t|t')
\\
&=
\log p_\theta(x|t)
-
\log p_\theta(x|t')
+
\Delta(t,t'),
\end{align}
where
\begin{align}
\Delta(t,t')
=
\log p(t)
+
\log p_n(t'|t)
-
\log p(t')
-
\log p_n(t|t').
\end{align}
This is exactly the tNCE logit. Substituting these expressions back into the \stnce loss yields
\begin{align}
\mathcal{L}_{\mathrm{stNCE}}
=
\frac12
\mathcal{L}_{\mathrm{tCNCE}}
+
\frac12
\mathcal{L}_{\mathrm{tNCE}},
\end{align}
which recovers the sum of the tNCE and tCNCE objectives up to a
constant rescaling factor independent of $\theta$.

\subsection{Time Score Matching}

\paragraph{Original loss}

The Time Score Matching objective is given by
\begin{align}
\mathcal{L}_{\mathrm{TSM}}(\theta)
=
\mathbb{E}_{p_d(x,t)}
\left[
\left(
\partial_t \log p_\theta(x|t)
-
\partial_t \log p_d(x|t)
\right)^2
\right].
\end{align}
where $\partial_t \log p_d(x|t)$ denotes the ground-truth temporal score. In practice, the ground-truth temporal score is not available in closed-form so many have proposed tractable ways to rewrite the objective, using integration by parts~\citep{choi2022densityratio} or by introducing a conditioning variable~\citep{yu2025dre}. 

\paragraph{Our loss}

Recall that the generic stNCE loss is
\begin{align}
\mathcal{L}_{\mathrm{stNCE}}(F_\theta)
&=
-
\mathbb{E}_{p_d(x, t) p_n(x', t' | x, t)}
\left[
\log \mathrm{sigmoid}(F_\theta(x,t,x',t'))
\right] \\
&\quad
-
\mathbb{E}_{p_d(x', t') p_n(x, t | x', t')}
\left[
\log
\left(
1-\mathrm{sigmoid}(F_\theta(x,t,x',t'))
\right)
\right],
\end{align}
where the logit is
\begin{align}
F_\theta(x,t,x',t')
&=
\log p_\theta(x|t)
+
\log p(t)
+
\log p_n(x',t'|x,t) \\
&\quad
-
\log p_\theta(x'|t')
-
\log p(t')
-
\log p_n(x,t|x',t').
\end{align}
To recover Time Score Matching, we choose a purely temporal infinitesimal
perturbation. That is, the perturbation kernel leaves the spatial
variable unchanged and perturbs time according to
$
p_n(x',t'|x,t)=\delta_x(x')\mathcal{N}(t';t,\varepsilon)
$,
with $\varepsilon \to 0$. We also assume that the temporal perturbation
is symmetric, so $p_n(t'|t)=p_n(t|t')$ up to boundary effects. We now
obtain
\begin{align}
\mathcal{L}_{\mathrm{stNCE}}(F_\theta)
&=
-
\mathbb{E}_{p_d(x,t)p_n(t'|t)}
\left[
\log \mathrm{sigmoid}(F_\theta(x,t,x,t'))
\right] \\
&\quad
-
\mathbb{E}_{p_d(x,t')p_n(t|t')}
\left[
\log
\left(
1-\mathrm{sigmoid}(F_\theta(x,t,x,t'))
\right)
\right].
\end{align}
Moreover,
\begin{align}
F_\theta(x,t,x,t')
&=
\log p_\theta(x|t)
+
\log p(t)
+
\log p_n(t'|t)
-
\log p_\theta(x|t')
-
\log p(t')
-
\log p_n(t|t')
\\
&=
\log p_\theta(x|t)
-
\log p_\theta(x|t')
+
\log p(t)
-
\log p(t'),
\end{align}
where the perturbation-kernel terms cancel by symmetry. If the time prior
is locally uniform, or if it is absorbed into the time-dependent density,
this reduces to
\begin{align}
F_\theta(x,t,x,t')
=
\log p_\theta(x|t)
-
\log p_\theta(x|t').
\end{align}
Writing $t'=t+\sqrt{\varepsilon}\xi$ with $\xi\sim\mathcal{N}(0,1)$, a
first-order Taylor expansion gives
\begin{align}
F_\theta(x,t,x,t')
=
-\sqrt{\varepsilon}\,
\partial_t \log p_\theta(x|t)\xi
+
o(\sqrt{\varepsilon}).
\end{align}
Using the second-order expansion of the logistic loss around
$F_\theta=0$, the infinitesimal temporal perturbation limit gives
\begin{align}
\mathcal{L}_{\mathrm{stNCE}}
=
\frac{\varepsilon}{4}
\mathbb{E}_{p_d(x,t)}
\left[
\left(
\partial_t \log p_\theta(x|t)
-
\partial_t \log p_d(x|t)
\right)^2
\right]
+
C
+
o(\varepsilon),
\end{align}
where $C$ is independent of $\theta$. Therefore, up to the constant
factor $\varepsilon/4$ and additive terms independent of $\theta$, the
infinitesimal temporal-only limit of \stnce recovers the Time Score
Matching objective. 

\subsection{Space Score Matching}

\paragraph{Original loss}

The Space Score Matching objective is given by
\begin{align}
\mathcal{L}_{\mathrm{SSM}}(\theta)
=
\mathbb{E}_{p_d(x|1)}
\left[
\left\|
\nabla_x \log p_\theta(x|1)
-
\nabla_x \log p_d(x|1)
\right\|^2
\right].
\end{align}
where $\nabla_x \log p_d(x|1)$ denotes the ground-truth spatial score. This quantity is unknown in closed-form so many have rewritten the objective in a tractable way using integration by parts~\citep{hyvarinen2005sm} or by introducing a conditioning variable~\citep{vincent2011dsm}.

\paragraph{Our loss}

Recall that the Energy-based Diffusion loss derived above is obtained
from the infinitesimal spatial-only limit of \stnce:
\begin{align}
\mathcal{L}_{\mathrm{stNCE}}
=
\frac{\varepsilon}{4}
\mathbb{E}_{p_d(x,t)}
\left[
\left\|
\nabla_x \log p_\theta(x|t)
-
\nabla_x \log p_d(x|t)
\right\|^2
\right]
+
C
+
o(\varepsilon).
\end{align}
Restricting the time distribution to the single time point $t=1$,
that is, choosing
$
p(t)=\delta_1(t),
$
yields
\begin{align}
\mathcal{L}_{\mathrm{stNCE}}
=
\frac{\varepsilon}{4}
\mathbb{E}_{p_d(x|1)}
\left[
\left\|
\nabla_x \log p_\theta(x|1)
-
\nabla_x \log p_d(x|1)
\right\|^2
\right]
+
C
+
o(\varepsilon),
\end{align}
which exactly recovers the standard Space Score Matching objective up to
a constant rescaling factor and additive terms independent of $\theta$.

\subsection{Energy-based Diffusion}
\label{app:ssec:ebm_diffusion_loss}

\paragraph{Original loss}

The Energy-based Diffusion loss~\citep{song2021sde,thornton2025ebmdiffusion} is given by
\begin{align}
\mathcal{L}_{\mathrm{EBD}}(\theta)
=
\mathbb{E}_{p_d(x,t)}
\left[
\left\|
\nabla_x \log p_\theta(x|t)
-
\nabla_x \log p_d(x|t)
\right\|^2
\right].
\end{align}
where $\nabla_x \log p_d(x|t)$ denotes the ground-truth spatial score.

\paragraph{Our loss}

Recall that the generic stNCE loss is
\begin{align}
\mathcal{L}_{\mathrm{stNCE}}(F_\theta)
&=
-
\mathbb{E}_{p_d(x, t) p_n(x', t' | x, t)}
\left[
\log \mathrm{sigmoid}(F_\theta(x,t,x',t'))
\right] \\
&\quad
-
\mathbb{E}_{p_d(x', t') p_n(x, t | x', t')}
\left[
\log
\left(
1-\mathrm{sigmoid}(F_\theta(x,t,x',t'))
\right)
\right],
\end{align}
where the logit is
\begin{align}
F_\theta(x,t,x',t')
&=
\log p_\theta(x|t)
+
\log p(t)
+
\log p_n(x',t'|x,t) \\
&\quad
-
\log p_\theta(x'|t')
-
\log p(t')
-
\log p_n(x,t|x',t').
\end{align}
To recover the Energy-based Diffusion loss, we choose a purely spatial
infinitesimal perturbation. That is, the perturbation kernel leaves the
temporal index unchanged and perturbs the spatial variable according to
$
p_n(x',t'|x,t)=\mathcal{N}(x';x,\varepsilon I)\delta_t(t')
$,
with $\varepsilon \to 0$. Since the Gaussian perturbation is symmetric,
$p_n(x'|x)=p_n(x|x')$. We now obtain
\begin{align}
\mathcal{L}_{\mathrm{stNCE}}
&=
-
\mathbb{E}_{p_d(x,t)p_n(x'|x)}
\left[
\log \mathrm{sigmoid}(F_\theta(x,t,x',t))
\right] \\
&\quad
-
\mathbb{E}_{p_d(x',t)p_n(x|x')}
\left[
\log
\left(
1-\mathrm{sigmoid}(F_\theta(x,t,x',t))
\right)
\right].
\end{align}
Moreover,
\begin{align}
F_\theta(x,t,x',t)
&=
\log p_\theta(x|t)
+
\log p(t)
+
\log p_n(x'|x)
-
\log p_\theta(x'|t)
-
\log p(t)
-
\log p_n(x|x')
\\
&=
\log p_\theta(x|t)
-
\log p_\theta(x'|t),
\end{align}
where the time prior terms cancel because $t'=t$, and the perturbation
kernel terms cancel by symmetry. Writing
$x'=x+\sqrt{\varepsilon}\xi$ with $\xi\sim\mathcal{N}(0,I)$, a
first-order Taylor expansion gives
\begin{align}
F_\theta(x,t,x',t)
=
-\sqrt{\varepsilon}\,
\nabla_x \log p_\theta(x|t)^\top \xi
+
o(\sqrt{\varepsilon}).
\end{align}
Using the second-order expansion of the logistic loss around
$F_\theta=0$, the infinitesimal spatial perturbation limit gives
\begin{align}
\mathcal{L}_{\mathrm{stNCE}}
=
\frac{\varepsilon}{4}
\mathbb{E}_{p_d(x,t)}
\left[
\left\|
\nabla_x \log p_\theta(x|t)
-
\nabla_x \log p_d(x|t)
\right\|^2
\right]
+
C
+
o(\varepsilon),
\end{align}
where $C$ is independent of $\theta$. Therefore, up to the constant
factor $\varepsilon/4$ and additive terms independent of $\theta$, the
infinitesimal spatial-only limit of \stnce recovers the Energy-Based
Diffusion loss.

\newpage
\section{Additional discussion}

\subsection{Compute burden of the likelihood: training-time \textit{versus} inference-time}
\label{app:sec:compute-burden}
This work highlights how the literature of learning energy-based models falls within three families of learning objectives: temporal, spatial, and spatiotemporal. Another axis for organizing the literature is two parameterizations of the energy-based model differences, both enabling likelihood estimation but differing primarily in when the computational burden is incurred.

A first class defers this cost to inference time by integrating learned quantities, including temporal methods based on $\partial_t \log p_t(x)$~\citep{choi2022densityratio,yu2025dre,chen2026timescore}, spatial methods based on $\nabla \log p_t(x)$, and spatio-temporal approaches integrating $\frac{\mathrm{d}}{\mathrm{d}t} \log p_t(x_t)$, such as continuous normalizing flows~\citep{chen2018neuralordinarydifferentialequations} and path-wise formulations~\citep{choi2022densityratio}. These methods are relatively easy to train, as they avoid higher-order differentiation, but incur significant inference-time cost due to numerical integration, although recent works look into \textit{learning} to integrate in few steps~\citep{boffi2025how,ai2026f2d2,ouyang2026scallop,chen2025diffusionsecantalignmentscorebased}.

A second class instead parameterizes these differences directly with an energy-based model, enabling one-step likelihood evaluation. This includes temporal methods such as NCE~\citep{gutmann2012nce_jmlr}, spatial methods such as (denoising) score matching~\citep{hyvarinen2005sm,vincent2011dsm}, and recent spatio-temporal approaches, including ours~\citep{he2026rne,guth2025dual,plainer2025consistent,aggarwal2025boltznce,ouyang2026diffusiveclassificationlosslearning}. These methods yield cheap inference—often a single network evaluation—but are harder to train, e.g., due to higher-order autodiff.

Overall, these approaches trade training-time for inference-time compute. We focus on the EBM approach, favoring expensive training for efficient inference. A systematic study of this trade-off is left for future work.

\subsection{Discussion on related spatiotemporal methods}
\label{app:sec:related-spatiotemporal}

Several recent papers~\citep{aggarwal2025boltznce,guth2025dual,he2026rne,ouyang2026diffusiveclassificationlosslearning,plainer2025consistent} proposed similar ideas of training both spatial and temporal differences of the energy-based models. Nevertheless, to the best of our knowledge, our work is the first to propose such a unifying framework that generalizes several of the previous work. As discussed earlier in the section, both \citet{aggarwal2025boltznce} and \citet{ouyang2026diffusiveclassificationlosslearning} are highly relevant to our \tnce algorithm. Essentially, the Noise-Contrastive Estimation loss used by \citet{aggarwal2025boltznce} is \tnce, while the diffusive classification loss proposed in \citet{ouyang2026diffusiveclassificationlosslearning} reduces to \tnce when there are two classes. As such, they suffer from the same pathology of \tnce. Both \citet{guth2025dual} and \citet{plainer2025consistent} learn the time scores apart from the Stein scores and can be interpreted as special instances of our framework, as will be discussed in Appendix~\ref{app:sec:dual_score_matching} and Appendix~\ref{app:sec:comparison-dual}. \citet{he2026rne} is related to our \stnce framework, but unlike \stnce, it is not consistent when using the model's own scores and used on its own; see further discussion in Appendix~\ref{app:sec:discussion-rne}.

\subsection{Discussion on RNE}
\label{app:sec:discussion-rne}

Next we consider RNE as proposed in \citet{he2026rne}, which was shown to be capable of regularizing energy-based diffusion models. The RNE loss is roughly given by
\begin{align}
\mathcal{L} &= \E_{x_{t}, t, x_{t+\Delta t}, t+ \Delta t} \left\|\text{sg}\left( \log p(x_{t}|x_{t+\Delta t}) - \log p(x_{t+\Delta t} | x_{t}) \right)\right. \\
&\quad \left.+ \log p_{t+\Delta t}(x_{t+\Delta t}) - \log p_{t}(x_{t})\right\|^{2},
\end{align}
where the diffusion convention of time is used, and $\text{sg}$ denotes stop gradient. \citet{he2026rne} showed that the regularizer, when combined with denoising score matching loss, allows training normalized energy-based models.

In our experiments, we refer to RNE with oracle access to the ground truth score as RNE-o, and RNE with the model's own scores as RNE-s. When used on its own and using the score estimates as given by the model itself, the RNE-s loss suffers from a notable pathology. 
Specifically, the optimum satisfying the Bayes rule constraint induced by RNE-s is not unique, in contrast to \tstnceSelf: while the ground-truth time-varying density $p_{t}$ is one optimum, there exist infinitely many other optimums that may differ arbitrarily from the data distribution. For example, the degenerate solution where $p_{t} \equiv c$ for some constant $c$ is an optimum under the Variance-Exploding noising process with Euler-Maruyama discretization, and any time-varying probability density induced by a clean data distribution that differ from the empirical distribution is also an optimum.

\subsection{Comparison against Dual Score Matching}
\label{app:sec:comparison-dual}

As discussed in Appendix~\ref{app:sec:dual_score_matching}, we can recover dual score matching from \stnce. Compared with \stnce without DSM, dual score matching requires optimizing the derivatives of the network and is more computationally costly; in this case \stnce provides a unique balance between training cost and performance. When combining \stnce with DSM, as validated on ImageNet-64 in Section~\ref{sec:image-imagenet-exp}, it can provide better performance than dual score matching.

\subsection{Broader impacts}
\label{app:sec:broader-impacts}

Our work is relatively theoretical and studies how to train EBMs. Nevertheless, due to the widespread usage of EBMs, it could be used by practioners in malicious manners. We attempt to mitigate such negative impacts by advocating for responsible usage of such trained models.

\newpage
\section{Implementation considerations}

\subsection{Sampling mechanisms}
\label{app:sec:sampling}

\stnce logits are determined by joint log densities in the form of $\log p(x, t) + \log p_{n}(x',t'|x,t) = \log p(t) + \log p(x | t) + \log p_{n}(x'|x, t, t') + \log p_{n}(t'|t)$. As such, one would need to take into account $\log p(t) + \log p_{n}(t'|t) = \log p_{n}(t,t')$ when forming the logit, unless $\log p_{n}(t,t') = \log p_{n}(t',t)$, in which case they cancel out with each other. We next introduce the folding mechanism, which will be beneficial afterwards.

\subsubsection{Useful identity: folding mechanism}
\label{app:sec:folding-mechanism}

We would like to sample $t'$ given $t$, such that when $t$ is uniformly distributed between $[0, 1]$, $t'$ follows the same distribution while remaining reasonably close to $t$. This can be achieved using the folding mechanism in Algorithm~\ref{alg:folding}. 

\begin{algorithm}
\caption{Folding mechanism.}
\label{alg:folding}
\begin{algorithmic}
\Require $t \sim U(0,1)$, $\sigma > 0$
\State $\epsilon \sim \mathcal{N}\left(0,\sigma^2\right)$
\State $\tilde{t} \gets t + \epsilon$
\State $t' = 1 - | (\tilde{t} \mod 2) - 1 |$
\State \Return $t'$
\end{algorithmic}
\end{algorithm}

\begin{proposition}
Let $t \sim U(0,1)$ and $\epsilon \sim \mathcal{N}(0,\sigma^2)$ be independent. Define
\begin{equation}
t' = 1 - | \left(\left(t + \epsilon\right) \mod 2\right) - 1 |,
\end{equation}
Then for any $\sigma > 0$, $t' \sim U(0,1)$ (\textbf{uniformity}), and $| t' - t| \leq |\epsilon|$ (\textbf{locality}).
\end{proposition}

\begin{proof}
\textbf{Uniformity}

Denote $\epsilon$ to be the Gaussian noise with standard deviation $\sigma$. Observe that a given value of $t'$ can be obtained from different $\tilde{t}$. Specifically, they could be $\tilde{t} = 2k + t'$ or $\tilde{t} = 2k - t'$, where $k \in Z$. We thus have
\begin{equation}
p_{t'}(t') = \sum_{k = -\infty}^{\infty} \left[ p_{\tilde{t}}(2k + t') + p_{\tilde{t}}(2k - t') \right],
\end{equation}
i.e. the sum of two terms. Note that
\begin{equation}
p_{\tilde{t}}(\tilde{t}) = \int_{0}^{1}p_{\epsilon}(\tilde{t}-t)\dd t,
\end{equation}
and
\begin{equation}
p_{t'}(t') = \sum_{k = -\infty}^{\infty} \left[ \int_{0}^{1} p_{\epsilon}(2k + t' - t) \dd t + \int_{0}^{1} p_{\epsilon}(2k - t' - t) \dd t \right].
\end{equation}
For the first term, using a change of variable $u = t - 2k$, we have
\begin{equation}
\sum_{k = -\infty}^{\infty} p_{\tilde{t}}(2k + t') = \sum_{k=-\infty}^{\infty} \int_{-2k}^{-2k + 1} p_{\epsilon} (t' - u) \dd u,
\end{equation}
i.e. it is the sum over all intervals of the form $[2k, 2k+1]$. For the seocnd term, using a change of variable $u = 2k - t$, we have
\begin{equation}
\sum_{k = -\infty}^{\infty} p_{\tilde{t}}(2k - t') = \sum_{k=-\infty}^{\infty} \int_{2k-1}^{2k} p_{\epsilon}(t' - u) \dd u,
\end{equation}
i.e. it is the sum over all intervals of the form $[2k -1, 2k]$.

Summarizing, these intervals cover the entire real line. As such,
\begin{equation}
p_{t'}(t') = \int_{-\infty}^{\infty}p_{\epsilon}(t' - u) \dd u = 1,
\end{equation}
i.e. $t'$ follows uniform distribution.

\textbf{Locality}

Denote $F(x) = 1 - | \left(x \mod 2\right) - 1 |$. Observe that $F(x) = x$ for $x \in [0,1]$, and $|F'(x)| = 1$ whenever the derivative exists. As such, $F$ is globally $1-$Lipschitz, and 
\begin{equation}
|t' - t| = |F(t+\epsilon) - t| = |F(t+\epsilon)-F(t)| \leq |(t+\epsilon)-t| = |\epsilon|.
\end{equation}
\end{proof}

Furthermore, given two values $t$ and $t'$, $p_{n}(t'|t) = p _{n}(t|t')$, as these probabilities only depend on the distances of the path between $t$ and $t'$, while a Gaussian is symmetric around its mean.

Note that there are different approaches to sample $t$, $t'$, $x$ and $x'$. In the following we describe two main sampling schemes, which we term \texttt{default} and \texttt{reuse}. We use $z$ to denote the clean data, and consider as perturbation kernel $p_{n}(x',t'|x,t) = p_{n}(t'|t) p_{n}(x'|x,t,t')$.

\subsubsection{default}

Perhaps the most natural implementation is to follow the data generation process of \stnce; we note that this is also the sampling mechanism as employed by \citet{aggarwal2025boltznce}. Given one $z$, one $(x, t, x', t')$ is constructed. It can be summarized as follows:
\begin{align}
t \sim p(t), &\quad  x \sim p(x | z, t),\\
t' \sim p_{n}(t' | t), &\quad x' \sim p_{n}(x' | x, t, t').
\end{align}

However, special care is needed on the logits involving $t$ and $t'$. Perhaps the simplest case is $p(t)$ is uniform and $p_{n}(t'|t) = p_{n}(t|t')$, as the part involving the time prior and time proposal naturally cancels out. Interestingly, when $p(t)$ is uniform, the folding mechanism from Section~\ref{app:sec:folding-mechanism} guarantees that $p_{n}(t'|t) = p_{n}(t|t')$, while ensuring that the marginal distribution $p(t') = \int p_{n}(t'|t) p(t) \dd t$ remains uniform.

When combining the NCE loss with a loss on the space score, we can use the sampled $x$ and $t$ as they are indeed unbiased samples from the stochastic interpolant.

\subsubsection{reuse}

Alternatively, one can consider the following mechanism:
\begin{align}
t_{0} \sim p(t_{0}), &\quad t_{1} \sim p_{n}(t_{1}|t_{0}),\\
t = [t_{0}, t_{1}], &\quad t' = [t_{1}, t_{0}],\\
x_{0} \sim p(x_{0} | z, t_{0}), &\quad x_{1} \sim p(x_{1} | z, t_{1}),\\
x'_{0} \sim p_{n}(x'_{0} | x_{0}, t_{0}, t_{1}), &\quad x'_{1} \sim p_{n}(x'_{1} | x_{1}, t_{1}, t_{0}),\\
x = [x_{0}, x_{1}], &\quad x' = [x'_{0}, x'_{1}].
\end{align}
Given one $z$, two $(x,t,x',t')$ instances are constructed, using $(t,t')$ and $(t',t)$, respectively. In this case, regardless of the choice of $p(t_{0})$ and $p_{n}(t_{1},t_{0})$, the sampling mechanism ensures that $p_{n}(t,t') = p_{n}(t',t)$. Nevertheless, one may still prefer to employ the folding mechanism from Section~\ref{app:sec:folding-mechanism} so as to make $t$ and $t'$ remain uniform. Furthermore, this enables mechanisms to reduce variances and even save computations, as detailed below.

If we combine the NCE loss with a space score loss, as $x$ and $t$, which contain $x_0, x_1, t_0, t_1$, are all unbiased samples for the stochastic interpolant, we use them when forming the latter loss.

\paragraph{General case}

\citet{rhodes2020telescoping} noted that reusing the same noise can reduce variance. As such, we could use the same noise when simulating $p(x | z, t)$ or $p_{n}(x' | x, t, t')$. Given one $z$, we would like to obtain samples from $t_{0}$ and $t_{1}$, and we can achieve this by interpolating between $z$ and the same noise using $t_{0}$ and $t_{1}$. Certain $p_{n}(x' | x, t, t')$ adds Gaussian noise of some standard deviations, and we can add the same Gaussian noise to $x_{0}$ and $x_{1}$ when simulating $x'_{0}$ and $x'_{1}$.  

\paragraph{\stnce with forward-reverse}

In the case of \stnce with forward-reverse kernel, $p_{n}$ could naturally correspond to diffusion forward process. As such, we could exploit this identity to reduce computations. Specifically, without loss of generality, one can design it so, such that $t_{0}$ is always smaller than $t_{1}$ by adding a sorting step upon the sampled $t_0$ and $t_1$. For numerical stability, one additionally ensure that the difference between them is not too small. Observe that one could reuse $x_{0}$ for $x'_{1}$, suppose one simulate $x_0$ from $x_{1}$ by following the diffusion forward process from $t_{1}$ to $t_{0}$. Furthermore, \stnce with forward-reverse kernel requires the scores always for the smaller $t$, and we only need the scores for $t_{0}$. When we combine \stnce loss with a loss on the space score, these scores are naturally needed for the latter loss, as $x_{0}$ are indeed samples from the corresponding distributions.

In order for the forward kernel to perfectly coincide with the forward process while avoiding numerical issues, we additionally employ a mechanism to guarantee that $t_{1}$ and $t_{0}$ are apart by at least $\epsilon$. We follow the convention that this mechanism is employed when we use the \texttt{reuse} scheme, even for algorithms apart from \stnce with forward-reverse.

\subsection{Implementing forward-reverse kernel}

Denote the forward denoising process as $x_{t} = \alpha_{t} z + \beta_{t} n$, where $z$ are data samples and $n$ are noise samples from standard normal distribution.

From the perspective of diffusion and flows, the choice of the kernel in the case where $t' < t$ is obvious (note that we use the flow matching convention of time), as it can be given by the forward process of diffusion models. In this case, we can choose
\begin{equation}
p(x'|x,t',t) = \mathcal{N}\left(x'\middle|\frac{\alpha_{t'}}{\alpha_{t}}x,\beta_{t'}^{2} - \frac{\alpha_{t'}^{2}\beta_{t}^{2}}{\alpha_{t}^{2}}\right),
\end{equation}
which essentially generates samples from $p(x'|t')$ given samples from $p(x|t)$.

Under the scenario that $t' > t$, with the above forward process one might be tempted to use as transition kernel the one that leads to the reverse kernel of diffusions, in which case we would sample from the posterior distribution. However, this posterior is highly non-trivial and is available only under specific data distributions.

\subsubsection{Recovery likelihood}

Inspired by diffusion recovery likelihood \citep{gao2021learning}, we can construct approximations to the reverse kernel. In the following, to simplify notations, we denote $p(x'|x,t',t)$ as $p(x_{\tau}|x_{t})$, where $x_{t} = \alpha x_{\tau} + \beta n$, with $\alpha = \frac{\alpha_{t}}{\alpha_{\tau}}$ and $\beta = \sqrt{\beta_{t}^{2} - \frac{\alpha_{t}^{2} \beta_{t'}^{2}}{\alpha_{t'}^{2}}}$. We have

\begin{align}
\log p(x_{\tau}|x_{t}) &= \log p(x_{\tau}) + \log p(x_{t}|x_{\tau}) + c \\
&\approx \log p(x_{t}) + \langle \nabla_{x}\log p(x_{t}),x_{\tau} - x_{t} \rangle - \frac{1}{2\beta^{2}}\lVert x_{t} - \alpha x_{\tau} \rVert^{2} + c\\
&= \langle \nabla_{x}\log p(x_{t}),x_{\tau} \rangle - \frac{\alpha^{2}}{2\beta^{2}}\left\lVert \frac{x_{t}}{\alpha} - x_{\tau} \right\rVert^{2} + c\\
&= \langle \nabla_{x}\log p(x_{t}),x_{\tau} \rangle - \frac{\alpha^{2}}{2\beta^{2}} \lVert x_{\tau} \rVert^{2} + \frac{\alpha^{2}}{2\beta^{2}}2\left\langle x_{\tau} ,\frac{x_{t}}{\alpha} \right\rangle + c\\
&= -\frac{\alpha^{2}}{2\beta^{2}}\left( \langle x_{\tau},x_{\tau} \rangle - 2 \left\langle x_{\tau},\frac{x_{t}}{\alpha} \right\rangle - \frac{2\beta^{2}}{\alpha^{2}}\left\langle x_{\tau}, \nabla_{x}\log p(x_{t}) \right\rangle \right) + c\\
&= -\frac{\alpha^{2}}{2\beta^{2}}\left\lVert x_{\tau} - \left(\frac{x_{t}}{\alpha} + \frac{\beta^{2}}{\alpha^{2}}\nabla_{x}\log p(x_{t})\right) \right\rVert^{2} + c,
\end{align}
which motivates using the following Gaussian approximation:
\begin{equation}
p(x_{\tau}|x_{t}) = \mathcal{N}\left( x_{\tau} \middle| \frac{x_{t}}{\alpha} + \frac{\beta^{2}}{\alpha^{2}}\nabla_{x}\log p(x_{t}) ,\frac{\beta^{2}}{\alpha^{2}} I \right).
\end{equation}

\subsubsection{Exponential integrator}

The sampling scheme based on diffusion recovery likelihood builds on the assumption that the score is locally constant across $x$. However, this assumption may not be ideal. Indeed, works on accelerated diffusion samplers \citep{lu2022dpm-solver,gonzalez2023seeds} typically assume that the score is locally constant across $t$.

Here we focus on the exponential integrator for the reverse SDE as presented in \citet{gonzalez2023seeds}, specifically the SEEDS-1 variant, which is a first order SDE solver. The update rule is given by
\begin{equation}
x_{\tau} = \frac{\alpha_{\tau}}{\alpha_{t}}x_{t} - 2\beta_{\tau}\left(\exp(h) - 1\right) F(x_{t},t) - \beta_{\tau}\sqrt{\exp(2h) - 1} \epsilon,
\end{equation}
where $F$ is the output of noise prediction model, and $h = \lambda_{\tau} - \lambda_{t} = \log\left(\frac{\alpha_\tau}{\beta_\tau}\right) - \log\left(\frac{\alpha_t}{\beta_t}\right)$. As such, $\exp(h) = \frac{\alpha_{\tau}\beta_{t}}{\beta_{\tau}\alpha_{t}}$ and $\exp(2h) = \left(\frac{\alpha_{\tau}\beta_{t}}{\beta_{\tau}\alpha_{t}}\right)^{2}$.

Observe that the update rule of SEEDS-1 naturally induces a Gaussian distribution. The mean is given by
\begin{align}
&\quad \frac{\alpha_{\tau}}{\alpha_{t}} x_{t} + 2 \beta_{\tau}\left( \frac{\alpha_{\tau}\beta_{t}}{\beta_{\tau}\alpha_{t}} - 1 \right) \beta_{t}\nabla_{x}\log p_{t}(x)\\
&= \frac{\alpha_{\tau}}{\alpha_{t}}x_{t} + 2\left( \frac{\alpha_{\tau}\beta_{t}^{2}}{\alpha_{t}} - \beta_{\tau}\beta_{t} \right)\nabla_{x}\log p_{t}(x)\\
&= \frac{\alpha_{\tau}}{\alpha_{t}}x_{t} + 2\beta_{t}\left( \frac{\alpha_{\tau}}{\alpha_{t}} \beta_{t} - \beta_{\tau} \right)\nabla_{x}\log p_{t}(x),
\end{align}
and the covariance is given by
\begin{equation}
\beta_{\tau}^{2}\left( \left(\frac{\alpha_{\tau}\beta_{t}}{\beta_{\tau}\alpha_{t}}\right)^{2}-1 \right) I = \left(\frac{\alpha_{\tau}^{2}\beta_{t}^{2}}{\alpha_{t}^{2}} - \beta_{\tau}^{2}\right) I.
\end{equation}
Compared with the one obtained through recovery likelihood, the only difference is the scaling term in front of the score.

Consider the scenario where $\alpha_{t} = t$ and $\beta_{t} = 1-t$. The above expression simplifies, resulting in a Gaussian distribution with mean given by
\begin{equation}
\frac{\tau}{t} x_{t} + \frac{2(1-t)(\tau - t)}{t}\nabla_{x}\log p_{t}(x_{t}),
\end{equation}
and covariance given by
\begin{equation}
\frac{(\tau - t) (t + \tau - 2 t \tau)}{t^{2}} I.
\end{equation}

\subsection{Parameterization of energy-based diffusion models}
\label{app:sec:parameterization}

A proper parameterization of energy-based diffusion models, $E_{\theta}(x,t)$, could be crucial in practice, where the training stability of the high variance of the Denoising Score Matching objective at low noise levels required to be cured \citep{song2019score}. Let $x_t=\alpha_tx_1+\sigma_tz$. In standard Diffusion models, \citet{karras2022edm} proposes to parameterize a denoiser with skip connection:
\begin{align}
D_\theta(x_t, t)=c_\mathrm{skip}(t)x_t+c_\mathrm{out}(t)f_\theta(c_\mathrm{in}(t)x_t, c_{t}(t)),
\end{align}
where $c_\mathrm{skip}(t), c_\mathrm{out}(t), c_\mathrm{in}(t)$ are predefined parameters that depend on the variance of target data and noise schedules. Assumed an optimal denoiser, the denoising score, $\nabla\log p_t$, is then recovered through Tweedie's formula:
\begin{align}
c_\mathrm{skip}(t)x_t+c_\mathrm{out}(t)f_\theta(c_\mathrm{in}(t)x_t, c_{t}(t)) = \alpha_t^{-1}\bigl(x_t + \sigma_t^2\nabla\log p_t(x_t)\bigr).
\end{align}
Therefore, one could parameterize an energy network as \citep{thornton2025ebmdiffusion}:
\begin{align}
U_\theta(x_t, t)=\frac{1-\alpha_tc_\mathrm{skip}(t)}{2\sigma_{t}^{2}}\|x_t\|^2-\frac{\alpha_tc_\mathrm{out}(t)}{c_\mathrm{in}(t)\sigma_{t}^{2}}F_\theta(c_\mathrm{in}(t)x_t, c_{t}(t)),
\end{align}
where $F_\theta$ is an arbitrary scalar network.
Then applying Tweedie's formula, it recovers the EDM preconditioning:
\begin{align}
D_\theta(x_t, t) &= \alpha_t^{-1}\bigl(x_t - \sigma_t^2\nabla U_\theta(x_t, t)\bigr)= c_\mathrm{skip}(t)x_t+c_\mathrm{out}(t)\nabla_{c_\mathrm{in}(t)x_t}F_\theta(c_\mathrm{in}(t)x_t, c_{t}(t)).
\end{align}

\subsubsection{Pathology of model without preconditioning}

In the following, we focus on the case of using Cond OT probability path \citep{lipman2023fm}.

One might be tempted to directly parameterize the EBMs using some neural architectures and obtain its gradient as the score. However, recall that the velocity can be expressed using the score as
\begin{equation}
v = \frac{x}{t} + \frac{1-t}{t}\nabla\log p_{t}(x_t).
\end{equation}
As such, when $t$ is small, any small error in the estimated $\nabla\log p_{t}(x_t)$ is significantly amplified.

\subsubsection{Choice of preconditioning}

Next we develop our exact choice of preconditioning. In the following we use $\sigma$ to denote the standard deviation of data.

\begin{align}
c_{in}(t) &= \frac{1}{\sqrt{t^{2}\sigma^{2} + (1-t)^{2}}},\\
c_{skip}(t) &= \frac{t\sigma^{2}}{t^{2}\sigma^{2} + (1-t)^{2}},\\
c_{out}(t) &= \frac{(1-t)\sigma}{\sqrt{t^{2}\sigma^{2} + (1-t)^{2}}},\\
c_{t}(t) &= \log \left(\frac{1-t}{t}\right).
\end{align}

With the above choices, the velocity can be expressed using the model as
\begin{equation}
v = \frac{c_{skip}(t) - 1}{1 - t}x + \sigma\nabla_{x}F(c_{in}x,c_{t}(t)),
\end{equation}
which avoids the singularity issue under the vanilla model.

$c_{in}(t)$ serves to keep the inputs to the network with unit standard deviation. We have
\begin{align}
\text{Var}(x_{t}) &= t^{2}\sigma^{2} + (1-t)^{2},\\
\text{Std}(x_{t}) &= \sqrt{t^{2}\sigma^{2} + (1-t)^{2}},\\
c_{in}(t) &= \frac{1}{\sqrt{t^{2}\sigma^{2} + (1-t)^{2}}}.
\end{align}

$c_{skip}(t)$ serves as the optimal linear coefficient to predict $x_{1}$ based on $x_{t}$, and its optimal value is given by
\begin{align}
c_{skip}(t) &= \frac{\text{Cov}(x_{1},x_{t})}{\text{Var}(x_{t})}\\
&= \frac{\text{Cov}(x_{1},t x_{1} + (1-t) x_{0}}{t^{2}\sigma^{2} + (1-t)^{2}}\\
&= \frac{t \sigma^{2}}{t^{2}\sigma^{2} + (1-t)^{2}}.
\end{align}

$c_{out}(t)$ serves to keep the network output have variance $1$, and can be derived by considering the variance of $x_{1} - c_{skip}(t) x_{t}$. Observe that
\begin{align}
&\quad x_{1} - c_{skip}(t) x_{t}\\
&= x_{1} - c_{skip}(t) (t x_{1} + (1-t) x_{0})\\
&= (1 - t c_{skip}(t))x_{1} - (1-t) c_{skip}(t) x_{0},
\end{align}
and
\begin{align}
\text{Var}(x_{1} - c_{skip}(t) x_{t}) &= \left(1-t c_{skip}(t)\right)^{2}\sigma^{2} + \left((1-t)c_{skip}(t)\right)^{2}.
\end{align}
Denote the denominator of $x_{1} - c_{skip}(t) x_{t}$ to be $D = t^{2}\sigma^{2} + (1-t)^{2}$, as such $c_{skip}(t) = \frac{t\sigma^{2}}{D}$, and
\begin{equation}
1 - t c_{skip}(t) = 1 - \frac{t^{2}\sigma^{2}}{D} = \frac{(1-t)^{2}}{D}.
\end{equation}
We have
\begin{align}
\text{Var}(x_{1} - c_{skip}(t) x_{t}) &= \left( \frac{(1-t)^{2}}{D} \right)^{2}\sigma^{2} + \left(\frac{(1-t)t\sigma^{2}}{D}\right)^{2}\\
&= \frac{(1-t)^{4}\sigma^{2} + (1-t)^{2}t^{2}\sigma^{4}}{D^{2}}\\
&= \frac{(1-t)^{2}\sigma^{2}\left((1-t)^{2} + t^{2}\sigma^{2}\right)}{D^{2}}\\
&= \frac{(1-t)^{2}\sigma^{2}}{D}.
\end{align}

As such,
\begin{equation}
c_{out}(t) = \text{Std}(x_{1} - c_{skip}(t) x_{t}) = \frac{(1-t)\sigma}{\sqrt{t^{2}\sigma^{2} + (1-t)^{2}}}.
\end{equation}

$c_{noise}(t)$ provides time embedding that reflects the effective noise level. With Cond OT path, the ratio of noise std to signal std is given by $\frac{1-t}{t\sigma}$, and we have
\begin{equation}
c_{noise}(t) = \ln\left(\frac{1-t}{t\sigma}\right).
\end{equation}

Flow matching model is known to work well without needing $c_{in}(t)$ and $c_{t}(t)$. As such, we can drop them from $F$, while keeping them elsewhere.

It might be tempting to directly employ the above choices to parameterize the energy and velocity. Observe that $\frac{c_{out}(t)}{c_{in}(t)} = (1-t)\sigma$. We have
\begin{align}
U(x_{t},t) &= \frac{1-t c_{skip}(t)}{2(1-t)^{2}}\norm{x_{t}}^{2} - \frac{\sigma t}{1-t} F,\\
v &= \frac{c_{skip}(t) - 1}{1 - t} x + \sigma\nabla_{x}F.
\end{align}
Note that $c_{skip}(1) = 1$. As such, when $t=1$, $U(x_{t},t)$ has a singularity.

In order to resolve the singularity, we can either sample time not until $1$ or replace $1-t$ with $1-at$ for some $a < 1$. In this work, we set $a = 0.75$.

\subsection{Score and velocity}
\label{app:sec:score-velocity}

Flow Matching \citep{lipman2023fm} provides an alternative viewpoint for diffusion models. When the noise distribution is Gaussian, we can convert between the score and the velocity. Similar identities have been shown in previous works, e.g. \citet{holderrieth2026glass}.

Consider $x_{t} = \alpha_{t} x_{1} + \beta_{t} x_0$, with $x_0$ following standard normal distribution. For simplicity, we ignore subscripts $t$ of $\alpha_t$ and $\beta_t$ below. We have
\begin{align}
v_{t}(x) &= \frac{\dot{\alpha}}{\alpha} x + \left( \frac{\dot{\alpha}\beta^{2}}{\alpha} - \dot{\beta}\beta \right)\nabla \log p_{t}(x),\\
\nabla\log p_{t}(x) &= \frac{v_{t}(x) - \frac{\dot{\alpha}}{\alpha}x}{\frac{\dot{\alpha}\beta^{2}}{\alpha} - \dot{\beta}\beta}.
\end{align}

\begin{proof}
In this proof, expectations are taken by conditioning on $x_t = x$, and we simply denote them as $\E$.
\begin{align}
\nabla \log p_{t}(x|x_{1}) &= -\frac{1}{\beta^{2}}\left( x - \alpha x_{1} \right),\\
v_{t}(x|x_{1}) &= \dot{\alpha}x_{1} + \dot{\beta}\frac{x - \alpha x_{1}}{\beta} = \dot{\alpha}x_{1} - \dot{\beta}\beta\nabla\log p_{t}(x|x_{1}),\\
\nabla \log p_{t}(x) &= \E\left[-\frac{x-\alpha x_{1}}{\beta^{2}}\right] = -\frac{x-\alpha \E\left[x_{1}\right]}{\beta^{2}},\\
\E\left[x_{1}\right] &= \frac{\beta^{2}\nabla\log p_{t}(x) + x}{\alpha},\\
v_{t}(x) &= \E\left[\dot{\alpha}x_{1} - \dot{\beta}\beta\nabla\log p_{t}(x|x_{1})\right]
= \dot{\alpha}\E\left[x_{1}\right] - \dot{\beta}\beta\E\left[\nabla\log p_{t}(x|x_{1})\right]\\
&= \frac{\dot{\alpha}}{\alpha}\left(\beta^{2}\nabla\log p_{t}(x) + x\right) - \dot{\beta}\beta \nabla\log p_{t}(x)= \frac{\dot{\alpha}}{\alpha}x + \left( \frac{\dot{\alpha}\beta^{2}}{\alpha} - \dot{\beta}\beta \right) \nabla\log p_{t}(x),\\
\nabla\log p_{t}(x) &= \frac{v_{t}(x) - \frac{\dot{\alpha}}{\alpha}x}{\frac{\dot{\alpha}\beta^{2}}{\alpha} - \dot{\beta}\beta}.
\end{align}
\end{proof}

The above connection between the score and the velocity also enables one to use flow matching loss to train models under different parameterizations \citep{li2026basicsletdenoisinggenerative}. Therefore, while \dualSm, following \citet{guth2025dual}, uses the original denoising score matching loss, for \tstnceSelfDsm we adapt flow matching loss to train the space score. As it still trains the space score, we refer to it still as a space score matching loss.

\subsection{Obtaining approximately normalized energies}
\label{app:sec:normalization}

\citet{guth2025dual} looked into how to obtain estimates of approximately normalized energies from the trained models. Specifically, as they trained the model using dual score matching, upon optimality the model could still possess an unknown normalization constant. They estimated the constant based on the entropy of the initial noise distribution.

In similar spirit, we estimate such constants for our models as well.

Spatial methods, even when training models that are time-varying to optimality, could naturally lead to different normalizing constants across times. As such, in order to normalize them, we require access to the ground truth energy at the time point corresponding to data.

For temporal methods, we extract the density ratio learned by the model and add it to the known log density of the Gaussian distribution at $t=0$. 

For spatiotemporal methods, for general model parameterizations, we estimate such constant through importance sampling around $t=0$ using the tractable initial noise distribution as the proposal as
\begin{equation}
x_i \sim p_d(x|0), \quad \widehat{Z_0} = \frac{1}{M}\sum_{i=1}^{M}\frac{\exp\left(-E_\theta(x_i,0)\right)}{p_d(x_i|0)}, \quad \ell_\theta(x,t) = -E_\theta(x,t) - \log \widehat{Z_0},
\end{equation}
where $E_\theta$ is the raw output energy and $\ell_\theta$ is the estimated log density. For specific model parameterizations, e.g. the one as in Appendix~\ref{app:sec:parameterization}, the energy at $t=0$ is analytically available, hence the constant is also tractable, in which case we directly use the analytical one.

\newpage
\section{Experimental details}
\label{app:sec:exp-details}

Recall that default scheme results in one training instance for each input data, and reuse scheme results in two training instances for each input data. As such, when comparing between reuse and default, we use two times larger batch size for default. Unless otherwise stated, we train the model for $100000$ steps and perform evaluation every $2000$ steps.

The algorithms involve different hyperparameters. Especially, when the algorithm involves comparisons across different times, we need to determine the sampling scheme for $t'$ given $t$. Inspired by \citet{aggarwal2025boltznce}, we consider either using a local Gaussian proposal before wrapping things back to the range of $t$ or sampling $t'$ uniformly at random. We denote the standard deviation of the Gaussian proposal as $\sigma_{time}$, and use $\sigma_{time} = -1.0$ to denote the case where $t'$ is sampled uniformly at random. Some algorithms involve adding local Gaussian perturbations to the positions, in which case we tune the standard deviation of the perturbation $\sigma_{white}$ between $[0.01, 0.1, 1.0]$.

For \tnce, we tune $\sigma_{time}$. For \tcnce, we tune $\sigma_{white}$; under reuse scheme, as an ablation we additionally tune $\sigma_{time}$. We tune \tstnceWhite in terms of both $\sigma_{time}$ and $\sigma_{white}$. \tstnceSelf and \tstnceOracle are tuned in terms of $\sigma_{time}$. 

We generally select the best model based on val set results under a single random seed, and report test set results. For each job, we select the step that yields the best val set results and use the test set results at that step as the final metrics. Sometimes we would like to obtain uncertainty estimates, in which case we additionally train the models using the best configs with two additional random seeds.

Unless otherwise stated, we use Adam optimizer \citep{kingma2015adam} with no weight decay.

\subsection{Glossary of methods}

We provide a compact glossary of methods in Table~\ref{tbl:glossary}.

\begin{table}[!htbp]
\centering
\caption{Experimental variants. ``self'' denotes the space score estimates induced by the energy network, and ``oracle'' denotes the ground truth space scores. Appendix~\ref{app:sec:connection_to_other_methods} presents additional discussion on the methods, and the training objective on the space scores is discussed in Appendix~\ref{app:sec:adding-dsm}.}
\label{tbl:glossary}
\renewcommand{\arraystretch}{1.15}
\begin{tabularx}{\linewidth}{@{}l>{\raggedright\arraybackslash}X@{}}
\toprule
Method & Training objective \\
\midrule
\tnce & \stnce loss with perturbation kernel keeping $x$ and perturbing $t$ \\
\tcnce & \stnce loss with perturbation kernel keeping $t$ and perturbing $x$ \\
\tstnceMixture & \stnce loss combining an equal mixture of \tnce perturbation and \tcnce perturbation \\
\tstnceWhite & \stnce loss jointly perturbing $x$ and $t$ \\
\tstnceSelf & \stnce loss with forward-reverse kernel using detached self space scores \\
\tstnceOracle & \stnce loss with forward-reverse kernel using oracle space scores \\
\rneSelf & RNE objective using self space scores \\
\rneOracle & RNE objective using oracle space scores \\
\tnceDsm & sum / average of \tnce and a loss on the space scores \\
\tstnceSelfDsm & sum / average of \tstnceSelf loss and a loss on the space scores \\
CTSM & Conditional Time Score Matching loss as in \citet{yu2025dre} and \citet{guth2025dual} \\
DSM & Denoising Score Matching loss as in \citet{song2021sde} and \citet{thornton2025ebmdiffusion} \\
\dualSm & sum / average of CTSM loss and DSM loss \\
\bottomrule
\end{tabularx}
\end{table}

\subsection{Losses on the scores}
\label{app:sec:adding-dsm}

We may use loss functions on the scores induced by the model. For instance, we may combine \tnce or \tstnceSelf with a loss function on the space score.

For experiments apart from molecule, we combine an NCE loss $\mathcal{L}_{\text{NCE}}$ with a loss on the space score $\mathcal{L}_{\text{DSM}}$ to obtain the final loss $\mathcal{L}_{\text{total}}$ as
\begin{equation}
\mathcal{L}_{\text{total}} = 0.5 w_{\text{NCE}} \mathcal{L}_{\text{NCE}} + 0.5 \mathcal{L}_{\text{DSM}},
\end{equation}
where $w_{\text{NCE}}$, defaulting to $1$, is an additional weight that can be tuned. In practice, we may implement $\mathcal{L}_{\text{DSM}}$ in two ways:

(1) As weighted denoising score matching loss similar to \citet{song2021sde} and \citet{thornton2025ebmdiffusion} as
\begin{align}
s_\theta^{\text{space}}(x,t) &= \nabla_x \log p_\theta(x\mid t), \\
\mathcal{L}_{\text{DSM}} &= \E_{p(t)p(x_1)p_t(x|x_1)}\left[\beta_t^2\norm{s_\theta^{\text{space}}(x,t) - \nabla_x \log p_t(x|x_1)}^2\right].
\end{align}
(2) As a conditional flow matching loss with uniform weighting as in \citet{lipman2023fm}, where the velocity estimates are obtained using the space score estimates obtained by differentiating the energy network, as described in Appendix~\ref{app:sec:score-velocity}.
\begin{align}
v_\theta(x,t) &= \frac{\dot{\alpha}_t}{\alpha_t}x + \left(\frac{\dot{\alpha}_t\beta_t^2}{\alpha_t} - \dot{\beta}_t\beta_t\right)\nabla_x\log p_\theta(x|t), \\
\mathcal{L}_{\text{DSM}} &= \E_{p(t)p(x_1)p_t(x|x_1)}\left[\norm{v_\theta(x,t)-v(x|x_1)}^2\right]
\end{align}

Apart from the space score, we may also use CTSM loss~\citep{yu2025dre,guth2025dual} to train the time score, given by
\begin{align}
s_\theta^{\text{time}}(x,t) &= \partial_t\log p_\theta(x|t), \\
w(t) &= \frac{1}{\E_{p(x_1)p_t(x|x_1)}\left[\partial_t\log p_t(x|x_1)^2\right]}, \\
\mathcal{L}_{\text{CTSM}} &= \E_{p(t)p(x_1)p_t(x|x_1)}\left[w(t)\norm{s_\theta^{\text{time}}(x,t)-\partial_t\log p_t(x|x_1)}^2\right],
\end{align}
where we use CTSM weighting, which can be derived in closed-form for stochastic interpolants in the form of $x_t = \alpha_t x_1 + \beta_t x_0$, as in \citet{yu2025dre}.

For both denoising score matching loss and CTSM loss, we may implement the weighting using the identity that $w(t)\norm{s-t}^2 = \norm{\sqrt{w(t)}s-\sqrt{w(t)}t}^2$ and prevent some singularities at the boundaries of time, following \citet{tong2024sb}.

For Dual SM, based on \citet{guth2025dual}, the noisy $x$ and $t$ are sampled based on the clean data. The final loss $\mathcal{L}_{\text{total}}$ is given as the weighted sum of CTSM loss $\mathcal{L}_{\text{CTSM}}$ with weighting function following \citet{yu2025dre} and \citet{guth2025dual} and DSM loss $\mathcal{L}_{\text{DSM}}$ with weighting function following \citet{ho2020ddpm} and \citet{guth2025dual}
\begin{equation}
\mathcal{L}_{\text{total}} = 0.5 \mathcal{L}_{\text{CTSM}} + 0.5 \mathcal{L}_{\text{DSM}}.
\end{equation}

The sampling mechanisms of data points used for the DSM loss vary across \texttt{default} and \texttt{reuse}, and are described in Appendix~\ref{app:sec:sampling}.

For molecule experiments, we adapt the implementation of \citet{plainer2025consistent}, where the final loss is
\begin{equation}
\mathcal{L}_{\text{total}} = \mathcal{L}_{\text{NCE}} + \mathcal{L}_{\text{DSM}},
\end{equation}
and the data points for DSM loss are sampled following the \texttt{reuse} scheme described above for experiments apart from molecule.

\subsection{Network architecture}

\begin{table}[t]
\caption{Neural network architectures for toy mixture, MNIST mixture, MNIST and ImageNet-64. For U-Net model, the encoder and decoder residual block counts per level exclude the bottleneck, which has two additional ResBlocks surrounding a self attention block. Dropout applies to every U-Net Resblock.}
\label{tbl:model-architectures}
\centering
\setlength{\tabcolsep}{4pt}
\renewcommand{\arraystretch}{1.15}
\begin{tabular}{@{}
    p{0.13\linewidth}
    p{0.23\linewidth}
    p{0.16\linewidth}
    p{0.09\linewidth}
    p{0.30\linewidth}
@{}}
\toprule
Dataset & Backbone; channels & Residual blocks & Dropout & Energy parameterization \\
\midrule
Toy mixture
& MLP with residual layers; width $128$, expansion factor $2$
& $4$ total
& None
& Final linear layer mapping to scalar \\
MNIST mixture
& U-Net; $(32,64,64)$
& $2$ encoder / $3$ decoder per level
& $0$
& Preconditioned sum \\
MNIST
& U-Net; $(64,128,128)$
& $2$ encoder / $3$ decoder per level
& $0.1$
& Preconditioned inner product with input, minus $z_\phi(t)$ \\
ImageNet-64
& U-Net; $(64,128,128)$
& $3$ encoder / $4$ decoder per level
& $0$
& Preconditioned inner product with input, minus $z_\phi(t)$ \\
\bottomrule
\end{tabular}
\end{table}

Table~\ref{tbl:model-architectures} summarizes the model architectures for toy mixture, MNIST mixture, MNIST and ImageNet experiments; the networks for the other experiments are detailed in the corresponding sections. The network structure for toy mixture is based on \citet{du2023rrr}. The networks for MNIST mixture, MNIST and ImageNet-64 are adapted from \citet{lipman2024flowmatchingguidecode}. Our preconditioning is described in Appendix~\ref{app:sec:parameterization}.

Given a U-Net with vector valued output, we consider two strategies to generate a scalar valued energy: (1) \textbf{sum}, directly summing the vector, (2) \textbf{inner product}, taking the inner product between the network output and the input. $z_\phi(t)$ is a separate neural network with scalar valued output, and the final output energy is given by the output of the preconditioned U-Net minus the output of $z_\phi(t)$. Our parameterization is inspired by previous works~\citep{du2023rrr,thornton2025ebmdiffusion,he2026rne,ouyang2026diffusiveclassificationlosslearning}.

When using DSM alone to train EBMs on MNIST mixture, we disable the bias of the final U-Net convolution. The bias is otherwise enabled. This difference leads to negligible difference in the number of parameters.

\subsection{Failure modes}
\label{app:sec:failure-modes-details}

\paragraph{Experimental settings}

We consider a simple setup where the reference distribution is a standard Gaussian and the data distribution is a Gaussian mixture:

\begin{align}
    p_0 = \mathcal{N}(0, 1),
    \quad
    p_1 = 0.5\, \mathcal{N}(\mu_1, 0.1^2) + 0.5\, \mathcal{N}(\mu_2, 0.1^2).
\end{align}
By varying $(\mu_1, \mu_2)$, we control both the multimodality of $p_1$ and the mismatch between the high-density regions of $p_0$ and $p_1$, quantified in $[0,1]$. We measure the error as $1 - R^2 \in [0, 1]$. Here, the Pearson $R^2$ measures the correlation between the ground truth and learnt energies of the data points. 
We evaluate three representative methods: \tnce (temporal differences), \tcnce (spatial differences), and \tstnceMixture (spatio-temporal). 

\paragraph{Multimodality and mismatch scores}
We used heuristic scores to measure two properties of the data distribution $p_1$: its multi-modality and its mismatch with the reference distribution $p_0$. The multimodality score is computed as
\begin{align}
    \mathrm{score}_{\mathrm{multimodality}}
    =
\frac{\mu_2 - \mu_1}{\mu_2 - \mu_1 + 2\sigma}, \in [0,1].
\end{align}
The mismatch score is computed as
\begin{align}
    \mathrm{score}_{\mathrm{mismatch}}
    = 
    1 - \frac{\max\bigl(0,\, \min(b,1) - \max(a,-1)\bigr)}{b - a},
\end{align}
where
\begin{align}
a = \min(\mu_1 - \sigma,\ \mu_2 - \sigma),
\quad
b = \max(\mu_1 + \sigma,\ \mu_2 + \sigma).
\end{align}
To vary both these scores on an approximately uniform grid on $[0, 1] \times [0, 1]$, we set $\sigma = 0.1$ and use $\mu_1$ and $\mu_2$ from the grid shown in Table~\ref{tbl:failure-mode-components}.

\begin{table}[!htbp]
\centering
\caption{The means of the two components of the one dimensional Gaussian mixture in the form of $(\mu_1, \mu_2)$. Rows correspond to different mismatch scores, and columns correspond to different multimodality scores.}
\label{tbl:failure-mode-components}
\setlength{\tabcolsep}{3pt}
\begin{tabular}{@{}lccccc@{}}
        \toprule
        & \multicolumn{5}{c}{Multimodality score} \\
        \cmidrule(l){2-6}
        \shortstack[l]{Mismatch\\score} & $\approx 0$ & $\approx 0.25$ & $\approx 0.50$ & $\approx 0.75$ & $[0.90, 0.96]$ \\
        \midrule
        $\approx 0$ & $(-0.90,-0.90)$ & $(0.79,0.85)$ & $(0.01,0.21)$ & $(0.01,0.61)$ & $(-0.91,0.91)$ \\
        $\approx 0.25$ & $(-0.95,-0.95)$ & $(-0.97,-0.90)$ & $(0.80,1.00)$ & $(-1.10,-0.50)$ & $(-1.07,1.42)$ \\
        $\approx 0.5$ & $(1.00,1.00)$ & $(-1.03,-0.97)$ & $(-1.10,-0.89)$ & $(-1.30,-0.70)$ & $(-2.29,1.55)$ \\
        $\approx 0.75$ & $(-1.05,-1.05)$ & $(-1.10,-1.03)$ & $(1.00,1.20)$ & $(-1.50,-0.90)$ & $(-3.00,-0.37)$ \\
        $\approx 1$ & $(-3.00,-3.00)$ & $(-3.00,-2.93)$ & $(-3.00,-2.80)$ & $(-3.00,-2.39)$ & $(-3.00,-1.07)$ \\
        \bottomrule
    \end{tabular}
\end{table}

\paragraph{Model and optimization} 
The energy function is parameterized by a neural network that is a simple Multi-Layer Perceptron with $2$ hidden layers and SiLU activations. Its weights are updated using the AdamW optimizer \citep{loshchilov2018adamw} with learning rate $1e-3$, weight decay $1e-4$, using a batch size of $256$ over $1e4$ epochs. An exponential moving average of the weights is used, with a high parameter of $0.9999$, to remove the residual variance in the loss. We empirically observed that the residual variance was high when the methods were expected to have a high estimation error, as is the case for example for \tcnce when the data distribution is multi-modal. 

\subsection{Toy mixture}

We consider the problem of modeling a $20$ mode Gaussian mixture in $10$ dimensions, where the means of the Gaussian mixture are sampled from a standard normal distribution, and the covariances of the individual mixture components are given by $0.1^{2} I$.

We use a batch size of $250$ for default scheme and $125$ for reuse scheme. We tune the learning rates between $[1e-4, 3e-4, 1e-3]$, and select the best model based on NormMSE on the val set.

\subsection{MNIST mixture}

\subsubsection{Settings}
\label{app:sec:mnist-mixture-settings}

We use a batch size of $512$ for default scheme and $256$ for reuse scheme. We fix the learning rate of all methods to $1e-4$, and select the best config based on NormMSE. For RNE, we tune $\Delta t$ between $[1e-5, 1e-4, 1e-3]$, while setting $\lambda$, the weighting of the RNE regularization loss, to $1000$, following \citet{he2026rne}.

\subsubsection{Evaluations}
\label{app:sec:mnist-mixture-evaluations}

Since we have access to both the ground truth data density and the sampling mechanism that generates samples from it, we estimate the log constant of the model density at the time corresponding to clean data using importance sampling, denoted as $\widehat{\log Z_1}$, and evaluate the model density normalized using $\widehat{\log Z_1}$.
\begin{align}
\mathrm{LogMSE}
&:=
\log\left(\E_{x \sim p_d(\cdot \mid 1)}
\!\left[
\bigl(
\log p_d(x \mid 1)
-
\log p_\theta(x \mid 1)
\bigr)^2
\right]\right),
\\
\mathrm{LogRatio}
&:=
\log\left(\E_{x,x' \sim p_d(\cdot \mid 1)^{\otimes 2}}
\!\left[
\Bigl(
\log p_d(x \mid 1)
-
\log p_d(x' \mid 1)\right.\right. \\
& \quad\left.\left.
-
\log p_\theta(x \mid 1)
+
\log p_\theta(x' \mid 1)
\Bigr)^2
\right]\right),
\\
\mathrm{LogNormMSE}
&:=
\log\left(\E_{x \sim p_d(\cdot \mid 1)}
\!\left[
\bigl(
\log p_d(x \mid 1)
-
\log p_\theta(x \mid 1)
+
\widehat{\log Z_1}
\bigr)^2
\right]\right),
\\
\mathrm{NormNLL}
&:=
\E_{x \sim p_d(\cdot \mid 1)}
\!\left[
-
\log p_\theta(x \mid 1)
+
\widehat{\log Z_1}
\right].
\end{align}
The first is the log squared error between the data and model log densities. The second is the log squared error between \textit{spatial} differences of the data and model log densities --- this metric is well-suited for evaluating unnormalized densities given that it does not depend on normalizing constants. The third and the fourth evaluate how good the model is as an unnormalized EBM, though naturally depends on the accuracy of the estimated $\widehat{\log Z_1}$. 

All means and standard deviations are computed based on results under the log transformed space.

\subsection{Images}

We do \textit{not} use the NormNLL as the evaluation metric for two reasons. First, related works~\citep{guth2025dual} do not use it and we wish to replicate a fair comparison. Second, the idea behind NormNLL is to correct the approximate normalization of the model using a separate estimate of the normalizing constant. In Section~\ref{sec:mnist-mixture-exp}, we could obtain a precise estimate because we knew the ground-truth distribution of the dataset. Here, we do not know the ground-truth distribution and the standard way to estimate the normalizing constant is using AIS, which in our preliminary experiments and in related literature~\citep{yu2025dre} was very high-variance on image datasets.

\subsubsection{MNIST}

We always use \texttt{reuse} scheme, and employ a batch size of $256$. We fix the learning rates of all methods to $1e-4$. For variants combining space score loss, we observe that using a smaller $w_{\text{NCE}}$ typically leads to better performance, and tune it between $1.0$, $0.1$, $0.01$ and $0.001$. The best configurations are selected based bits-per-dimension on the val set. Especially for \tstnceSelfDsm, we avoid the configurations that show clear sign of overfitting to the train set as indicated by bits-per-dimension.

\subsubsection{ImageNet}

We use $50000$ samples from the original training set as the val set. Same as \citet{guth2025dual}, we use the original val set as our test set. We use reuse scheme with batch size $256$. We additionally employ mixed precision with \textit{bf16} to speed up training. The model is trained for $200000$ steps. We use as learning rate $2e-4$ with $\sigma_{time} = 0.01$ in the first $122000$ steps, while changing the learning rate to $1e-4$ and $\sigma_{time}$ to $0.001$ afterwards. Here $\sigma_{time}$ refers to the standard deviation used to perturb $t$; see $\sigma$ as in Appendix~\ref{app:sec:folding-mechanism}.

\subsection{Molecules}
\label{app:sec:molecule-details}

\subsubsection{Settings}

In practice, we have access to energy and score function (\textit{a.k.a.} force field) of the target distribution via physics, but they are expensive to evaluate as they may require solving a PDE. To avoid evaluating them for every new data point, there are two prevailing approaches in the literature.

The first approach is Machine Learning Force Fields (MLFFs), which regress a neural network against the energy and score functions on some finite grid of data points~\citep{Husic_2020,batatia2022macehigherorderequivariant}. The second approach which we follow here is Energy-based Models.

To form the dataset for training such models, a sampling process such as Langevin Dynamics (\textit{a.k.a.} Molecular Dynamics), which requires evaluating the target energy and score functions, is run long enough so that the final samples are approximately \textit{i.i.d.} samples from the target distribution. 

Following \citet{plainer2025consistent}, we optimize the models using AdamW \citep{loshchilov2018adamw}, where we directly use the models as employed by \citet{plainer2025consistent}. Note that we do not employ the mixture version as in \citet{plainer2025consistent}.

We fix the weights of \tnce and \tstnceSelf loss to DSM loss as $1$, and tune $\sigma_{time}$ between $[-1.0, 0.1, 0.01]$; in some cases we experimented with the weights of NCE losses and $\sigma_{time} = 0.001$, but did not observe notable differences.

When benchmarking the running time, for each method, we estimate the time per epoch by measuring the time for $10$ epochs, following $2$ epochs to warmup.

\subsubsection{Evaluations}

Given the histogram-based density estimates of the corresponding $2$ dimensional ground-truth and model distributions as $P_b$ and $Q_b$, we compute the resulting PMF error according to \citet{durumeric2024learningdataefficientcoarsegrained} and \citet{plainer2025consistent}. 
Specificlly, let $M_b = \frac{P_b + Q_b}{2}$ and introduce $\tilde{P_b}$ and $\tilde{Q_b}$ as the masses with missing bins replaced by $\epsilon=10^{-6}$. Define $w_b = \frac{\tilde{P_b} + \tilde{Q_b}}{\sum_c\left(\tilde{P_c} + \tilde{Q_c}\right)}$, then the PMF error is given by
\begin{equation}
\mathrm{PMF}(P,Q) = \sum_b w_b\left(\log \tilde{P_b}-\log \tilde{Q_b}\right)^2.
\end{equation}

\subsection{Compute resources}
\label{app:sec:compute-resources}

For failure modes experiments in Section~\ref{sec:failure-modes-exp}, we run the experiments using Intel Xeon Gold 6230. Based on the saved logs and scripts, there are a total of $75$ jobs, with each job using $4$ CPU cores and taking around half an hour. We could not verify whether these exact logs and scripts produced the results reported in the paper, but report them as they appear reasonable.

For synthetic Gaussian mixture experiments in Section~\ref{sec:app:syn-mixture-results}, we run the experiments using Intel Xeon Gold 6230. Apart from the jobs to generate the datasets which are generally fast, each job generally uses $4$ CPU cores and takes roughly $1 - 2$ hours.

MNIST Gaussian mixture experiments in Section~\ref{sec:mnist-mixture-exp} and MNIST experiments in Section~\ref{sec:mnist-exp} are run using AMD MI250x GPU and AMD EPYC 7A53 CPU. For MNIST Gaussian mixture experiments, apart from some lightweight preprocessing jobs run using CPUs, each model is trained using $1$ GPU, which is composed of $2$ GCDs. Depending on the method, running times vary between $9$ hours and $20$ hours. For MNIST experiments, each model is trained using $2$ GPUs, containing $4$ GCDs. Across different methods, the running times vary between around $3$ hours and around $8$ hours. The long run for \tstnceSelf took around $19$ hours.

ImageNet64 experiments are run using a single NVIDIA H200 GPU, which takes roughly $5$ days.

For molecule experiments, for ALDP and ALDP full, the models are trained using a single NVIDIA V100 GPU. For Chignolin dataset, the models are trained using a single NVIDIA A100 GPU. Each job takes between several hours and tens of hours. For benchmarking the running time, we use a single NVIDIA A100; using the recorded time per epoch in seconds to estimate the time to run the benchmarks, it takes less than one minute to run the benchmark for ALDP and few minutes to run the benchmark for Chignolin.

We do not know the exact peak memory usage of each job. Preliminary and failed runs and analysis and visualization runs consumed additional compute.

\subsection{License information}
\label{app:sec:license}

For experiments apart from Failure modes and molecules, our code is inspired by \citet{song2021sde} (Apache-2.0 license), and, notably, some parts are based on \citet{lipman2024flowmatchingguidecode} (CC BY-NC license) and \citet{guth2025dual} (BSD-3-Clause license). Our code for molecule experiment is based on \citet{plainer2025consistent} (MIT license), which uses OpenMM~\citep{eastman2024openmm} (MIT license and LGPL license). MNIST dataset (MIT) license, ImageNet64 dataset (non-commercial research and education) and Chignolin dataset (D.E. Shaw Research Trajectory Data License Agreement) are also employed. Our code is based on PyTorch \citep{ansel2024pytorch} (custom license) or JAX \citep{jax2018github} (Apache-2.0 license). Huggingface accelerate \citep{accelerate2022} (Apache-2.0 license) is used to perform multi-GPU training, and Weights \& Biases \citep{wandb2020} (MIT license) is used to monitor trainings.

\newpage
\section{Additional experimental results}

\subsection{Synthetic Gaussian mixtures}
\label{sec:app:syn-mixture-results}

The models are evaluated using the same evaluation metrics as discussed in Section~\ref{app:sec:mnist-mixture-evaluations}. The experimental results using the reuse scheme are shown in Table~\ref{tbl:synthetic-gmm}.

\begin{table*}[!htbp]
\setlength{\tabcolsep}{3pt} %
	\begin{center}
	\caption{Results on random gmm 20 0.1 1.0 with dim 10}
    \label{tbl:synthetic-gmm}
		\begin{tabular}{llccccc}
			\toprule
			\textbf{Category} & \textbf{Method} & \textbf{MSE} $\downarrow$ & \textbf{Ratio} $\downarrow$ & \textbf{NormMSE} $\downarrow$ & \textbf{NormNLL} $\downarrow$ & \textbf{Time} $\downarrow$ \\
			\midrule
			\multirow{1}{*}{\begin{tabular}[c]{@{}c@{}}Temporal\end{tabular}} & \tnce & 72.1 $\pm$ 48.78 & 28.67 $\pm$ 25.24 & 62.16 $\pm$ 70.56 & -0.29 $\pm$ 4.15 & 1.78 \\
			\midrule
			\multirow{1}{*}{\begin{tabular}[c]{@{}c@{}}Spatial\end{tabular}} & \tcnce & 34.46 $\pm$ 18.92 & 10.7 $\pm$ 0.96 & 33.89 $\pm$ 3.16 & -0.5 $\pm$ 0.27 & \textbf{1.45} \\
			\midrule
			\multirow{3}{*}{\begin{tabular}[c]{@{}c@{}}Spatial-\\temporal\end{tabular}} & \tstnceMixture & \underline{8.83 $\pm$ 2.4} & 5.74 $\pm$ 0.79 & 5.69 $\pm$ 1.08 & -4.16 $\pm$ 0.23 & 2.08$^{*}$ \\
			 & \tstnceWhite & 27.72 $\pm$ 17.25 & \underline{5.47 $\pm$ 0.98} & \underline{4.79 $\pm$ 0.92} & \underline{-4.4 $\pm$ 0.16} & \underline{1.62} \\
			 & \tstnceSelf & \textbf{4.29 $\pm$ 2.2} & \textbf{3.82 $\pm$ 0.45} & \textbf{3.16 $\pm$ 0.51} & \textbf{-4.73 $\pm$ 0.15} & 1.87 \\
			\midrule
			\multirow{1}{*}{\begin{tabular}[c]{@{}c@{}}Oracle\end{tabular}} & \tstnceOracle & \textbf{2.02 $\pm$ 0.34} & \textbf{3.51 $\pm$ 0.42} & \textbf{2.82 $\pm$ 0.36} & \textbf{-4.8 $\pm$ 0.1} & \textbf{1.47} \\
			\bottomrule
		\end{tabular}
	\end{center}
\end{table*}

In Table~\ref{tbl:synthetic-gmm-default} we additionally present results using the default scheme as a verification that the results are consistent across different sampling schemes.

\begin{table*}[!htbp]
\setlength{\tabcolsep}{3pt} %
	\begin{center}
	\caption{Results on random gmm 20 0.1 1.0 with dim 10}
    \label{tbl:synthetic-gmm-default}
		\begin{tabular}{llccccc}
			\toprule
			\textbf{Category} & \textbf{Method} & \textbf{MSE} $\downarrow$ & \textbf{Ratio} $\downarrow$ & \textbf{NormMSE} $\downarrow$ & \textbf{NormNLL} $\downarrow$ & \textbf{Time} $\downarrow$ \\
			\midrule
			\multirow{1}{*}{\begin{tabular}[c]{@{}c@{}}Temporal\end{tabular}} & \tnce & 51.11 $\pm$ 21.93 & 18.33 $\pm$ 1.9 & 32.74 $\pm$ 6.27 & -1.03 $\pm$ 0.62 & \underline{1.72} \\
			\midrule
			\multirow{1}{*}{\begin{tabular}[c]{@{}c@{}}Spatial\end{tabular}} & \tcnce & 35.24 $\pm$ 2.26 & 10.63 $\pm$ 0.84 & 34.55 $\pm$ 3.17 & -0.44 $\pm$ 0.28 & \textbf{1.42} \\
			\midrule
			\multirow{3}{*}{\begin{tabular}[c]{@{}c@{}}Spatial-\\temporal\end{tabular}} & \tstnceMixture & \underline{4.2 $\pm$ 0.33} & 7.29 $\pm$ 0.78 & 6.79 $\pm$ 0.84 & -4.06 $\pm$ 0.16 & 1.95$^{*}$ \\
			 & \tstnceWhite & 4.32 $\pm$ 1.29 & \underline{5.67 $\pm$ 0.61} & \underline{4.89 $\pm$ 0.38} & \underline{-4.4 $\pm$ 0.13} & 1.86 \\
			 & \tstnceSelf & \textbf{2.15 $\pm$ 0.38} & \textbf{3.63 $\pm$ 0.37} & \textbf{2.78 $\pm$ 0.45} & \textbf{-4.86 $\pm$ 0.12} & 2.52 \\
			\midrule
			\multirow{1}{*}{\begin{tabular}[c]{@{}c@{}}Oracle\end{tabular}} & \tstnceOracle & \textbf{2.05 $\pm$ 0.31} & \textbf{3.68 $\pm$ 0.74} & \textbf{2.85 $\pm$ 0.59} & \textbf{-4.84 $\pm$ 0.13} & \textbf{1.82} \\
			\bottomrule
		\end{tabular}
	\end{center}
\end{table*}

We additionally verified that the results are highly consistent across different configurations of the Gaussian mixture, e.g. when the means and / or the covariances are scaled.

\subsection{MNIST mixture}
\label{sec:app:mnist-mixture-results}

We report samples drawn using the model trained with \tstnceSelf and the model trained with DSM in Figure~\ref{fig:mnist-mixture-samples}. When drawing the samples, we differentiate the energy network to obtain score estimates, form velocity estimates using the score estimates as described in Appendix~\ref{app:sec:score-velocity}, and integrate the resulting ODE using $1000$ steps of Euler integrations. While DSM results in better samples, as shown in Table~\ref{tbl:mnist-mixture-results} of main paper and further verified in Figure~\ref{fig:mnist-mixture-ddplots}, DSM does not learn the correct energies.

In the main paper, we reported results where the NCE algorithms use the reuse scheme. We also report the results where the NCE algorithms use the default scheme in Table~\ref{tbl:mnist-mixture-default}. The results are largely consistent regardless of the sampling scheme.

Perhaps surprisingly, \tstnceSelf generally outperforms \tstnceOracle. One possible and heuristic explanation is that \tstnceSelf adjusts the perturbation kernel as training progresses, thus helping optimization. Nevertheless, we do not have enough evidence to provide a strong explanation.

\begin{figure}
    \centering
    \begin{tabular}{cc}
        \includegraphics[width=0.3\linewidth]{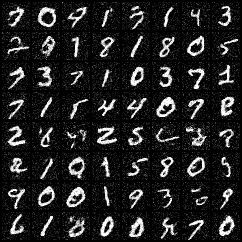} & \includegraphics[width=0.3\linewidth]{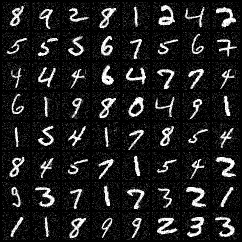}
    \end{tabular}
    \caption{Left: samples drawn using model trained using \tstnceSelf; right: samples drawn using model trained using DSM.}
    \label{fig:mnist-mixture-samples}
\end{figure}

\begin{figure}
    \centering
    \begin{tabular}{c}
        \includegraphics[width=0.9\linewidth]{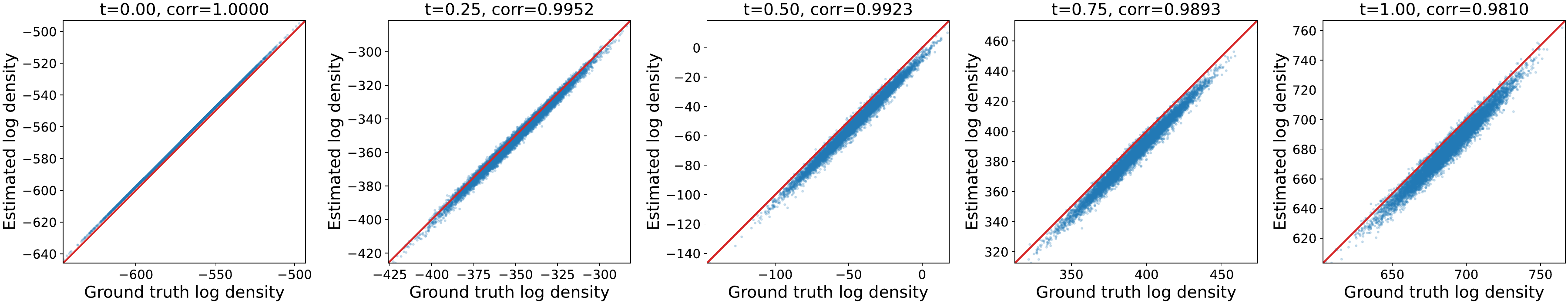} \\
        \includegraphics[width=0.9\linewidth]{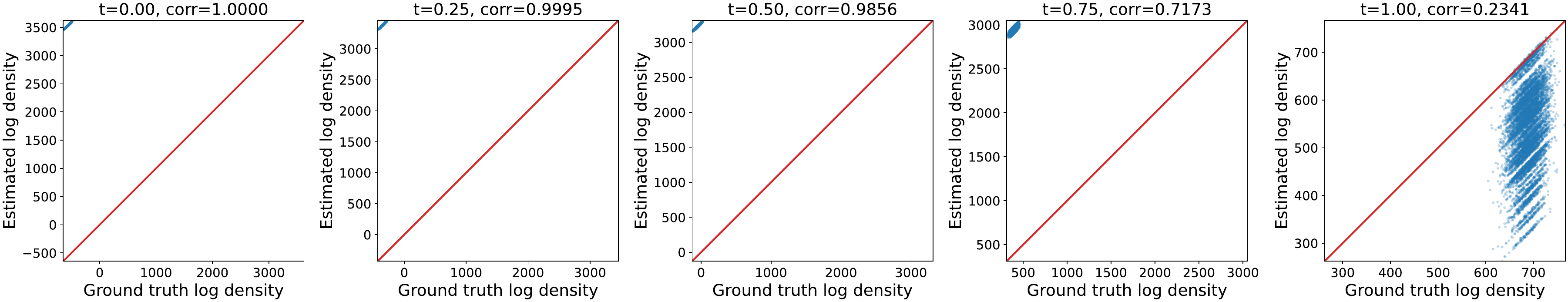}
    \end{tabular}
    
    \caption{We generate density-density plots, where the model's learned energies are normalized using log constant estimated at time corresponding to data. $t=0$ and $t=1$ are approximated using the closest possible $t$ that the models were trained on. Top: Density-density plots of model trained using \tstnceSelf across times. The learned densities are highly accurate. Bottom: Density-density plots of model trained using \dsm across times. The trained model has time-varying log-constant, and the learned energies become highly inaccurate as $t$ approaches $1$, in which case the multimodality of the data distribution increases.}
    \label{fig:mnist-mixture-ddplots}
\end{figure}

\begin{table*}[!htbp]
\centering
\setlength{\tabcolsep}{3pt} %
\caption{Results on mnist mixture with dim 784 using default scheme. Reported times are average times per step in seconds. Results are largely consistent with reuse scheme.}
\label{tbl:mnist-mixture-default}
\begin{tabular}{@{}llcccccc@{}}
\toprule
			\textbf{Category} & \textbf{Method} & \textbf{LogMSE} $\downarrow$ & \textbf{LogRatio} $\downarrow$ & \textbf{LogNormMSE} $\downarrow$ & \textbf{NormNLL} $\downarrow$ & \textbf{Time} $\downarrow$ \\
			\midrule
			\multirow{2}{*}{\begin{tabular}[c]{@{}c@{}}Temporal\end{tabular}} & \tnce & 13.31 $\pm$ 0.03 & 6.72 $\pm$ 0.11 & 8.89 $\pm$ 0.33 & -604.45 $\pm$ 13.76 & \textbf{0.4} \\
			 & \ctsm & 12.44 $\pm$ 0.45 & 7.14 $\pm$ 0.41 & 8.96 $\pm$ 0.24 & -602.92 $\pm$ 9.59 & 0.68 \\
			\midrule
			\multirow{2}{*}{\begin{tabular}[c]{@{}c@{}}Spatial\end{tabular}} & \tcnce & 11.56 $\pm$ 0.59 & 7.58 $\pm$ 0.17 & 8.9 $\pm$ 0.08 & -608.12 $\pm$ 2.46 & \textbf{0.4} \\
			 & \dsm & 16.39 $\pm$ 0.07 & 9.31 $\pm$ 0.43 & 10.24 $\pm$ 0.39 & -535.54 $\pm$ 28.81 & 0.7 \\
			\midrule
			\multirow{5}{*}{\begin{tabular}[c]{@{}c@{}}Spatial-\\temporal\end{tabular}} & \tstnceMixture & 10.55 $\pm$ 0.18 & 6.42 $\pm$ 0.36 & 7.72 $\pm$ 0.1 & -644.25 $\pm$ 1.36 & \underline{0.44}$^{*}$ \\
			 & \tstnceWhite & 10.09 $\pm$ 0.13 & 6.31 $\pm$ 0.41 & 7.59 $\pm$ 0.41 & -646.07 $\pm$ 8.23 & \textbf{0.4} \\
			 & \tstnceSelf & \underline{5.19 $\pm$ 0.11} & \underline{3.64 $\pm$ 0.03} & \underline{4.51 $\pm$ 0.1} & \underline{-679.61 $\pm$ 0.47} & 0.47 \\
			 & \rneSelf & 14.86 $\pm$ 0.31 & 7.75 $\pm$ 0.39 & 9.62 $\pm$ 0.05 & -570.51 $\pm$ 1.09 & \underline{0.44} \\
			 & \dualSm & \textbf{3.73 $\pm$ 0.18} & \textbf{2.88 $\pm$ 0.06} & \textbf{3.25 $\pm$ 0.11} & \textbf{-683.96 $\pm$ 0.33} & 0.7 \\
			\midrule
			\multirow{2}{*}{\begin{tabular}[c]{@{}c@{}}Oracle\end{tabular}} & \tstnceOracle & 4.99 $\pm$ 0.11 & 3.69 $\pm$ 0.06 & 4.95 $\pm$ 0.03 & -677.07 $\pm$ 0.12 & \textbf{0.32} \\
			 & \rneOracle & \textbf{2.92 $\pm$ 0.08} & \textbf{2.06 $\pm$ 0.02} & \textbf{2.24 $\pm$ 0.07} & \textbf{-685.75 $\pm$ 0.09} & 0.4 \\
			\bottomrule
\end{tabular}
\end{table*}

\subsubsection{Analysis of trained models}
\label{app:sec:analysis-mnist-mixture}

We consider a temporal method (\tnce), a spatial method (\tcnce), and four different spatiotemporal methods (\tstnceMixture, \tstnceWhite, \tstnceOracle and \tstnceSelf). When applicable, we use the models trained with the standard deviations of the spatial and temporal Gaussian perturbation kernels set to $0.01$; note that these do not necessarily correspond to their best configurations. We sample $100$ data samples with one from each mode, so as to cover all the $100$ modes of the Gaussian mixture. We randomly shift these data points and form $100$ pairs of data samples, such that both the first elements and the second elements of the pairs cover all the modes and no pairs contain points from the same mode.

For each pair, we construct the interpolations as described in Figure~\ref{fig:limitations} and Section~\ref{sec:limitations} using $100$ steps. The interpolation steps are indicated using $k$ between $0$ and $K$, adapting the notations from Section~\ref{sec:limitations}.

Along the interpolation, we classify a point $(x,t)$ as having low density if $\min_j\frac{\left\lVert x-t\mu_j \right\rVert^2}{t^2(0.2)^2 +(1-t)^2} > \chi^2_{784}(0.99)$, and an edge $((x,t),(x',t'))$ as low density if either end point is low density; we remark that it is $0.2$ instead of $0.1$ due to rescaling of the data. We compute pointwise errors $q_k$ and define the local ratio errors as $q_{k+1} - q_k$. The resulting ratio errors telescope across $k$ when summed over all $k$, resulting in $q_K - q_0$.

Consider two data points, $x$ and $y$. For the temporal method, we fix $x$, $y$ and vary $t$. The pointwise error is given by
\begin{equation}
q_k = \left(\log p_{t_k}^{\theta}(x) - \log p_{t_k}^{\theta}(y)\right) - \left(\log p_{t_k}(x) - \log p_{t_k}(y)\right).
\end{equation}
For spatial methods, we fix $t=t^*$, the time point corresponding to data, and linearly interpolate between $x$ and $y$ using $z_k$. The resulting pointwise errors are given by
\begin{equation}
q_k = \log p_{t^*}^{\theta}(z_k) - \log p_{t^*}(z_k).
\end{equation}
For spatiotemporal methods, we sample a common standard normal noise $\epsilon$. As $t$ varies, $x_k$ and $y_k$ interpolate from $\epsilon$ to $x$ and $y$. The pointwise errors are given by
\begin{equation}
q_k = \left(\log p_{t_k}^{\theta}(x_{k}) - \log p_{t_k}^{\theta}(y_k)\right) - \left(\log p_{t_k}(x_k)-\log p_{t_k}(y_k)\right).
\end{equation}
Note that these expressions correspond to the expressions in Section~\ref{sec:limitations}.

We measure the proportions of low density edges along each interpolation and compare the models' errors in high and low density regions and report the results in Table~\ref{tbl:mnist-mixture-analysis}. In the result table, \textbf{HighDensityRMSE} and \textbf{LowDensityRMSE} denote the RMSE of local ratio in high and low density regions, \textbf{LowDensityRate} denotes the proportion of edges classified as low density, and \textbf{LogNormMSE} refers to the same LogNormMSE as reported in the main paper.

As expected, \tnce and \tcnce, the temporal and spatial methods, frequently traverse low density regions and exhibit substantially larger errors in such regions than the high density ones. By contrast, spatiotemporal methods encounter few low density edges. This provides a fine-grained possible explanation regarding why the spatial and temporal methods fail on MNIST mixture dataset: their final performance is likely dominated by the estimation errors in low-density regions.

\begin{table}
\centering
\caption{Analysis of models trained on MNIST mixture. Unlike spatiotemporal methods, temporal and spatial methods encounter notable low density edges along the corresponding interpolation paths, and have consistently higher errors in such low density regions that high density ones.}
\label{tbl:mnist-mixture-analysis}
\begin{tabular}{lrrrr}
\toprule
method & HighDensityRMSE & LowDensityRMSE & LowDensityRate & LogNormMSE \\
\midrule
tNCE & 0.41 & 1.00 & 29.69\% & 9.22 \\
tCNCE & 5.79 & 27.54 & 57.47\% & 9.02 \\
stNCE-m & 0.44 & 0.51 & 1.19\% & 7.73 \\
stNCE-w & 0.43 & 0.53 & 1.19\% & 7.18 \\
stNCE-o & 0.20 & 0.19 & 1.19\% & 4.92 \\
stNCE-s & 0.18 & 0.14 & 1.19\% & 4.32 \\
\bottomrule
\end{tabular}
\end{table}

\subsection{Images}

\subsubsection{MNIST}
\label{app:sec:mnist-additional-results}

We draw samples by obtaining the spatial scores through differentiating the energy network, converting them into velocity estimates as in Appendix~\ref{app:sec:score-velocity}, and integrate the resulting ODE with $1000$ steps of Euler integrator. The samples as visualized in Figure~\ref{fig:mnist_samples} look realistic.

\begin{table*}[!htbp]
\centering

\begin{minipage}{0.69\textwidth}
\centering
\setlength{\tabcolsep}{3pt}
\renewcommand{\arraystretch}{0.95}
\caption{Negative log likelihood (in bits/dimension) on MNIST test set. Result for Glow is taken from~\citet{mate2022flowification}.
Results for FFJORD, i-ResNet, and MintNet are taken from~\citet{song2019mintnet}. TSM result is taken from~\citet{yu2025dre}. Reported times are average times per step in seconds.}
\label{table:mnist}
\begin{tabular}{@{}lcccc@{}}
\toprule
\textbf{Method} & \textbf{Type} & \textbf{Single NFE} & \textbf{Time $\downarrow$} & \textbf{NLL $\downarrow$} \\
\midrule
\tnce & Estimate & \cmark & 0.127 & 7.01 \\
FP-Diffusion~\citep{lai2023fpdiffusion} & Normalized & \xmark & - & 2.98 \\
TSM~\citep{choi2022densityratio} & Estimate & \xmark & - & 1.3 \\
TRE~\citep{rhodes2020telescoping} & Estimate & \xmark & - & 1.09 \\
i-ResNet~\citep{behrmann2019iresnet} & Normalized & \xmark & - & 1.06 \\
Glow~\citep{kinga2018glow} & Normalized & \cmark & - & 1.05 \\
CTSM-v~\citep{yu2025dre} & Estimate & \xmark & - & 1.03 \\
FFJORD~\citep{grathwohl2018scalable} & Normalized & \xmark & - & 0.99 \\
MintNet~\citep{song2019mintnet} & Normalized & \cmark & - & 0.98 \\
\midrule
Ours (\tstnceMixture) & Estimate & \cmark & 0.17$^{*}$ & 4.92 \\
Ours (\tstnceWhite) & Estimate & \cmark & 0.13 & 4.56 \\
Ours (\tstnceSelf) & Estimate & \cmark & 0.13 & 1.61 \\
Ours (\tstnceSelfDsm) & Estimate & \cmark & 0.27 & 1.00 \\
\bottomrule
\end{tabular}
\end{minipage}
\hfill
\begin{minipage}{0.35\textwidth}
    \centering
    \includegraphics[width=0.55\linewidth]{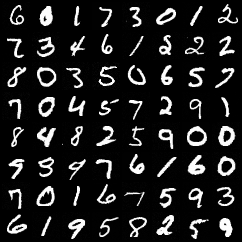}
    \captionof{figure}{Samples drawn using \tstnceSelfDsm.}
    \label{fig:mnist_samples}
\vspace{0.5em}
\captionsetup{
justification=raggedright,
singlelinecheck=false
}
\captionof{table}{Test BPD of \tstnceSelf on MNIST. \textbf{Last} and \textbf{Best} indicates the test BPD at the step and at the step with the best validation BPD up to that step.}
\begin{tabular}{rcc}
\hline
\textbf{Step} & \textbf{Last} & \textbf{Best} \\
\hline
$100$k & $1.62$ & $1.61$ \\
$200$k & $1.49$ & $1.48$ \\
$300$k & $1.42$ & $1.41$ \\
$400$k & $1.37$ & $1.36$ \\
$500$k & $1.34$ & $1.33$ \\
\hline
\end{tabular}
\label{tbl:mnist-tstnce-long}
\end{minipage}
\end{table*}

\subsubsection{ImageNet}

\begin{figure}
    \centering
    \includegraphics[width=0.5\linewidth]{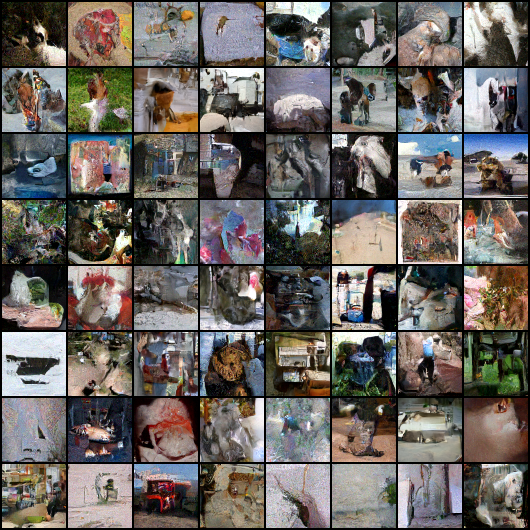}
    \caption{Samples drawn using $200$ steps of Euler integrator with the trained model.}
    \label{fig:imagenet-samples}
\end{figure}

In order to validate the sampling capabilities of the model trained with \tstnceSelfDsm on ImageNet-64, we form score estimates by differentiating through the energy network, form velocity estimates from the score estimates according to Appendix~\ref{app:sec:score-velocity} and draw samples using $200$ steps of Euler integrator. Using a batch size of $512$ on a single NVIDIA GH200, it takes less than a minute. We further visualize the samples in Figure~\ref{fig:imagenet-samples}, validating that the samples are largely reasonable.

\subsection{Molecules}

\subsubsection{Additional results}
\label{app:sec:molecule-additional-results}

In the main paper, we reported PMF and time per epoch. Here we additionally report JS. For Jensen-Shannon Divergence (JS), similar to PMF, we use histograms of the ground truth and generated samples in the $2$ dimensional space, $P_b$ and $Q_b$, and compute the JS as
\begin{equation}
\mathrm{JSD}(P,Q) = \frac{1}{2}\sum_b P_b\log\frac{P_b}{M_b} + \frac{1}{2}Q_b\log\frac{Q_b}{M_b}.
\end{equation}
We report the numerical results in Table~\ref{tbl:aldp-results} and Table~\ref{tbl:chignolin-results}. JS largely follows the same trend as PMF.

\begin{table}[t]
\centering
\caption{Alanine dipeptide results. Time is time per epoch in seconds.}
\label{tbl:aldp-results}
\setlength{\tabcolsep}{3pt}
\begin{tabular}{@{}lccccc@{}}
\toprule
Method & Diffusion JS $\downarrow$ & Langevin JS $\downarrow$ & Diffusion PMF $\downarrow$ & Langevin PMF $\downarrow$ & Time $\downarrow$ \\
\midrule
\tnce & $0.0093 \pm 0.0001$ & $0.0111 \pm 0.0003$ & $0.095 \pm 0.001$ & $0.119 \pm 0.003$ & $0.610$ \\
\tnceDsm & $0.0063 \pm 0.0000$ & $0.0065 \pm 0.0001$ & $0.063 \pm 0.002$ & $0.066 \pm 0.003$ & $0.780$ \\
\midrule
\tstnceSelf & $0.0098 \pm 0.0002$ & $0.0111 \pm 0.0001$ & $0.106 \pm 0.004$ & $0.139 \pm 0.001$ & $0.575$ \\
\tstnceSelfDsm & $0.0072 \pm 0.0000$ & $0.0082 \pm 0.0002$ & $0.076 \pm 0.002$ & $0.088 \pm 0.005$ & $0.648$ \\
\midrule
FPE \citep{plainer2025consistent} & $0.0082 \pm 0.0002$ & $0.0090 \pm 0.0006$ & $0.098 \pm 0.003$ & $0.104 \pm 0.004$ & $1.608$ \\
\bottomrule
\end{tabular}
\end{table}

\begin{table}[t]
\centering
\caption{Chignolin results. Time is time per epoch in seconds.}
\label{tbl:chignolin-results}
\setlength{\tabcolsep}{3pt}
\begin{tabular}{@{}lccccc@{}}
\toprule
Method & Diffusion JS $\downarrow$& Langevin JS $\downarrow$& Diffusion PMF $\downarrow$& Langevin PMF $\downarrow$ & Time $\downarrow$\\
\midrule
\tnce & $0.0593 \pm 0.0001$ & $0.2420 \pm 0.0020$ & $0.654 \pm 0.001$ & $3.143 \pm 0.031$ & $2.343$ \\
\tnceDsm & $0.0073 \pm 0.0019$ & $0.0181 \pm 0.0055$ & $0.060 \pm 0.015$ & $0.162 \pm 0.053$ & $2.866$ \\
\midrule
\tstnceSelf & $0.0097 \pm 0.0000$ & $0.0199 \pm 0.0015$ & $0.083 \pm 0.000$ & $0.170 \pm 0.014$ & $2.103$ \\
\tstnceSelfDsm & $0.0074 \pm 0.0001$ & $0.0082 \pm 0.0006$ & $0.059 \pm 0.001$ & $0.065 \pm 0.005$ & $2.349$ \\
\midrule
FPE \citep{plainer2025consistent} & $0.0048 \pm 0.0001$ & $0.0050 \pm 0.0001$ & $0.037 \pm 0.000$ & $0.039 \pm 0.001$ & $6.462$ \\
\bottomrule
\end{tabular}
\end{table}

We report the samples drawn using diffusion sampler and molecular dynamic simulations based on model trained using \tstnceSelfDsm on Chignolin in Figure~\ref{fig:chignolin-stnce-samples}. The samples agree well with each other, highlighting the quality of the learned energies.

\begin{figure}
    \centering
    \begin{tabular}{cc}
        \includegraphics[width=0.3\textwidth]{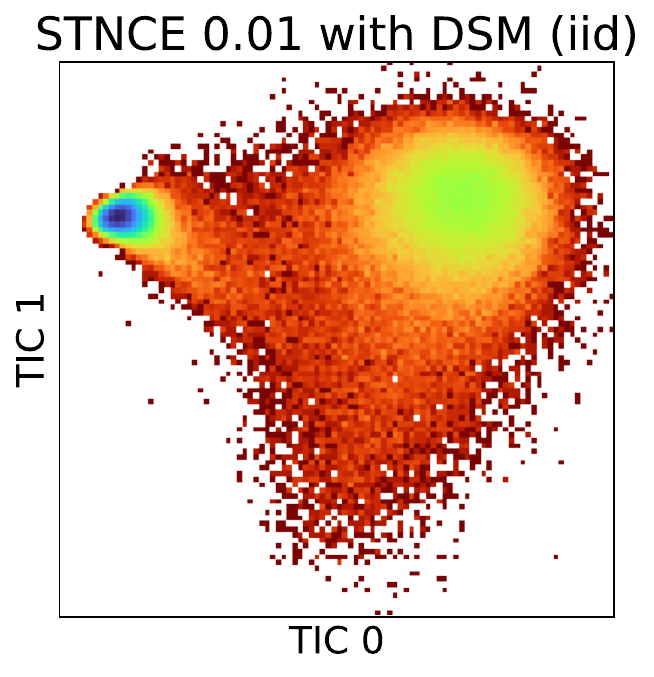} & \includegraphics[width=0.3\textwidth]{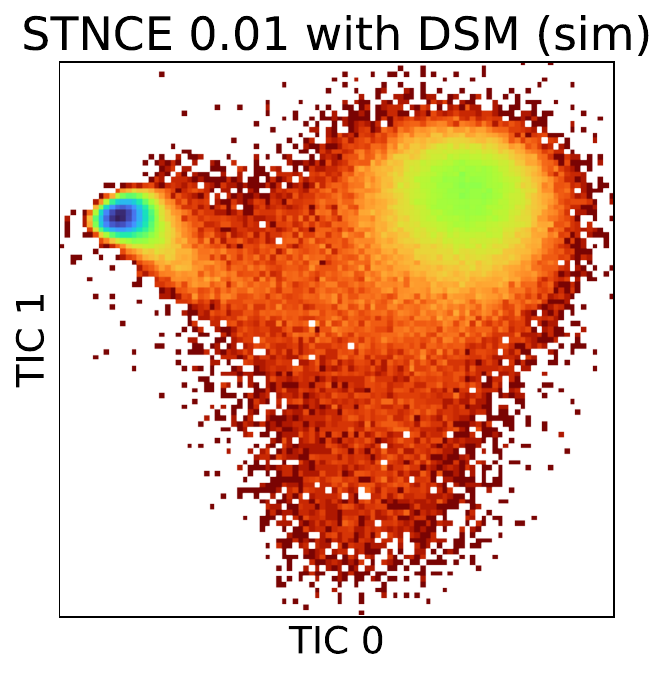} 
    \end{tabular}
    \caption{Samples drawn using diffusion sampler (left) and molecular dynamic simulations (right) with the model trained using \tstnceSelfDsm. The samples are highly consistent, verifying the correctness of the learned energy-based diffusion model.}
    \label{fig:chignolin-stnce-samples}
\end{figure}

\subsubsection{Higher dimensional problems}

As presented in the main paper, in general, on alanine dipeptide \tnce outperforms \tstnceSelf either with or without DSM, while both outperfrom FPE, the previous SOTA solution, when combined with DSM. On chignolin \tstnceSelf outperforms \tnce, both with and without DSM. 

We next list three molecule datasets along with their dimensionalities:
\begin{itemize}
\item alanine dipeptide (aldp), the coarse-grained version, has $15$ dimensions obtained from $5$ retained atoms,
\item chignolin, has $30$ dimensions,
\item the full version of alainine dipeptide (aldp full), the fine-grained version, has $66$ dimensions obtained from $22$ atoms.
\end{itemize}

We hypothesize that the dimensionality of the dataset contributes to the relative performance of the different algorithms:
\begin{itemize}
\item temporal methods fail in high dimensions, because the support mismatch between the data and the noise generally grows larger as the dimensionality increases,
\item spatiotemporal methods can work in high dimensions when they follow high-density regions as illustrated in blue in the left panel of Figure~\ref{fig:teaser_figure}. This is especially the case for density-informed perturbation kernels, e.g. \tstnceSelf, as illustrated in Figure~\ref{fig:perturbation_kernels_illustration}.
\end{itemize}

We similarly hypothesize that spatial methods also fail in high dimensions due to similar reason. In low dimensions and when data and noise have sufficient overlapping support, we do not claim that spatiotemporal methods would uniformly outperform temporal methods.

The empirical evidence indeed supports the above hypothesis. The molecule datasets are normalized before feeding into the network, as such the dimensionality is arguably the most important factor influencing the performance, where we would expect \tstnceSelf to outperform \tnce on aldp full given its larger dimensionality. Indeed, as verified in Table~\ref{tbl:aldp-full}, \tstnceSelf notably outperforms \tnce and \tcnce, though aldp was obtained as taking a subset of coordinates from aldp full.

\begin{table}[t]
\centering
\caption{Alanine dipeptide (full version) results. \tstnceSelf notably outperforms \tnce and \tcnce, validating the hypothesized improvement of spatiotemporal methods based on density-informed forward-reverse kernel over simpler temporal and spatial methods.}
\label{tbl:aldp-full}
\begin{tabular}{lcccc}
\toprule
Method & IID JS $\downarrow$ & Langevin JS $\downarrow$ & IID PMF $\downarrow$ & Langevin PMF $\downarrow$ \\
\midrule
\tnce & $0.1023 \pm 0.0002$ & $0.2757 \pm 0.0008$ & $1.681 \pm 0.010$ & $5.952 \pm 0.037$ \\
\tcnce & $0.0891 \pm 0.0001$ & $0.3829 \pm 0.0024$ & $1.423 \pm 0.005$ & $11.838 \pm 0.108$ \\
\tstnceSelf & $0.0479 \pm 0.0003$ & $0.0635 \pm 0.0005$ & $0.678 \pm 0.011$ & $0.708 \pm 0.011$ \\
\bottomrule
\end{tabular}
\end{table}

\clearpage

\newpage
\section*{NeurIPS Paper Checklist}

\begin{enumerate}

\item {\bf Claims}
    \item[] Question: Do the main claims made in the abstract and introduction accurately reflect the paper's contributions and scope?
    \item[] Answer: \answerYes %
    \item[] Justification: We provide a separate \textit{Contributions} paragraph in Introduction. These claims were later verified either empirically or mathematically.

\item {\bf Limitations}
    \item[] Question: Does the paper discuss the limitations of the work performed by the authors?
    \item[] Answer: \answerYes %
    \item[] Justification: We provide a separate \textit{Limitations and future directions} paragraph in Discussion.

\item {\bf Theory assumptions and proofs}
    \item[] Question: For each theoretical result, does the paper provide the full set of assumptions and a complete (and correct) proof?
    \item[] Answer: \answerYes{} %
    \item[] Justification: For the theoretical results, we provide the full set of assumptions and proofs, possibly in the Appendix, notably Appendix~\ref{app:sec:theory}.

    \item {\bf Experimental result reproducibility}
    \item[] Question: Does the paper fully disclose all the information needed to reproduce the main experimental results of the paper to the extent that it affects the main claims and/or conclusions of the paper (regardless of whether the code and data are provided or not)?
    \item[] Answer: \answerYes{} %
    \item[] Justification: We provide experimental details in the main paper or the Appendix, notably Appendix~\ref{app:sec:exp-details}, and provide code.

\item {\bf Open access to data and code}
    \item[] Question: Does the paper provide open access to the data and code, with sufficient instructions to faithfully reproduce the main experimental results, as described in supplemental material?
    \item[] Answer: \answerYes{} %
    \item[] Justification: We provide our code along with instructions to obtain the relevant datasets.

\item {\bf Experimental setting/details}
    \item[] Question: Does the paper specify all the training and test details (e.g., data splits, hyperparameters, how they were chosen, type of optimizer) necessary to understand the results?
    \item[] Answer: \answerYes{} %
    \item[] Justification: We provide experimental details mainly in the Appendix, especially Appendix~\ref{app:sec:exp-details}.

\item {\bf Experiment statistical significance}
    \item[] Question: Does the paper report error bars suitably and correctly defined or other appropriate information about the statistical significance of the experiments?
    \item[] Answer: \answerYes{} %
    \item[] Justification: For experiments other than MNIST and ImageNet, we provide error bars as computed across multiple independent runs. For MNIST and ImageNet, following convention in generative modeling, we directly report results obtained in a single run.

\item {\bf Experiments compute resources}
    \item[] Question: For each experiment, does the paper provide sufficient information on the computer resources (type of compute workers, memory, time of execution) needed to reproduce the experiments?
    \item[] Answer: \answerYes{} %
    \item[] Justification: We provide information on the employed compute resources for the experiments in Appendix~\ref{app:sec:compute-resources}.
    
\item {\bf Code of ethics}
    \item[] Question: Does the research conducted in the paper conform, in every respect, with the NeurIPS Code of Ethics \url{https://neurips.cc/public/EthicsGuidelines}?
    \item[] Answer: \answerYes{} %
    \item[] Justification: In our research we conform with the NeurIPS Code of Ethics.

\item {\bf Broader impacts}
    \item[] Question: Does the paper discuss both potential positive societal impacts and negative societal impacts of the work performed?
    \item[] Answer: \answerYes{} %
    \item[] Justification: We discuss broader impacts in Appendix~\ref{app:sec:broader-impacts}.
    
\item {\bf Safeguards}
    \item[] Question: Does the paper describe safeguards that have been put in place for responsible release of data or models that have a high risk for misuse (e.g., pre-trained language models, image generators, or scraped datasets)?
    \item[] Answer: \answerYes{} %
    \item[] Justification: The trained EBMs could be used as e.g. image generators. We attempt to safeguard this by advocating for responsible usages of such trained models.

\item {\bf Licenses for existing assets}
    \item[] Question: Are the creators or original owners of assets (e.g., code, data, models), used in the paper, properly credited and are the license and terms of use explicitly mentioned and properly respected?
    \item[] Answer: \answerYes{} %
    \item[] Justification: We provide the license information in Appendix~\ref{app:sec:license}.

\item {\bf New assets}
    \item[] Question: Are new assets introduced in the paper well documented and is the documentation provided alongside the assets?
    \item[] Answer: \answerYes{} %
    \item[] Justification: We provide the code and instructions on how to use the code to reproduce our results.

\item {\bf Crowdsourcing and research with human subjects}
    \item[] Question: For crowdsourcing experiments and research with human subjects, does the paper include the full text of instructions given to participants and screenshots, if applicable, as well as details about compensation (if any)? 
    \item[] Answer: \answerNA{} %
    \item[] Justification: The paper does not involve crowdsourcing nor research with human subjects.

\item {\bf Institutional review board (IRB) approvals or equivalent for research with human subjects}
    \item[] Question: Does the paper describe potential risks incurred by study participants, whether such risks were disclosed to the subjects, and whether Institutional Review Board (IRB) approvals (or an equivalent approval/review based on the requirements of your country or institution) were obtained?
    \item[] Answer: \answerNA{} %
    \item[] Justification: The paper does not involve crowdsourcing nor research with human subjects.

\item {\bf Declaration of LLM usage}
    \item[] Question: Does the paper describe the usage of LLMs if it is an important, original, or non-standard component of the core methods in this research? Note that if the LLM is used only for writing, editing, or formatting purposes and does \emph{not} impact the core methodology, scientific rigor, or originality of the research, declaration is not required.
    \item[] Answer: \answerNA{} %
    \item[] Justification: The core methods does not involve LLMs as any important, original or non-standard components.

\end{enumerate}

\end{document}